%% file: _main.tex
\documentclass{article}
\usepackage[numbers]{natbib}

\usepackage[preprint]{neurips_2021}
\usepackage[utf8]{inputenc} 
\usepackage[T1]{fontenc}    
\usepackage{hyperref}       
\usepackage{url}            
\usepackage{booktabs}       
\usepackage{amsfonts}       
\usepackage{nicefrac}       
\usepackage{microtype}      
\usepackage{xcolor}         

\usepackage{here}
\usepackage{color}
\usepackage{mathtools} 
\usepackage{lipsum}
\usepackage{multirow}
\usepackage{amsmath, amssymb} 
\usepackage{amsfonts} 
\usepackage{type1cm} 
\usepackage{bm} 
\usepackage{bbm} 
\usepackage{cases} 
\usepackage{amsthm} 
    \newtheorem{theorem}{Theorem}[section]
    
    \newtheorem{lemma}{Lemma}[section]
    
    \newtheorem{definition}{Definition}[section]
\usepackage{./multibib}
\newcites{App}{Supplementary references}

\def\nn{\nonumber \\ }
\def\no{\nonumber }
\def\ti{\tilde}
\def\mc{\mathcal} 
\def\mr{\mathrm} 
\def\f{\frac}
\def\lt{\left}

\def\rt{\right}

\def\mbr{\mathbb{R}} 

\def\mbz{\mathbb{Z}}
\def\mbn{\mathbb{N}}

\def\mbe{\mathbb{E}}

\def\mcx{\mathcal{X}}
\def\mcy{\mathcal{Y}}

\def\al{\alpha}
\def\be{\beta}
\def\ga{\gamma}
\def\de{\delta}
\def\ep{\epsilon}

\def\th{\theta}

\def\la{\lambda}
\def\rh{\rho}
\def\si{\sigma}
\def\ta{\tau}

\def\ph{\phi}

\def\ps{\psi}
\def\om{\omega}

\def\Th{\Theta}

\def\La{\Lambda}

\def\Up{\Upsilon}

\def\Ps{\Psi}
\def\Om{\Omega}

\title{The Power of Log-Sum-Exp:\\ Sequential Density Ratio Matrix Estimation\\for Speed-Accuracy Optimization}

\author{%
  Taiki Miyagawa \\
  NEC Corporation, Japan\\
  \texttt{miyagawataik@nec.com} \\
  \And
  Akinori F. Ebihara \\
  NEC Corporation, Japan\\
  \texttt{aebihara@nec.com} \\
}

\begin{document}
\maketitle

\begin{abstract}
\input{Abstract}
\end{abstract}
\input{introductionv2}

\input{related}

\input{MSPRT-TANDEMv2}

\input{Experimentv2}
\input{Conclusion}
\input{Acknowledgements} 
{\small
\bibliographystyle{abbrv}
\bibliography{_main}
}

\clearpage
\section*{Appendix}
\appendix
\input{supp_MSPRT}
\clearpage 
\input{supp_Related}

\clearpage 
\input{supp_ProofConsistency}

\clearpage 
\input{supp_Variants}
\clearpage 
\input{supp_TANDEMvsOblivion}

\clearpage 
\input{supp_GuessAversion}
\clearpage
\input{supp_AblationMultLSEL}
\clearpage
\input{supp_DetailExperiment}

\clearpage 
\input{supp_StatTest}
\clearpage 
\input{supp_Discussion}

\clearpage
{\small
\bibliographystyleApp{abbrv}
\bibliographyApp{_main}
}

\end{document}

%% file: Abstract.tex
We propose a model for multiclass classification of time series to make a prediction as early and as accurate as possible.
The \textit{matrix sequential probability ratio test} (MSPRT) is known to be asymptotically optimal for this setting, but contains a critical assumption that hinders broad real-world applications; the MSPRT requires the underlying probability density.
To address this problem, we propose to solve \textit{density ratio matrix estimation} (DRME), a novel type of density ratio estimation that consists of estimating matrices of multiple density ratios with constraints and thus is more challenging than the conventional density ratio estimation.
We propose a log-sum-exp-type loss function (LSEL) for solving DRME and prove the following: (i) the LSEL provides the true density ratio matrix as the sample size of the training set increases (\textit{consistency}); (ii) it assigns larger gradients to harder classes (\textit{hard class weighting effect}); and (iii) it provides discriminative scores even on class-imbalanced datasets (\textit{guess-aversion}).
Our overall architecture for early classification, \textit{MSPRT-TANDEM}, statistically significantly outperforms baseline models on four datasets including action recognition, especially in the early stage of sequential observations.
Our code and datasets are publicly available\footnote{\url{https://github.com/TaikiMiyagawa/MSPRT-TANDEM}}.

%% file: Introductionv2.tex
\section{Introduction}\label{sec:introduction}
Classifying an incoming time series as early and as accurately as possible is challenging yet crucial, especially when the sampling cost is high or when a delay results in serious consequences \cite{Xing2009ECTSorig, Xing2012ECTS, Mori2015early_cost_minimization_point_of_view_NIPSw, Mori2018}. Early classification of time series is a multi-objective optimization problem, and there is usually no ground truth indicating when to stop observation and classify a sequence.

\begin{figure}[ht]
    \centering
    \begin{minipage}[b]{0.49\linewidth}
        \centering
        \includegraphics[width=\columnwidth, keepaspectratio]
        {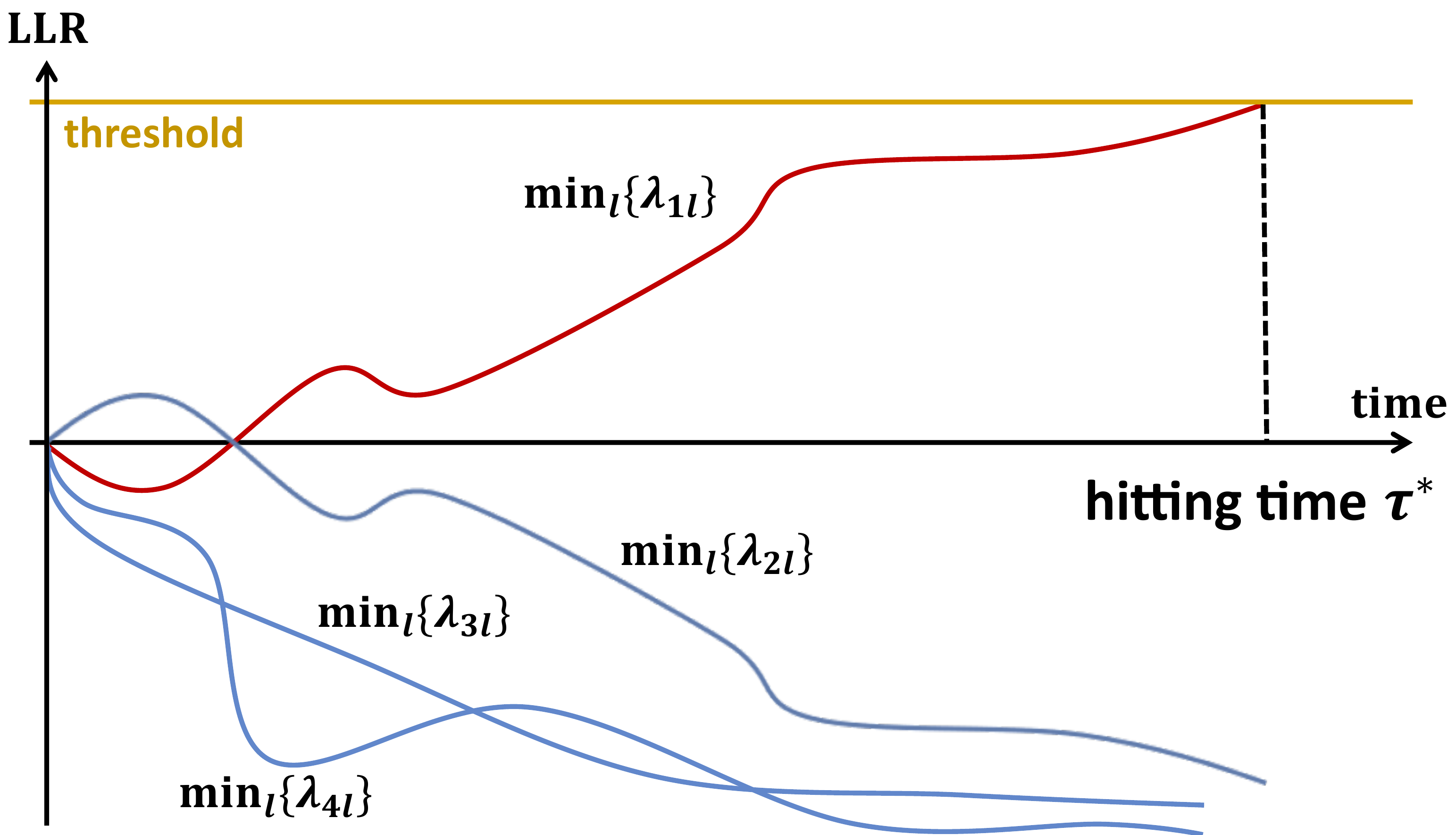} 
    \end{minipage}
    \begin{minipage}[b]{0.49\linewidth}
        \centering
        \includegraphics[width=\columnwidth, keepaspectratio]
        {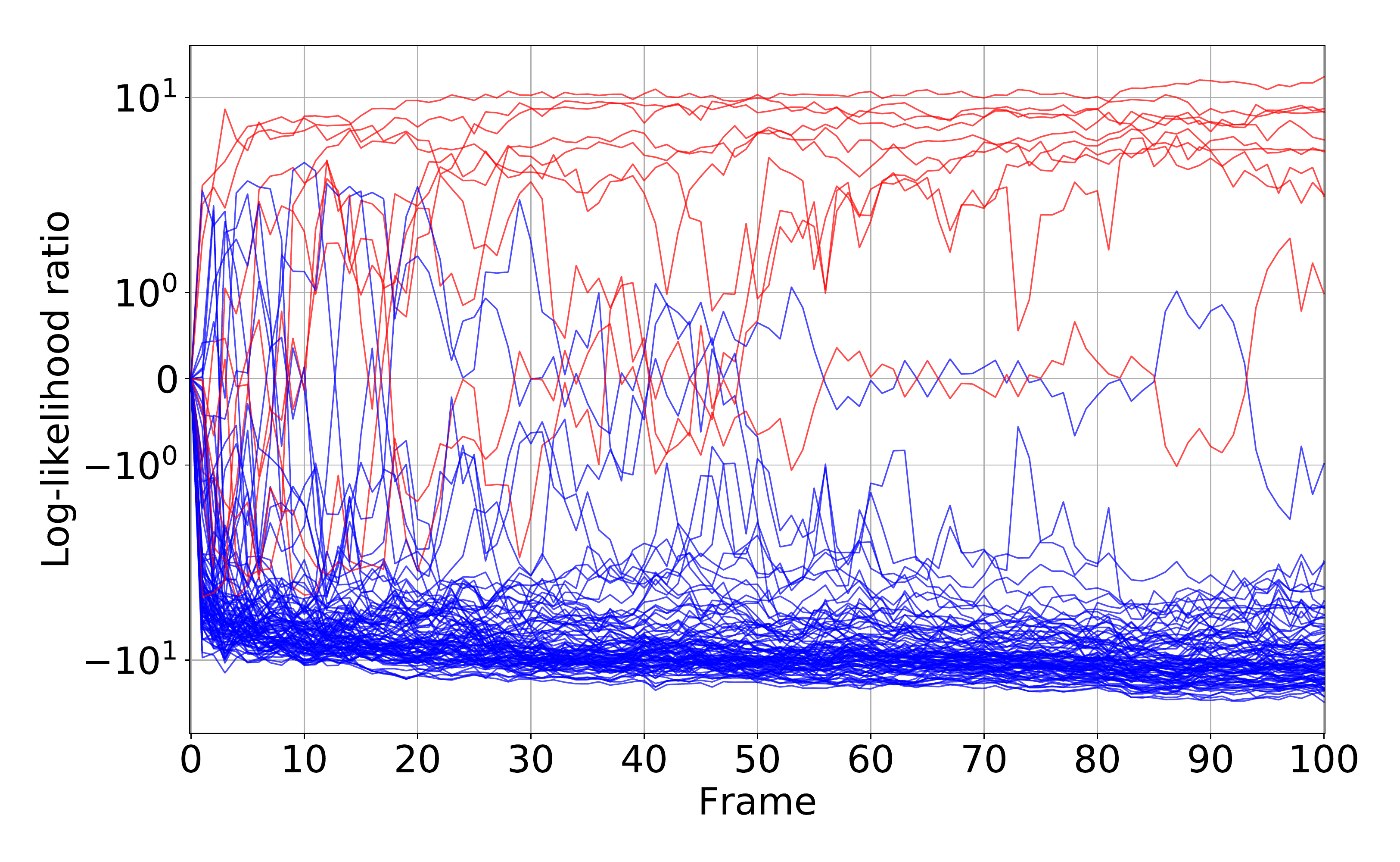} 
    \end{minipage}
    \caption{
        \textbf{Left: Early Classification of Time Series with MSPRT.} The figure illustrates how the MSPRT predicts the label $y$ of an incoming time series $X^{(1,t)} = \{ x^{(1)}, x^{(2)},...x^{(t)} \}$. The MSPRT uses the LLR matrix denoted by $\la_{kl} (X^{(1,t)}) := \log ( p(X^{(1,t)} | y=k) / p(X^{(1,t)} | y=l) )$, where $k, l = 1,2,...,K$. $K (\in \mbn)$ is the number of classes. If one of $\min_{l} \la_{kl} = \log {p(X^{(1,t)} | k)}/{\max_l p(X^{(1,t)} | l)}$ ($k \in \{1,2,...,K\}$) reaches the threshold, the prediction is made; otherwise, the observation continues.
        In this figure, $K=4$, the prediction is $y = 1$, and the hitting time is $\tau^*$.
        A larger threshold leads to more accurate but delayed predictions, while a smaller threshold leads to earlier but less accurate predictions. 
        \textbf{Right: Estimated LLRs of ten sequences.} 
        (See Appendix \ref{app: Figure fig: LLR trajectories of NMNIST-100f} for exact settings.)
        }
    \label{fig: MSPRT}
\end{figure}

\begin{figure}[t]
    \centering
    \begin{minipage}[b]{0.49\linewidth}
    \centerline{\includegraphics[width=\columnwidth, keepaspectratio]
    {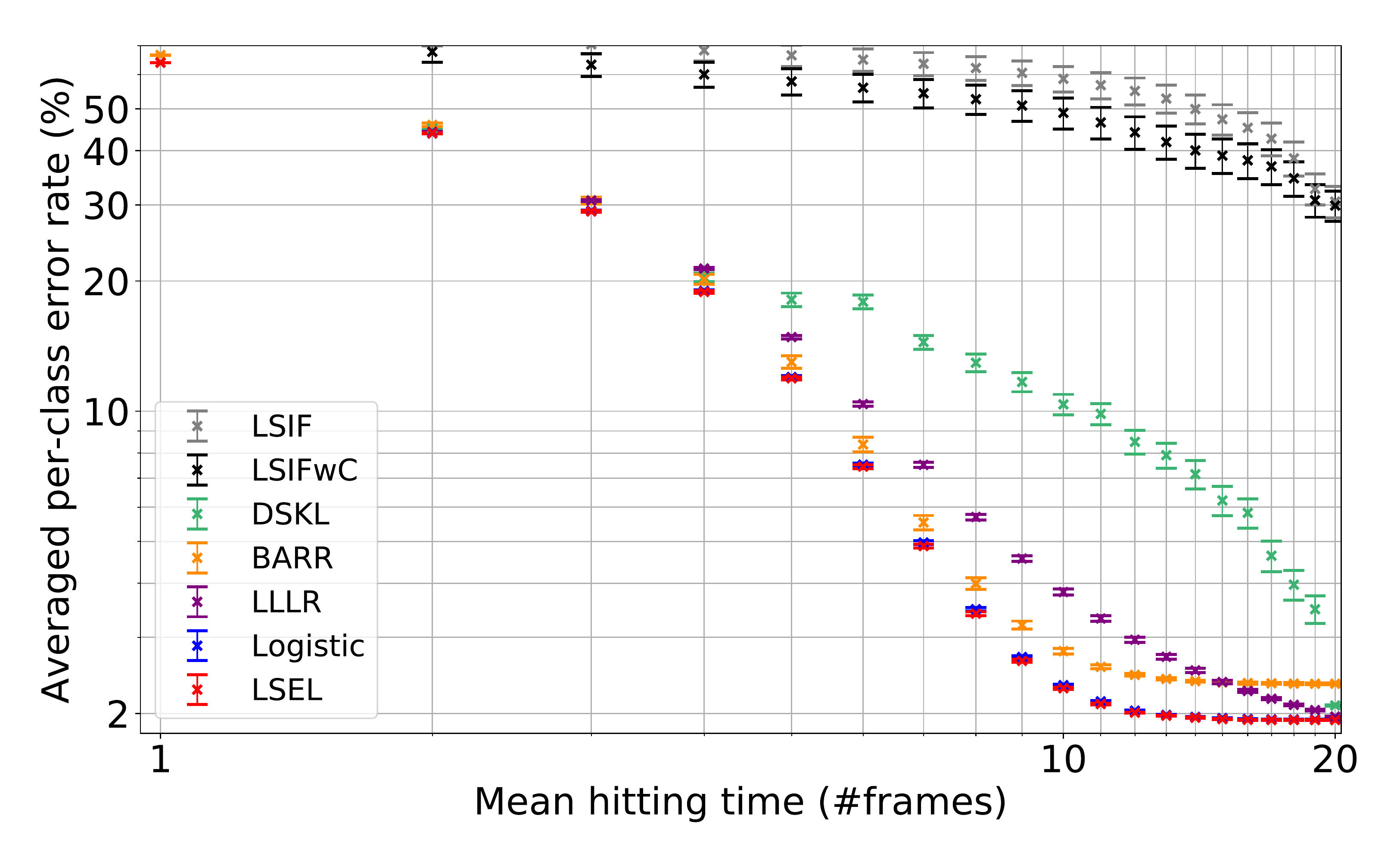}}
    \end{minipage}
    \begin{minipage}[b]{0.49\linewidth}
        \centering
        \includegraphics[width=\columnwidth, keepaspectratio]
        {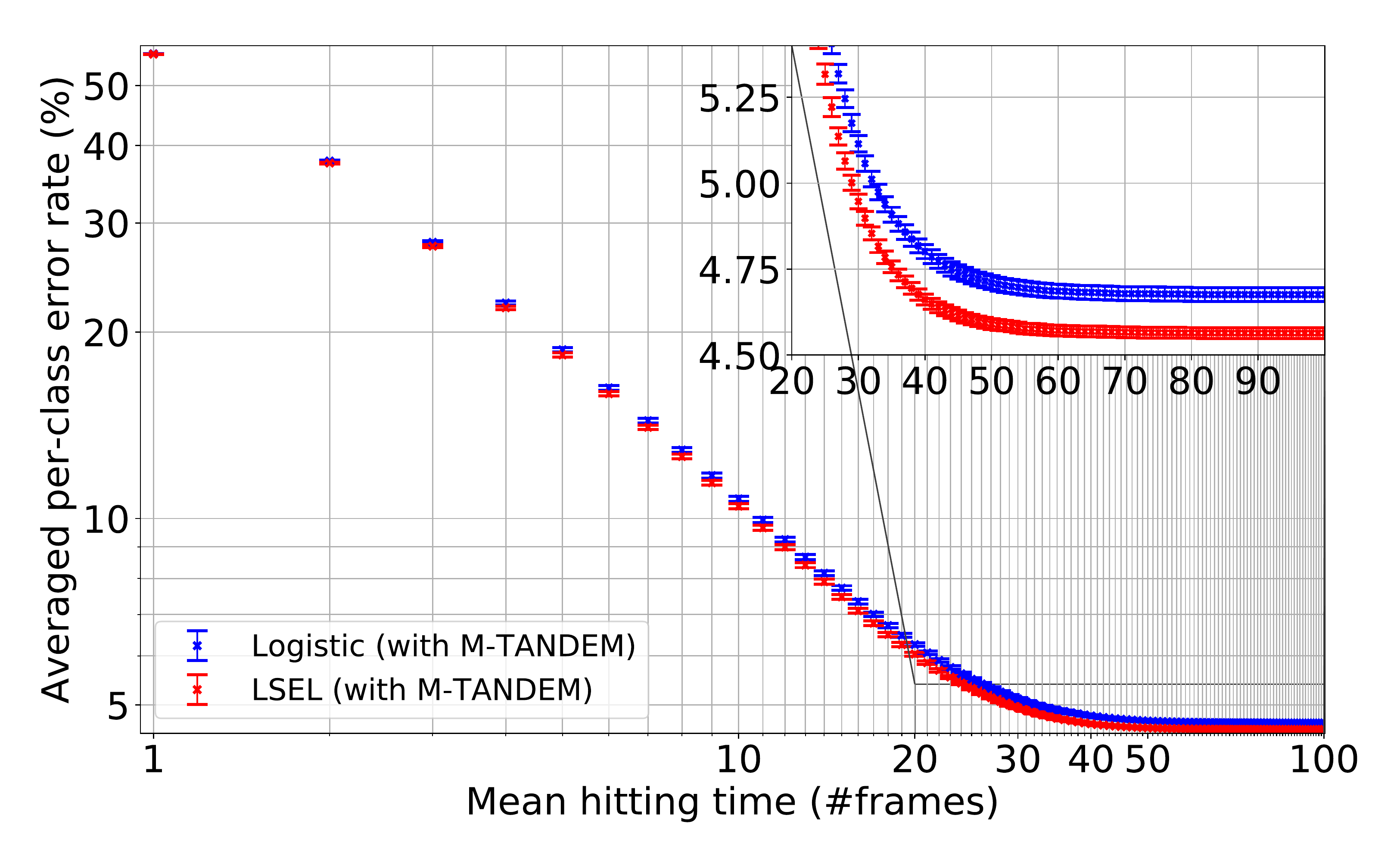} 
    \end{minipage}
    \caption{\textbf{LSEL v.s. Conventional Losses.}
        The datasets (NMNSIT-H and NMNIST-100f) are introduced in Section \ref{sec:experiment}. Curves in the lower left region are better.
        \textbf{Left: LSEL v.s. Binary DRE-based Losses on NMNSIT-H.}  The conventional losses do not generalize well in DRME.
        The MSPRT is run, using the LLR matrices estimated with seven different loss functions: LSIF \cite{kanamori2009least_LSIF_original} minimizes the mean squared error of $p$ and $\hat{r}q$ ($\hat{r}=\hat{p}/\hat{q}$); LSIFwC stabilizes LSIF by adding a normalization constraint of $\hat{r}q$; DSKL \cite{khan2019deep_DSKL_BARR_original} is based on KLIEP \cite{sugiyama2008direct_KLIEP_original} and minimizes the Kullback-Leibler divergence between $p$ and $\hat{r}q$; BARR \cite{khan2019deep_DSKL_BARR_original} stabilizes DSKL by adding the normalization constraint; LLLR \cite{SPRT-TANDEM} is similar to DSKL but is bounded above and below and is thus more stable; the logistic loss is the standard sum-log-exp-type loss; and the LSEL is our proposed loss. Their formal definitions are summarized in Appendix \ref{app: Figure fig: DRE Losses vs LSEL}. Only the logistic loss shows a comparable performance, but the LSEL is consistently better (Tables \ref{tab: DRE Losses vs LSEL (Right: NMNIST-100f, 1--50)}--\ref{tab: DRE Losses vs LSEL (Right: NMNIST-100f, 51--100)}). 
        \textbf{ Right: LSEL v.s. Logistic Loss on NMNIST-100f.} The M-TANDEM approximation is used, which is introduced in Section \ref{sec: Overall Architecture: MSPRT-TANDEM}. The error gap is statistically significant (Appendix \ref{app: Statistical Tests}).
        } \label{fig: DRE Losses vs LSEL}
\end{figure}

The MSPRT is a provably optimal algorithm for early multiclass classification and has been developed in mathematical statistics \cite{armitage1950MSPRT_TartarBook12, chernoff1959TartarBook97, kiefer1963TartarBook231, lorden1967TartarBook269, lorden1977TartarBook275_GaryMSPRT, pavlov1984TartarBook351, dragalin1987TartarBook123, pavlov1991TartarBook352, baum1994MSPRT_TartarBook50, dragalin1999TartarBook127}. 
The MSPRT uses a \textit{matrix} of log-likelihood ratios (LLRs), the $(k,l)$-entry of which is the LLR of hypothesis $H_k$ to hypothesis $H_l$ and depends on the current time $t$ through consecutive observations of sequential data $X^{(1,t)}$ (Figure \ref{fig: MSPRT}). 
A notable property of the MSPRT is that it is \textit{asymptotically optimal} \cite{tartakovsky1998TartarBook455}: It achieves the minimum stopping time among all the algorithms with bounded error probabilities as the thresholds go to infinity, or equivalently, as the error probabilities go to zero or the stopping time goes to infinity (Appendix \ref{app: MSPRT}).
Therefore, the MSPRT is a promising approach to early multiclass classification with strong theoretical support.

However, the MSPRT has a critical drawback that hinders its real-world applications in that it requires the true LLR matrix, which is generally inaccessible. To address this problem, we propose to solve \textit{density ratio matrix estimation} (DRME); i.e., we attempt to estimate the LLR matrix from a dataset. 
DRME has yet to be explored in the literature but can be regarded as a generalization of the conventional density ratio estimation (DRE), which usually focuses on only two densities \cite{sugiyama2012Density_Ratio_Estimation_in_Machine_Learning}. The difficulties with DRME come from simultaneous optimization of multiple density ratios; the training easily diverges when the denominator of only one density ratio is small. In fact, a naive application of conventional binary DRE-based loss functions does not generalize well in this setting, and sometimes causes instability and divergence of the training (Figure \ref{fig: DRE Losses vs LSEL} Left). 

Therefore, we propose a novel loss function for solving DRME, the \textit{log-sum-exp loss (LSEL)}. 
We prove three properties of the LSEL, all of which contribute to enhancing the performance of the MSPRT.
(i) The LSEL is \textit{consistent}; i.e., by minimizing the LSEL, we can obtain the true LLR matrix as the sample size of the training set increases. 
(ii) The LSEL has the \textit{hard class weighting effect}; i.e., it assigns larger gradients to harder classes, accelerating convergence of neural network training. Our proof also explains why log-sum-exp-type losses, e.g., \cite{song2016CVPR_lifted_structure_loss, wang2019CVPR_multi-similarity_loss, sun2020CVPR_circle_loss}, have performed better than sum-log-exp-type losses.
(iii) We propose the cost-sensitive LSEL for class-imbalanced datasets and prove that it is \textit{guess-averse} \cite{beijbom2014guess-averse}. Cost-sensitive learning \cite{elkan2001foundations_cost-sensitive_learning}, or loss re-weighting, is a typical and simple solution to the class imbalance problem \cite{Kubat97addressing_class_imbalance_maybeOriginal, japkowicz2002class_imbalance_systematic_study, he2009learning_class_imbalance, buda2018systematic_class_imbalance}. Although the consistency does not necessarily hold for the cost-sensitive LSEL, we show that the cost-sensitive LSEL nevertheless provides discriminative ``LLRs'' (scores) by proving its guess-aversion.

Along with the novel loss function, we propose the first DRE-based model for early multiclass classification in deep learning, \textit{MSPRT-TANDEM}, enabling the MSPRT's practical use on real-world datasets. MSPRT-TANDEM can be used for arbitrary sequential data and thus has a wide variety of potential applications.
To test its empirical performance, we conduct experiments on four publicly available datasets.
We conduct two-way analysis of variance (ANOVA) \cite{FisherBook_ANOVA} followed by the Tukey-Kramer multi-comparison test \cite{Tukey1949, Kramer1956} for reproducibility and find that MSPRT-TANDEM provides statistically significantly better accuracy with a smaller number of observations than baseline models.

Our contributions are summarized as follows.
\begin{enumerate}
    \item We formulate a novel problem setting, DRME, to enable the MSPRT on real-world datasets.
    \item We propose a loss function, LSEL, and prove its consistency, hard class weighting effect, and guess-aversion.
    \item We propose MSPRT-TANDEM: the first DRE-based model for early multiclass classification in deep learning. We show that it outperforms baseline models statistically significantly.
\end{enumerate}



%% file: Related.tex
\section{Related Work}
\paragraph{Early classification of time series.}
Early classification of time series aims to make a prediction as early and as accurately as possible \cite{Xing2009ECTSorig, Mori2015early_cost_minimization_point_of_view_NIPSw, Mori2016reliable_ECDIRE_non-deep_famous, Mori2018}.
An increasing number of real-world problems require earliness as well as accuracy, especially when a sampling cost is high or when a delay results in serious consequences; e.g.,
early detection of human actions for video surveillance and health care \cite{vats2016early}, 
early detection of patient deterioration on real-time sensor data \cite{mao2012integrated}, 
early warning of power system dynamics \cite{zhang2017intelligent_early_warning_of_power_system_dynamics}, 
and autonomous driving for early and safe action selection \cite{dona2020MSPRT_autonomous_driving}.
In addition, early classification saves computational costs \cite{amir2021FrameExit_CVPR_oral_early_exit}.

\paragraph{SPRT.}
Sequential multihypothesis testing has been developed in \cite{sobel1949TartarBook435, armitage1950MSPRT_TartarBook12, paulson1963TartarBook350, simons1967TartarBook432}. The extension of the binary SPRT to multihypothesis testing for i.i.d. data was conducted in \cite{armitage1950MSPRT_TartarBook12, chernoff1959TartarBook97, kiefer1963TartarBook231, lorden1967TartarBook269, lorden1977TartarBook275_GaryMSPRT, pavlov1984TartarBook351, dragalin1987TartarBook123, pavlov1991TartarBook352, baum1994MSPRT_TartarBook50,dragalin1999TartarBook127}.
The MSPRT for non-i.i.d. distributions was discussed in \cite{lai1981TartarBook248, tartakovsky1998TartarBook455, dragalin1999TartarBook128_MSPRT-I, tartakovsky2014TartarBook}. The asymptotic optimality of the MSPRT was proven in \cite{tartakovsky1998TartarBook455}.

\paragraph{Density ratio estimation.}
DRE consists of estimating a ratio of two densities from their samples without separately estimating the numerator and denominator \cite{sugiyama2012Density_Ratio_Estimation_in_Machine_Learning}. 
DRE has been widely used for, e.g., covariate shift adaptation \cite{sugiyama2008direct_KLIEP_original}, representation learning \cite{oord2018representation, hjelm2018learning}, mutual information estimation \cite{belghazi2018mutual}, and off-policy reward estimation in reinforcement learning \cite{liu2018breaking_off-policy_estimation}.
Our proof of the consistency of the LSEL is based on \cite{gutmann2012noise_NCE}.


We provide more extensive references in Appendix \ref{app: Supplementary Related Work}.
To the best of our knowledge, only \cite{SPRT-TANDEM} and \cite{moustakides2019training} combine the SPRT with DRE. Both restrict the number of classes to only two. The loss function proposed in \cite{SPRT-TANDEM} has not been proven to be unbiased; there is no guarantee for the estimated LLR to converge to the true one. \cite{moustakides2019training} does not provide empirical validation for the SPRT.

%% file: MSPRT-TANDEMv2.tex
\section{Density Ratio Matrix Estimation for MSPRT} \label{sec:DRME}

\subsection{Log-Likelihood Ratio Matrix}

Let $p$ be a probability density over $(X^{(1,T)}, y)$. $X^{(1,T)} = \{ \bm{x}^{(t)} \}_{t=1}^{T} \in \mc{X}$ is an example of sequential data, where $T \in \mbn$ is the sequence length. $\bm{x}^{(t)} \in \mbr^{d_x}$ is a feature vector at timestamp $t$; e.g., an image at the $t$-th frame in a video $X^{(1,T)}$. $y \in \mc{Y} = [K] := \{ 1,2,...,K \}$ is a multiclass label, where $K \in \mbn$ is the number of classes. The LLR matrix is defined as
$\la (X^{(1,t)}) 
:= ( \la_{kl} (X^{(1,t)}) )_{k,l \in [K]}
:= ( \log  p(X^{(1,t)} | y = k) / p(X^{(1,t)} | y = l)  )_{k,l \in [K]}$,
where $p(X^{(1,t)} | y)$ is a conditional probability density. $\la (X^{(1,t)})$ is an anti-symmetric matrix by definition; thus the diagonal entries are 0. Also, $\la$ satisfies $\la_{kl} + \la_{l m} = \la_{km}$ ($\forall k, l, m \in [K]$). Let 
$\hat{\la} (X^{(1,t)}; \bm{\th})
:= ( \hat{\la}_{kl} (X^{(1,t)}; \bm{\th}) )_{k,l \in [K]}
:= ( \log  \hat{p}_{\bm{\th}} ( X^{(1,t)} | y=k ) / \hat{p}_{\bm{\th}} ( X^{(1,t)} | y=l ) )_{k,l \in [K]}$ 
be an estimator of the true LLR matrix $\la (X^{(1,t)})$, where $\bm{\th} \in \mbr^{d_\th}$ ($d_\th \in \mbn$) denotes trainable parameters, e.g., weight parameters of a neural network. We use the hat symbol ($\hat{\cdot}$) to highlight that the quantity is an estimated value. The $\hat{\la}$ should be anti-symmetric and satisfy $\hat{\la}_{kl} + \hat{\la}_{lm} = \hat{\la}_{km}$ ($\forall k, l, m \in [K]$). To satisfy these constraints, one may introduce additional regularization terms to the objective loss function, which can cause learning instability. Instead, we use specific combinations of the posterior density ratios $\hat{p}_{\bm{\th}} ( y=k | X^{(1,t)} ) / \hat{p}_{\bm{\th}} ( y=l | X^{(1,t)} )$, which explicitly satisfy the aforementioned constraints (see the following M-TANDEM and M-TANDEMwO formulae).

\subsection{MSPRT}

Formally, the MSPRT is defined as follows (see Appendix \ref{app: MSPRT} for more details):
\begin{definition}[Matrix sequential probability ratio test] \label{def: MSPRT}
    Let $P$ and $P_k$ ($k \in [K]$) be probability distributions.
    Define a threshold matrix $a_{kl} \in \mbr$ ($k,l \in [K]$), where the diagonal elements are immaterial and arbitrary, e.g., 0. The MSPRT of multihypothesis $H_k: P = P_k$ ($k \in [K]$) is defined as $\de^* := (d^*, \ta^*)$, where $d^* := k$ if $\ta^* = \ta_k$ ($k \in [K]$),
    $\ta^* := \min \{ \ta_k | k \in [K] \}$, and
    $\ta_k := \inf \{ 
        t \geq 1 
        | \underset{l (\neq k) \in [K]}{\min} \{
            \la_{kl}( X^{(1,t)} ) - a_{lk}
        \} \geq 0
    \}$.    
\end{definition}
In other words, the MSPRT terminates at the smallest timestamp $t$ such that for a class of $k \in [K]$, $\la_{kl}(t)$ is greater than or equal to the threshold $a_{lk}$ for all $l (\neq k)$ (Figure \ref{fig: MSPRT}). By definition, we must know the true LLR matrix $\la (X^{(1,t)})$ of the incoming time series $X^{(1,t)}$; therefore, we estimate $\la$ with the help of the LSEL defined in the next section. For simplicity, we use single-valued threshold matrices ($a_{kl} = a_{k^\prime l^\prime}$ for all $k,l,k^\prime,l^\prime \in [K]$) in our experiment.

\subsection{LSEL for DRME} \label{sec: LSEL for DRME}
To estimate the LLR matrix, we propose the \textit{log-sum-exp loss} (LSEL): 
\begin{align} \label{eq:LLLR-E}
    L_{\rm \text{LSEL}} [\tilde{\la}] 
    := \f{1}{KT} \sum_{k \in [K]} \sum_{t \in [T]} \int d X^{(1,t)} p ( X^{(1,t)} | k )
    \log ( 
        1 + \sum_{ l ( \neq k ) } e^{ - \tilde{\la}_{k l} ( X^{(1,t)} ) }
    ) \, .
\end{align}
Let $S := \{ (X_i^{(1,T)}, y_i) \}_{i=1}^{M} \sim p(X^{(1,T)}, y)^M$ be a training dataset, where $M \in \mbn$ is the sample size. The empirical approximation of the LSEL is
\begin{align} \label{eq:LLLR-E with empirical approx.}
    \hat{L}_{\rm \text{LSEL}} (\bm{\th}; S) :=
     \f{1}{KT} \sum_{k \in [K]} \sum_{t \in [T]} \f{1}{M_k} \sum_{i \in I_k}  \log ( 
        1 + \sum_{ l ( \neq k ) } e^{ - \hat{\la}_{k l} ( X_i^{(1,t)}; \bm{\th} ) }
    ) \, .
\end{align}
$M_k$ and $I_k$ denote the sample size and index set of class $k$, respectively; i.e., $M_k = |\{ i \in [M] | y_i = k \}| = |I_k|$ and $\sum_k M_k = M$. 

\subsubsection{Consistency}
A crucial property of the LSEL is \textit{consistency}; therefore, by minimizing (\ref{eq:LLLR-E with empirical approx.}), the estimated LLR matrix $\hat{\la}$ approaches the true LLR matrix $\la$ as the sample size increases. The formal statement is given as follows:
\begin{theorem}[Consistency of the LSEL]\label{thm:consistency of LSEL}
    Let $L(\bm{\th})$ and $\hat{L}_S (\bm{\th})$ denote $L_{\rm \text{LSEL}} [\hat{\la} (\cdot; \bm{\th})]$ and $\hat{L}_{\rm \text{LSEL}} (\bm{\th} ;S)$ respectively. 
    Let $\hat{\bm{\th}}_S$ be the empirical risk minimizer of $\hat{L}_S$; namely, $\hat{\bm{\th}}_S := \mr{argmin}_{\bm{\th}} \hat{L}_S (\bm{\th})$.
    Let $\Th^* := \{ \bm{\th}^* \in \mbr^{d_{\th}} | \hat{\la} ( X^{(1,t)} ; \bm{\th}^* ) = \la (X^{(1,t)}) \,\, (\forall t \in [T]) \}$ be the target parameter set. Assume, for simplicity of proof, that each $\bm{\th}^*$ is separated in $\Th^*$; i.e., $\exists \de > 0$ such that $B ( \bm{\th}^* ; \de ) \cap B ( \bm{\th}^{*\prime}; \de ) = \emptyset$ for arbitrary $\bm{\th}^*$ and $\bm{\th}^{*\prime} \in \Th^*$, where $B(\bm{\th}; \de)$ denotes an open ball at center $\bm{\th}$ with radius $\de$. Assume the following three conditions:
    \begin{itemize}
        \item[(a)] $\forall k, l \in [K]$, $\forall t \in [T]$ , $p(X^{(1,t)} | k) = 0 \Longleftrightarrow p(X^{(1,t)} | l) = 0$.
        
        \item[(b)] $\mr{sup}_{\bm{\th}} | \hat{L}_S (\bm{\th}) - L (\bm{\th}) | \xrightarrow[M\rightarrow\infty]{P} 0$; i.e., $\hat{L}_S (\bm{\th})$ converges in probability uniformly over $\bm{\th}$ to $L(\bm{\th})$.

        \item[(c)] For all $\th^* \in \Th^*$, there exist $t \in [T]$, $k\in[K]$ and $l\in[K]$, such that the following $d_\th \times d_\th$ matrix is full-rank:
        \begin{align}
            \int d X^{(1,t)} p( X^{(1,t)} | k )
            \nabla_{\bm{\th}^*} \hat{\la}_{k l}(X^{(1,t)}; \bm{\th}^*)
            \nabla_{\bm{\th}^*} \hat{\la}_{k l}(X^{(1,t)}; \bm{\th}^*)^{\top} \, .
        \end{align}
        
    \end{itemize}
    Then,  $P( \hat{\bm{\th}}_S \notin \Th^* ) \xrightarrow[]{M\rightarrow\infty}0$; i.e., $\hat{\bm{\th}}_S$ converges in probability into $\Th^*$.
\end{theorem}
Assumption (a) ensures that $\la(X^{(1,t)})$ exists and is finite. Assumption (b) can be satisfied under the standard assumptions of the uniform law of large numbers (compactness, continuity, measurability, and dominance) \cite{jennrich1969ULLN, newey1986ULLN}. Assumption (c) is a technical requirement, often assumed in the literature \cite{gutmann2012noise_NCE}. The complete proof is given in Appendix \ref{app:Proof of Consistency Theorem of LSEL}.

The critical hurdle of the MSPRT to practical applications (availability to the true LLR matrix) is now relaxed by virtue of the LSEL, which is provably consistent and enables a precise estimation of the LLR matrix. 
We emphasize that the MSPRT is the earliest and most accurate algorithm for early classification of time series, at least asymptotically (Theorem \ref{thm: Asymptotic optimality with first moment condition}, \ref{thm: Asymptotic optimality with second moment condition}, and \ref{thm: Asymptotic optimality of MSPRT (non-i.i.d.)}). 

\subsubsection{Hard Class Weighting Effect} \label{sec: Hard class weighting effect of the LSEL}
We further discuss the LSEL by focusing on a connection with hard negative mining \cite{song2016CVPR_lifted_structure_loss}. It is empirically known that designing a loss function to emphasize hard classes improves model performance \cite{lin2017ICCV_focal_loss}. The LSEL has this mechanism.

Let us consider a multiclass classification problem to obtain a high-performance discriminative model. To emphasize hard classes, let us minimize $\hat{L} := \f{1}{KT} \sum_{k \in [K]} \sum_{t \in [T]} \f{1}{M_k} \sum_{i \in I_k} \max_{l (\neq y_i)} \{ e^{ - \hat{\la}_{y_i l} (X^{(1,t)}_i; \bm{\th}) } \}$; however, mining the single hardest class with the max function induces a bias and causes the network to converge to a bad local minimum. Instead of $\hat{L}$, we can use the LSEL because it is not only provably consistent but is a smooth upper bound of $\hat{L}$: Because
$\max_{l (\neq y_i)} \{ e^{ - \hat{\la}_{y_i l} (X^{(1,t)}_i; \bm{\th} ) } \}
    < \sum_{l (\neq y_i)} e^{ - \hat{\la}_{y_i l} (X^{(1,t)}_i; \bm{\th} ) } $,
we obtain $\hat{L} < \hat{L}_{\rm LSEL}$ by summing up both sides with respect to $i \in I_k$ and then $k \in [K]$ and $t \in [T]$. Therefore, a small $\hat{L}_{\rm LSEL}$ indicates a small $\hat{L}$. In addition, 
the gradients of the LSEL are dominated by the hardest class $k^* \in {\rm argmax}_{k (\neq y)} \{ e^{ - \hat{\la}_{y k} (X^{(1,T)}; \bm{\th}) } \}$, because 
\begin{align}
    \lt| \f{\partial \hat{L}_{\rm LSEL}}{\partial \hat{\la}_{y k}} \rt|
        \propto \f{
             e^{ - \hat{\la}_{y k} }
        }{
            \sum_{l \in [K] }  e^{ - \hat{\la}_{y l} }
        }
    < \f{
             e^{ - \hat{\la}_{y k^*} }
        }{
            \sum_{l \in [K] }  e^{ - \hat{\la}_{y l} }
        } 
    \propto \lt| \f{\partial \hat{L}_{\rm LSEL}}{\partial \hat{\la}_{y k^*}} \rt| 
    \hspace{5pt} ( \forall k (\neq y, k^*) ) \, , \no
\end{align}
meaning that the LSEL assigns large gradients to the hardest class during training, which accelerates convergence.


Let us compare the hard class weighting effect of the LSEL with that of the logistic loss (a sum-log-exp-type loss extensively used in machine learning).
For notational convenience, let us define $\ell_{\rm LSEL} := \log (1 + \sum_{k (\neq y)} e^{a_k})$ and $\ell_{\rm logistic} := \sum_{k (\neq y)} \log (1 + e^{a_k})$, where $a_k := -\hat{\la}_{yk} ( X^{(1,t)} ; \bm{\th})$, and compare their gradient scales.
The gradients for $k \neq y$ are:
\begin{equation}
    \f{\partial \ell_{\rm logistic}}{\partial \hat{\la}_{y k}} 
    = - \f{
        e^{ - \hat{\la}_{y k} }
    }{
        1 +  e^{ - \hat{\la}_{y k} }
    } =: b_k 
    \hspace{20pt} \text{ and } \hspace{20pt}
    \f{\partial \ell_{\rm LSEL}}{\partial \hat{\la}_{y k}} 
    = - \f{
        e^{ - \hat{\la}_{y k} }
    }{
        \sum_{l \in [K] } e^{ - \hat{\la}_{y l} }
    } =: c_k \, . \no
\end{equation}
The relative gradient scales of the hardest class to the easiest class are:
\begin{align}
    R_{\rm logistic} := \f{
        \max_{k (\neq y)} \{ b_k \}
    }{
        \min_{k (\neq y)} \{ b_k \}
    } = \f{e^{a_{k^*}}}{e^{a_{k_*}}} \f{e^{1 + a_{k_*}}}{e^{1 + a_{k^*}}} \, , &
    \hspace{10pt} R_{\rm LSEL} := \f{
        \max_{k (\neq y)} \{ c_k \}
    }{
        \min_{k (\neq y)} \{ c_k \}
    } = \f{e^{a_{k^*}}}{e^{a_{k_*}}} \, , \no
\end{align}
where $k_* := {\rm argmin}_{k (\neq y)} \{ a_k \}$. Since $R_{\rm logistic} \leq R_{\rm LSEL}$, we conclude that the LSEL weighs hard classes more than the logistic loss. Note that our discussion above also explains why log-sum-exp-type losses (e.g., \cite{song2016CVPR_lifted_structure_loss, wang2019CVPR_multi-similarity_loss, sun2020CVPR_circle_loss}) perform better than sum-log-exp-type losses.
In addition, Figure \ref{fig: DRE Losses vs LSEL} (Left and Right) shows that the LSEL performs better than the logistic loss---a result that supports the discussion above. See Appendix \ref{app: Performance Comparison of LSELwith logistic} for more empirical results.

\subsubsection{Cost-Sensitive LSEL and Guess-Aversion} \label{app: Cost-Sensive LSEL and Guess-Aversion}
\input{MSPRT-TANDEMv2_snip_Guess-Aversion}

\subsection{MSPRT-TANDEM} \label{sec: Overall Architecture: MSPRT-TANDEM}
\begin{figure}[t]
    \begin{center}
    \centerline{\includegraphics[width=\columnwidth, keepaspectratio]
    {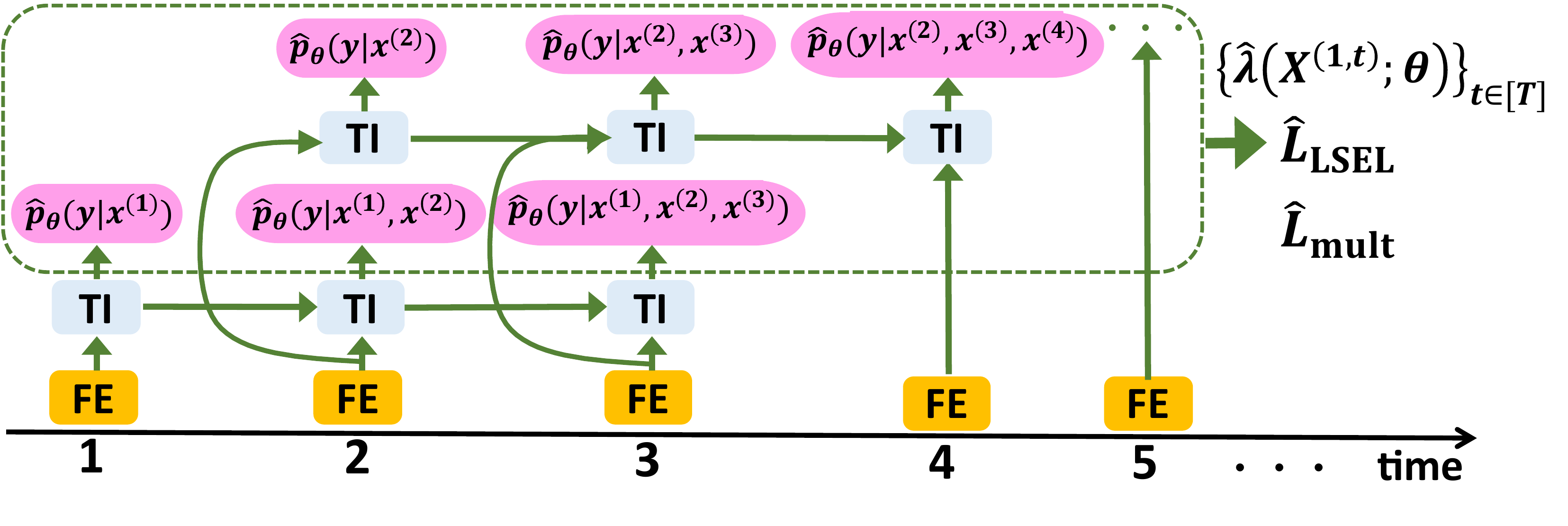}}
    \caption{
        \textbf{MSPRT-TANDEM ($\bm{N=2}$).} 
        $x^{(t)}$ is an input vector; e.g., a video frame. FE is a feature extractor. TI is a temporal integrator, which allows two inputs: the feature vector and a hidden state vector, which encodes the information of the past frames. We use ResNet and LSTM for FE and TI, respectively, but are not limited to them in general. The output posterior densities are highlighted with pink circles. By aggregating the posterior densities, the multiplet loss is calculated. Also, the estimated LLR matrix $\hat{\la}$ is constructed using the M-TANDEM or M-TANDEMwO formulae. Finally, $\hat{\la}$ is input to the LSEL. $\hat{L}_{\rm LSEL} + \hat{L}_{\rm mult}$ is optimized with gradient descent. In the test phase, $\hat{\la} (X^{(1,t)})$ is used to execute the MSPRT (Figure \ref{fig: MSPRT} and Definition \ref{def: MSPRT}). 
        } \label{fig: MSPRT-TANDEM}
    \end{center}
\end{figure}

Although the LSEL alone works well, 
we further combine the LSEL with a DRE-based model, SPRT-TANDEM, recently proposed in \cite{SPRT-TANDEM}. 
Specifically, we use the \textit{TANDEM formula} and \textit{multiplet loss} to accelerate the convergence. 
The TANDEM formula transforms the output of the network ($\hat{p}(y | X^{(1,t)})$) to the likelihood $\hat{p}(X^{(1,t)} | y)$ under the $N$-th order Markov approximation, which avoids the gradient vanishing of recurrent neural networks \cite{SPRT-TANDEM}:
\begin{align}
    \hspace{15pt} \hat{\la}_{kl} (X^{(1,t)}) 
    \fallingdotseq \sum_{s=N+1}^{t} \log \left( 
        \frac{
            \hat{p}_{\bm{\th}} (k | X^{(s-N, s)}) 
        }{
            \hat{p}_{\bm{\th}} (l | X^{(s-N, s)}) 
        }
    \right)
    - \sum_{s=N+2}^{t} \log \left(
            \frac{
                \hat{p}_{\bm{\th}} (k | X^{(s-N, s-1)}) 
            }{
                \hat{p}_{\bm{\th}} (l | X^{(s-N, s-1)}) 
            }
        \right)
    \label{eq: M-TANDEM formula} \, ,
\end{align}
where we do not use the prior ratio term $- \log ( \hat{p}(k) / \hat{p}(l) ) = - \log (M_k / M_l)$ in our experiments because it plays a similar role to the cost matrix \cite{menon2021ICLR_long-tail_logit_adjustment}. Note that (\ref{eq: M-TANDEM formula}) is a generalization of the original to DRME, and thus we call it the \textit{M-TANDEM formula}.

However, we find that the M-TANDEM formula contains contradictory gradient updates caused by the middle minus sign. Let us consider an example $z_i := (X^{(1,t)}_i, y_i)$. The posterior $\hat{p}_{\bm{\th}} (y=y_i | X_i^{(s-N, s-1)})$ (appears in (\ref{eq: M-TANDEM formula})) should take a \textit{large} value for $z_i$ because the posterior density represents the probability that the label is $y_i$. For the same reason, $\hat{\la}_{y_i l} (X^{(1,t)})$ should take a high value; thus $\hat{p}_{\bm{\th}} (y=y_i | x^{(s-N)}, ...,x^{(s-1)})$ should take a \textit{small} value in accordance with (\ref{eq: M-TANDEM formula}) --- an apparent contradiction. These contradictory updates may cause a conflict of gradients and slow the convergence of training, leading to performance deterioration. Therefore, in the experiments, we use either (\ref{eq: M-TANDEM formula}) or another approximation formula:
$\hat{\la}_{kl} (X^{(1,t)}; \bm{\th}) 
\fallingdotseq \log ( 
    \hat{p}_{\bm{\th}} (k | X^{(t-N, t)}) 
    /
    \hat{p}_{\bm{\th}} (l | X^{(t-N, t)}) 
)$, which we call the M-TANDEM with Oblivion (M-TANDEMwO) formula. Clearly, the gradient does not conflict. Note that both the M-TANDEM and M-TANDEMwO formulae are anti-symmetric and do not violate the constraint $\hat{\la}_{kl} + \hat{\la}_{lm} = \hat{\la}_{km}$ ($\forall k,l,m \in [K]$). See Appendix \ref{app: TANDEM vs TANDEMwO} for a more detailed empirical comparison. Finally, the multiplet loss is a cross-entropy loss combined with the $N$-th order approximation:
$\hat{L}_{\textrm{mult}} (\bm{\th} ; S) :=  
\f{1}{M} \sum_{i=1}^{M} \sum_{k=1}^{N+1} \sum_{ t = k }^{ T - ( N + 1 - k ) }
( - \log \hat{p}_{\bm{\th}} ( y_i | X_i^{(t - k + 1 , t )} )
)$. An ablation study of the multiplet loss and the LSEL is provided in Appendix \ref{app: Ablation Study: Multiplet Loss and LSEL}.

The overall architecture, \textit{MSPRT-TANDEM}, is illustrated in Figure \ref{fig: MSPRT-TANDEM}. 
MSPRT-TANDEM can be used for \textit{arbitrary} sequential data and thus has a wide variety of potential applications, such as computer vision, natural language processing, and signal processing. We focus on vision tasks in our experiments.

Note that in the training phase, MSPRT-TANDEM does not require a hyperparameter that controls the speed-accuracy tradeoff. A common strategy in early classification of time series is to construct a model that optimizes two cost functions: one for earliness and the other for accuracy \cite{dachraoui2015non-myopic_ECTS_MLP_Kmeans, Mori2015early_cost_minimization_point_of_view_NIPSw, tavenard2016cost-aware_ECTS, Mori2018, martinez2020EARLIEST-like_RL-based_ECTS}. This approach
typically requires a hyperparameter that controls earliness and accuracy \cite{achenchabe2020early_cost-based_optimization_criterion}. The tradeoff hyperparameter is determined by heuristics and cannot be changed after training. However, MSPRT-TANDEM does not require such a hyperparameter and enables us to control the speed-accuracy tradeoff \textit{after training} because we can change the threshold of MSPRT-TANDEM without retraining. This flexibility is an advantage for efficient deployment \cite{Cai2020ICLROnce-for-all_efficient_deploy}.

%% file: MSPRT-TANDEMv2_snip_Guess-Aversion.tex
Furthermore, we prove that the cost-sensitive LSEL provides discriminative scores even on imbalanced datasets. 
Conventional research for cost-sensitive learning has been mainly focused on binary classification problems \cite{fan1999adacost_binaryCSL, elkan2001foundations_cost-sensitive_learning, viola2001NIPS_binaryCSL, masnadi2010cost_binaryCSL}.
However, in \textit{multiclass} cost-sensitive learning, \cite{beijbom2014guess-averse} proved that random score functions (a ``random guess'') can lead to even smaller values of the loss function.
Therefore, we should investigate whether our loss function is averse (robust) to such random guesses, i.e., \textit{guess-averse}.



\paragraph{Definitions}
Let $\bm{s}: \mcx \rightarrow \mbr^K$ be a score vector function; i.e., $s_k (X^{(1,t)})$ represents how likely it is that $X^{(1,t)}$ is sampled from class $k$. In the LSEL, we can regard $\log \hat{p}_{\bm{\th}} (X^{(1,t)} | k)$ as $s_k(X^{(1,t)})$.
A cost matrix $C$ is a matrix on $\mbr^{K \times K}$ such that 
$C_{kl} \geq 0$ ($\forall k,l \in [K]$),
$C_{kk} = 0$ ($\forall k \in [K]$),
$\sum_{l \in [K]} C_{kl} \neq 0$  ($\forall k \in [K]$).
$C_{kl}$ represents a misclassification cost, or a weight for the loss function, when the true label is $k$ and the prediction is $l$. 
The \textit{support set} of class $k$ is defined as $\mc{S}_k := \{ \bm{v} \in \mbr^K \, | \,\, \forall l (\neq k), \,\,  v_k > v_l \}$.
Ideally, discriminative score vectors should be in $\mc{S}_k$ when the label is $k$.
In contrast, the \textit{arbitrary guess set} is defined as $\mc{A} := \{ \bm{v} \in \mbr^K \, | \, v_1 = v_2 = ... = v_K \}$.
If $\bm{s}(X^{(1,t)}_i) \in \mc{A}$, we cannot gain any information from $X^{(1,t)}_i$; therefore, well-trained discriminative models should avoid such an \textit{arbitrary guess} of $\bm{s}$.
We consider a class of loss functions such that $\ell (\bm{s}(X^{(1,t)}), y ; C)$: It depends on $X^{(1,t)}$ through the score function $\bm{s}$.
The loss $\ell (\bm{s}(X^{(1,t)}), y ; C)$ is \textit{guess-averse}, if for any $k \in [K]$, any $\bm{s} \in \mc{S}_k$, any $\bm{s}^\prime \in \mc{A}$, and any cost matrix $C$, $\ell (\bm{s}, k ; C) < \ell (\bm{s}^\prime, k ; C)$; thus, the guess-averse loss can provide discriminative scores by minimizing it.
The empirical loss $\hat{L} = \f{1}{MT} \sum_{i=1}^M \sum_{t=1}^T \ell (\bm{s}(X^{(1,t)}_i), y_i; C )$ is said to be guess-averse, if $\ell$ is guess-averse. 
The guess-aversion trivially holds for most binary and multiclass loss functions but does not generally hold for cost-sensitive multiclass loss functions due to the complexity of multiclass decision boundaries \cite{beijbom2014guess-averse}.

\paragraph{Cost-sensitive LSEL is guess-averse.} 
We define a cost-sensitive LSEL:
\begin{align}
    &\hat{L}_{\operatorname{CLSEL}} (\bm{\th}, C; S)
    := \f{1}{MT} \sum_{i=1}^M \sum_{t=1}^T C_{y_i} \log ( 
            1 + \sum_{l (\neq y_i)} e^{ - \hat{\la}_{y_i l} (X^{(1,t)}_i; \bm{\th}) }
        ) \label{eq: CLSEL} \, ,
\end{align}
where $C_{kl} = C_k$ ($\forall k,l \in [K]$). Note that $\hat{\la}$ is no longer an unbiased estimator of the true LLR matrix; i.e., $\hat{\la}$ does not necessarily converge to $\la$ as $M \rightarrow \infty$, except when $C_{k} = M/M_k (K-1)$ ($\hat{L}_{\rm CLSEL}$ reduces to $\hat{L}_{\rm \text{LSEL}}$). Nonetheless, the following theorem shows that $\hat{L}_{\rm CLSEL}$ is guess-averse. The proof is given in Appendix \ref{app: Proof of Theorem thm: CLSEL is guess-averse}.
\begin{theorem} \label{thm: CLSEL is guess-averse}
    $\hat{L}_{\operatorname{CLSEL}}$ is guess-averse, provided that the log-likelihood vector 
    \begin{align}
        \Big( \log \hat{p}_{\bm{\th}} (X^{(1,t)} | y=1), \log \hat{p}_{\bm{\th}} (X^{(1,t)} | y=2),
        \,...,\, \log \hat{p}_{\bm{\th}} (X^{(1,t)} | y=K) ) \Big)^\top
        \in \mbr^{K} \no
    \end{align}
    is regarded as the score vector $\bm{s}(X^{(1,t)})$.
\end{theorem}

Figure \ref{fig: LSEL vs NGA-LSEL} illustrates the risk of non-guess-averse losses.
We define the non-guess-averse LSEL (NGA-LSEL) as $\ell (\bm{s}, y; C) = \sum_{k (\neq y)} C_{y,l} \log (1 + \sum_{l (\neq k)} e^{s_l - s_k})$.
It is inspired by a variant of an exponential loss $\ell (\bm{s}, y; C) = \sum_{k,l \in [K]} C_{y,l} e^{s_l - s_k}$, which is proven to be classification calibrated but is not guess-averse \cite{beijbom2014guess-averse}.
The NGA-LSEL benefits from the log-sum-exp structure but is not guess-averse (Appendix \ref{app: NGA-LSEL Is Not Guess-Averse}), unlike the LSEL.

\begin{figure}[ht]
    \centering
    \begin{minipage}[b]{0.45\linewidth}
        \centering
        \includegraphics[width=\columnwidth, keepaspectratio]
            {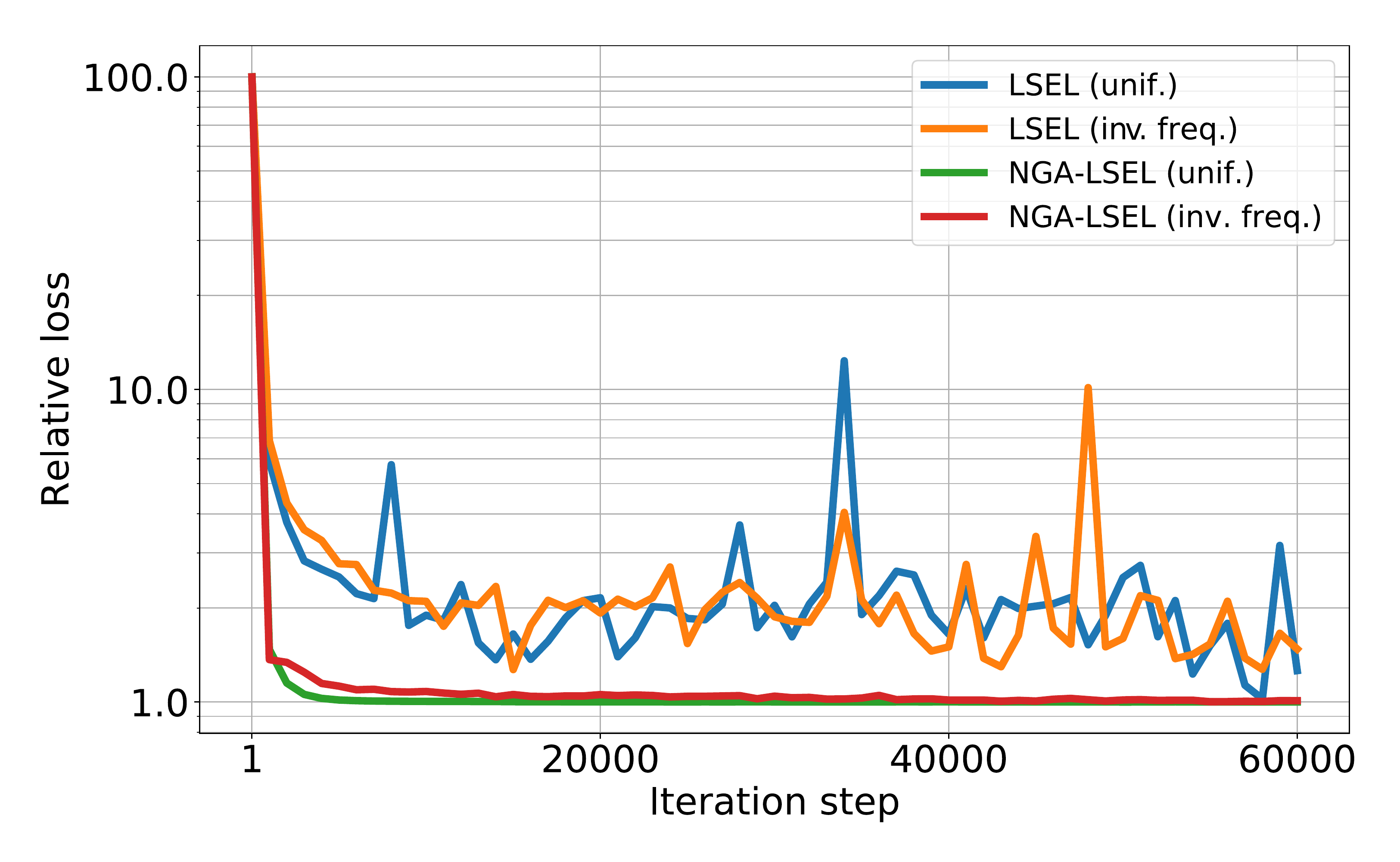} 
    \end{minipage}
    \begin{minipage}[b]{0.45\linewidth}
        \centering
        \includegraphics[width=\columnwidth, keepaspectratio]
        {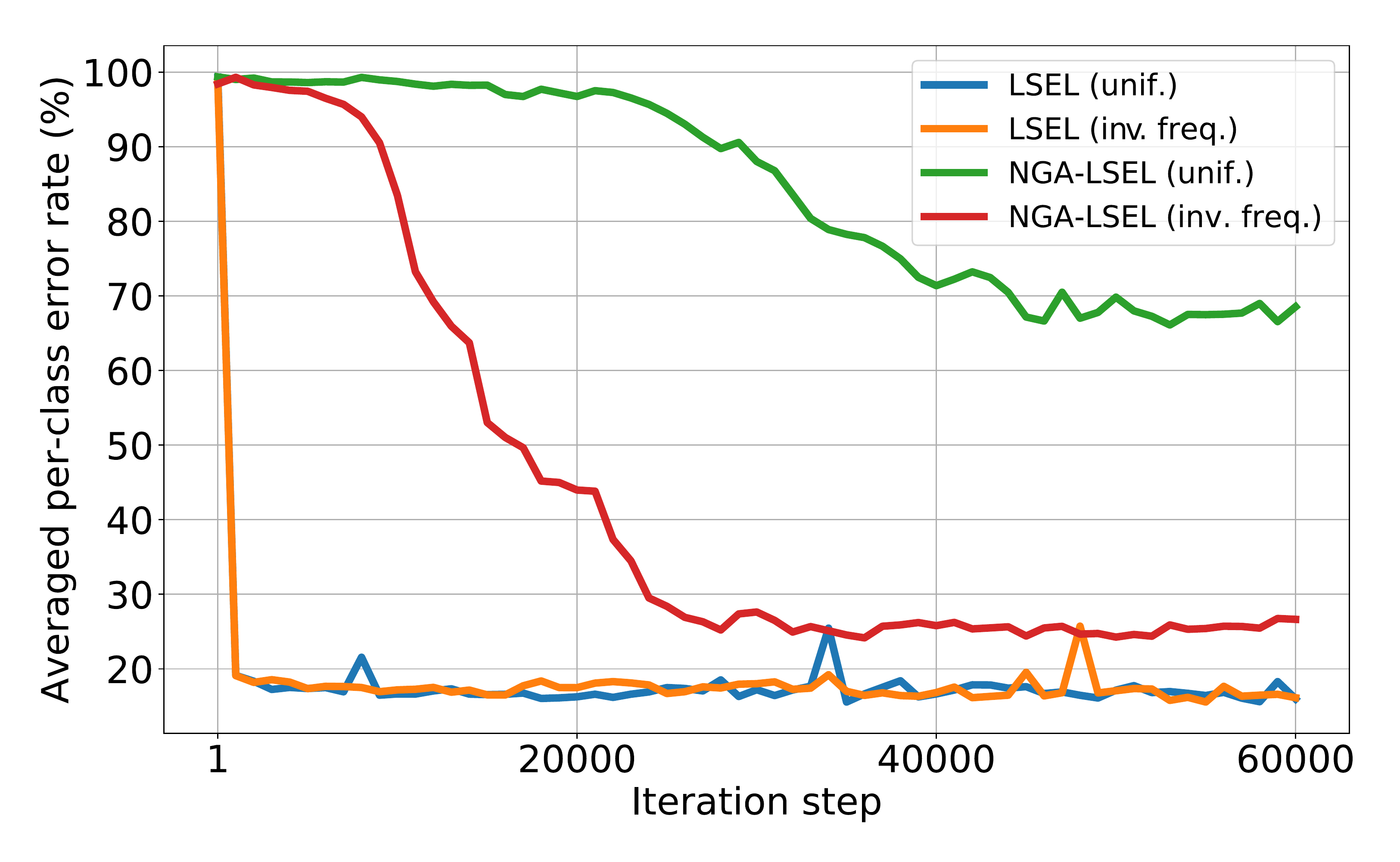} 
    \end{minipage}
    \caption{
        \textbf{Left: Relative Loss v.s. Training Iteration of LSEL and NGA-LSEL with Two Cost Matrices for Each. Right: Averaged Per-Class Error Rate of Last Frame v.s. Training Iteration.} Although all the loss curves decrease and converge (left), the error rates of the NGA-LSEL converge slowly and show a large gap depending on the cost matrix, while the error rates of the LSEL converge rapidly, and the gap is small (right). ``unif.'' means $C_{kl} = 1$ and ``inv. freq.'' means $C_{kl} = 1/M_{k}$. The dataset is UCF101 \cite{UCF101}.
        }
    \label{fig: LSEL vs NGA-LSEL}
\end{figure}

%% file: Experimentv2.tex
\section{Experiment} \label{sec:experiment}
\begin{figure*}[t]
    \begin{minipage}[b]{0.5\linewidth}
        \centering
        \includegraphics[width=\columnwidth, keepaspectratio]
        {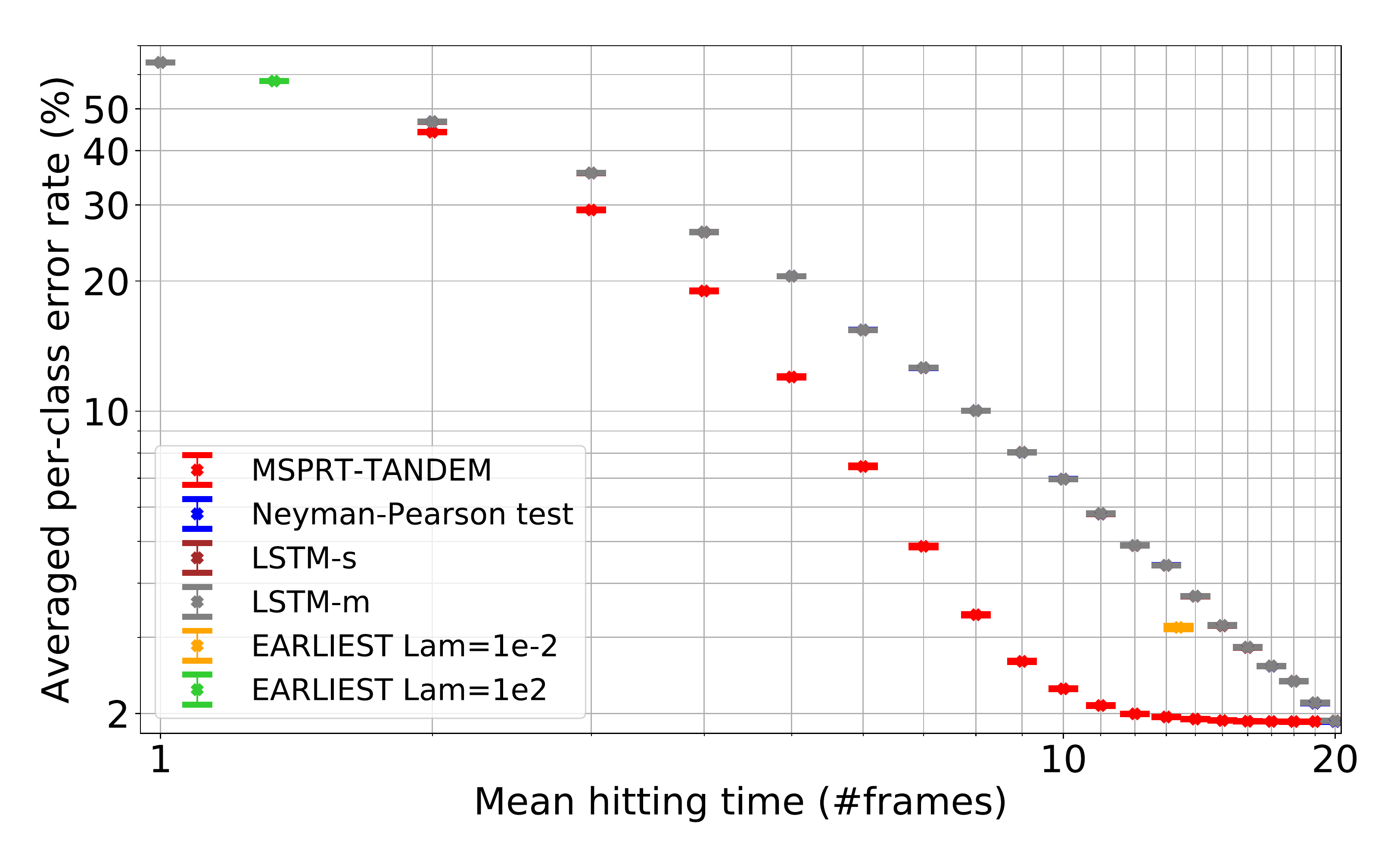} 
    \end{minipage}
    \begin{minipage}[b]{0.5\linewidth}
        \centering
        \includegraphics[width=\columnwidth, keepaspectratio]
        {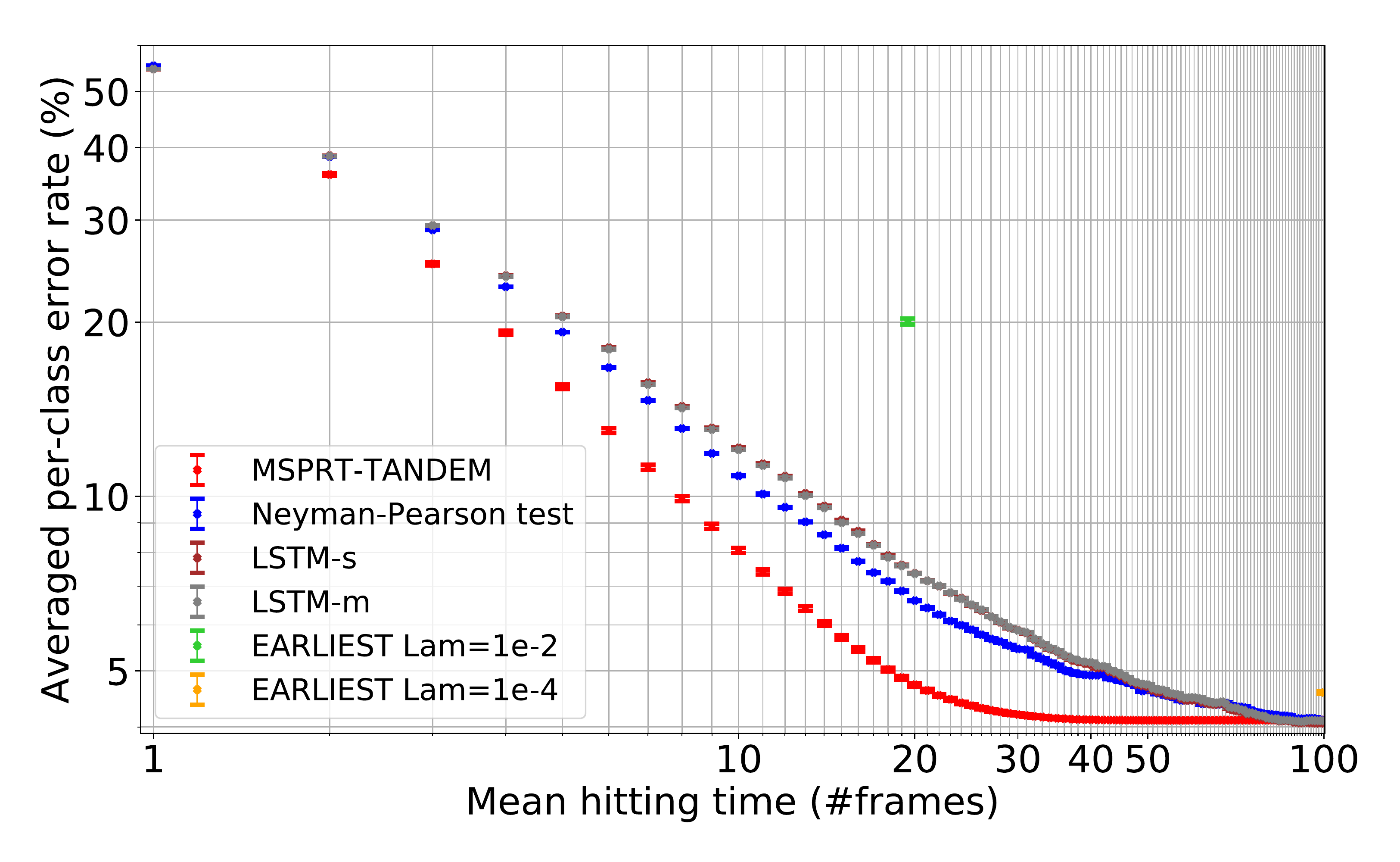} 
    \end{minipage} \\
    \begin{minipage}[b]{0.5\linewidth}
        \centering
        \includegraphics[width=\columnwidth, keepaspectratio]
        {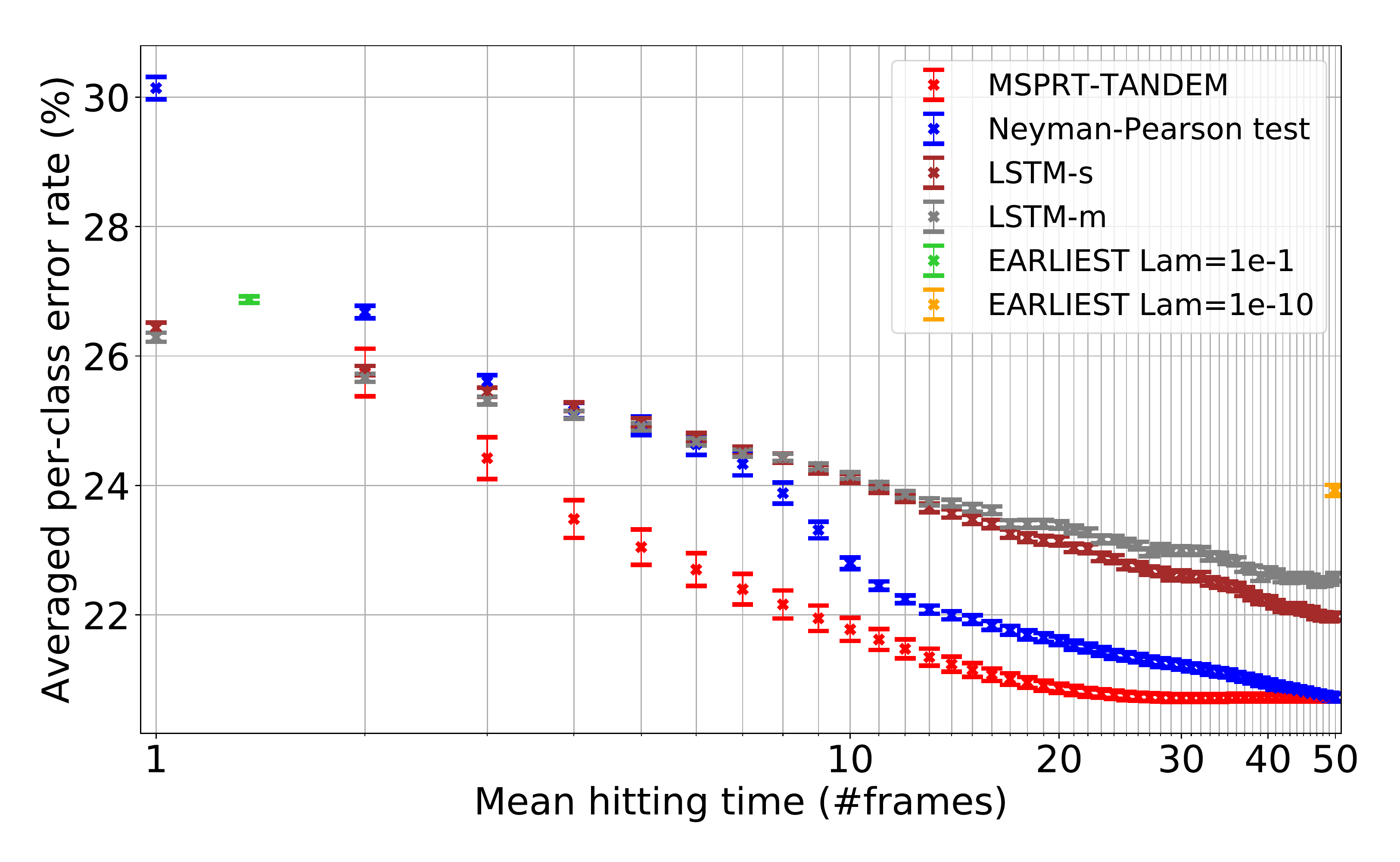} 
    \end{minipage}
    \begin{minipage}[b]{0.5\linewidth}
        \centering
        \includegraphics[width=\columnwidth, keepaspectratio]
        {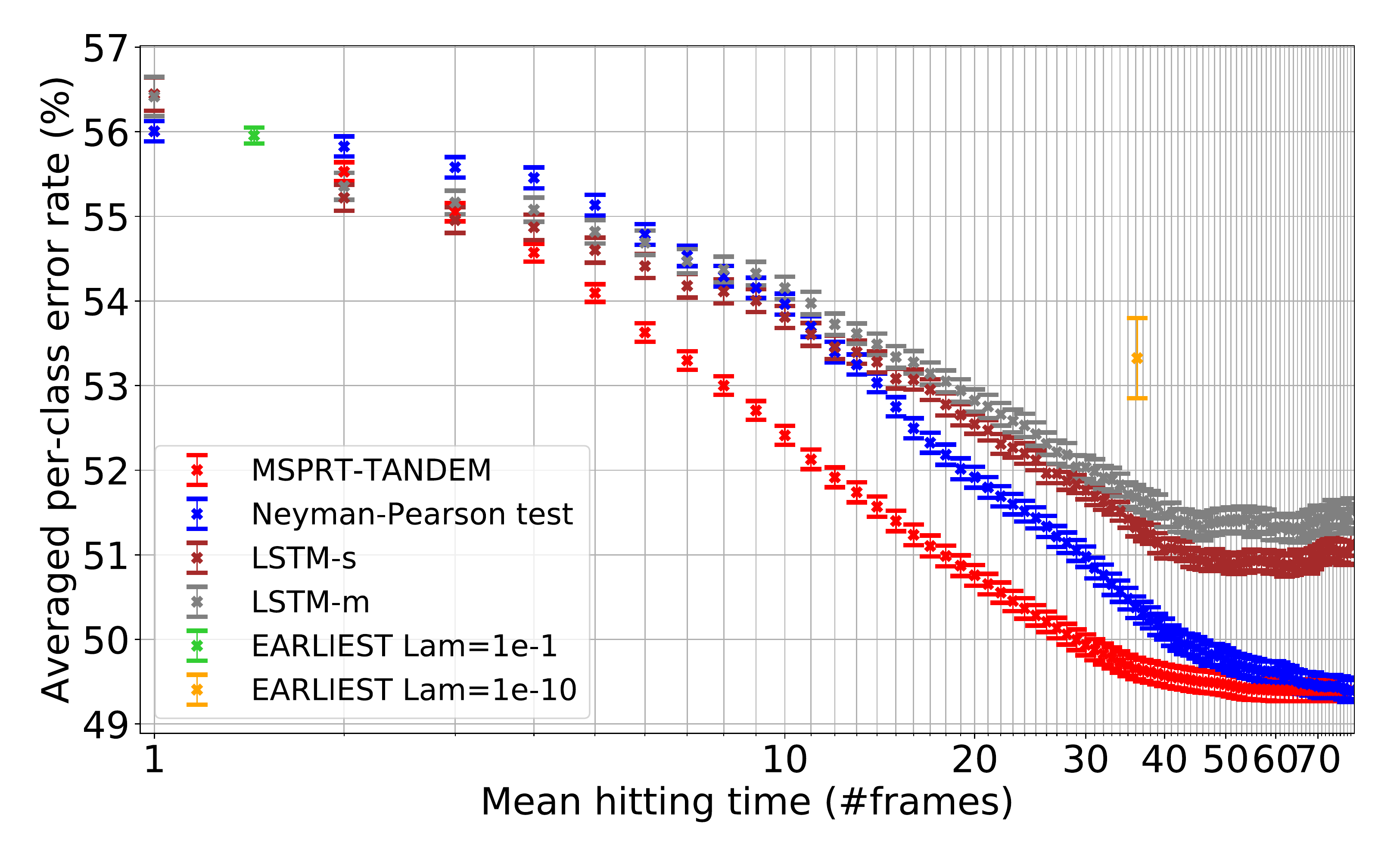} 
    \end{minipage}
    \caption{\textbf{Speed-Accuracy Tradeoff (SAT) Curves.} The vertical axis represents the averaged per-class error rate, and the horizontal axis represents the mean hitting time. Early and accurate models come in the lower-left area. The vertical error bars are the standard error of mean (SEM); however, some of the error bars are too small and are collapsed. 
    \textbf{Upper left: NMNIST-H.} The Neyman-Pearson test, LSTM-s, and LSTM-m almost completely overlap. \textbf{Upper right: NMNIST-100f.} LSTM-s and LSTM-m completely overlap. \textbf{Lower left: UCF101}. \textbf{Lower right: HMDB51}.} 
    \label{fig: SATCs}
\end{figure*}

To evaluate the performance of MSPRT-TANDEM, we use averaged per-class error rate and mean hitting time: Both measures are necessary because early classification of time series is a multi-objective optimization problem. The averaged per-class error rate, or balanced error, is defined as $1 - \f{1}{K} \sum_{k=1}^{K} \f{ | \{ i \in [M] | h_i = y_i = k \} | }{ | \{ i \in [M] | y_i = k \} | }$, where $h_i \in [K]$ is the prediction of the model for $i \in [M]$ in the dataset. 
The mean hitting time is defined as the arithmetic mean of the stopping times of all sequences. 

We use four datasets: two are new simulated datasets made from MNIST \cite{MNIST} (NMNIST-H and NMNIST-100f), and two real-world public datasets for multiclass action recognition (UCF101 \cite{UCF101} and HMDB51 \cite{HMDB51}).
A sequence in NMNIST-H consists of 20 frames of an MNIST image filled with dense random noise, which is gradually removed (10 pixels per frame), 
while a sequence in NMNIST-100f consists of 100 frames of an MNIST image filled with random noise that is so dense that humans cannot classify any video (Appendix \ref{app: NMNIST-H and NMNIST-100f}); only 15 of $28\times28$ pixels maintain the original image. The noise changes temporally and randomly and is not removed, unlike in NMNIST-H. 


\subsection{Models}
We compare the performance of MSPRT-TANDEM with four other models: LSTM-s, LSTM-m \cite{LSTM_ms}, EARLIEST \cite{EARLIEST}, and the Neyman-Pearson (NP) test \cite{neyman1933Neyman-Pearson_test_NPT}. LSTM-s and LSTM-m, proposed in a pioneering work in deep learning-based early detection of human action \cite{LSTM_ms}, use loss functions that enhance monotonicity of class probabilities (LSTM-s) and margins of class probabilities (LSTM-m). Note that LSTM-s/m support only the fixed-length test; i.e., the stopping time is fixed, unlike MSPRT-TANDEM. EARLIEST is a reinforcement learning algorithm based on recurrent neural networks (RNNs). The base RNN of EARLIEST calculates a current state vector from an incoming time series. The state vector is then used to generate a stopping probability in accordance with the binary action sampled: Halt or Continue. EARLIEST has two objective functions in the total loss: one for classification error and one for earliness. The balance between them cannot change after training. The NP test is known to be the most powerful, Bayes optimal, and minimax optimal test \cite{borovkov1998mathematical, lehmann2006TartarBook261}. The NP test uses the LLR to make a decision in a similar manner to the MSPRT, but the decision time is fixed. The decision rule is $d^{\textrm{NP}} (X^{(1,t)}) := \mathrm{argmax}_{k \in [K]} \min_{l \in [K]} \la_{kl} (X^{(1,t)})$ with a fixed $t \in [T]$.
In summary, LSTM-s/m have different loss functions from MSPRT-TANDEM, and the stopping time is fixed. EARLIEST is based on reinforcement learning, and its stopping rule is stochastic. The only difference between the NP test and MSPRT-TANDEM is whether the stopping time is fixed.

We first train the feature extractor (ResNet \cite{ResNetV1, ResNetV2}) by solving multiclass classification with the softmax loss and extract the bottleneck features, which are then used to train LSTM-s/m, EARLIEST, and the temporal integrator for MSPRT-TANDEM and NP test. Note that all models use the same feature vectors for the training. For a fair comparison, hyperparameter tuning is carried out with the default algorithm of Optuna \cite{Optuna} with an equal number of tuning trials for all models. 
Also, all models have the same order of trainable parameters.
After fixing the hyperparameters, we repeatedly train the models with different random seeds to consider statistical fluctuation due to random initialization and stochastic optimizers. Finally, we test the statistical significance of the models with the two-way ANOVA followed by the Tukey-Kramer multi-comparison test. 
More detailed settings are given in Appendix \ref{app: Details of Experiment and More Results}.

\subsection{Results}
The performances of all the models are summarized in Figure \ref{fig: SATCs} (The lower left area is preferable). 
We can see that MSPRT-TANDEM outperforms all the other models by a large margin, especially in the early stage of sequential observations. We confirm that the results have statistical significance; i.e., our results are reproducible (Appendix \ref{app: Statistical Tests}). The loss functions of LSTM-s/m force the prediction score to be monotonic, even when noisy data are temporally observed, leading to a suboptimal prediction. In addition, LSTM-s/m have to make a decision, even when the prediction score is too small to make a confident prediction. However, MSPRT-TANDEM can wait until a sufficient amount of evidence is accumulated. A potential weakness of EARLIEST is that reinforcement learning is generally unstable during training, as pointed out in  \cite{Nikishin2018, Kumar2020}. The NP test requires more observations to attain a comparable error rate to that of MSPRT-TANDEM, as expected from the theoretical perspective \cite{tartakovsky2014TartarBook}: In fact, the SPRT was originally developed to outperform the NP test in sequential testing \cite{wald1945TartarBook492, wald1947WaldBook_TartarBook494}.



%% file: Conclusion.tex
\section{Conclusion}
We propose the LSEL for DRME, which has yet to be explored in the literature.
The LSEL relaxes the crucial assumption of the MSPRT and enables its real-world applications.
We prove that the LSEL has a theoretically strong background: consistency, hard class weighting, and guess-aversion.
We also propose MSPRT-TANDEM, the first DRE-based model for early multiclass classification in deep learning.
The experiment shows that the LSEL and MSPRT-TANDEM outperform other baseline models statistically significantly.


%% file: Acknowledgements.tex
\section*{Acknowledgments}
The authors thank the anonymous reviewers for their careful reading to improve the manuscript. 
We would like to thank Jiro Abe, Genki Kusano, Yuta Hatakeyama, Kazuma Shimizu, and Natsuhiko Sato for helpful comments on the proof of the LSEL's consistency.

%% file: supp_MSPRT.tex
\section{Asymptotic Optimality of MSPRT} \label{app: MSPRT}
In this appendix, we review the asymptotic optimality of the matrix sequential probability ratio test (MSPRT). The whole statements here are primarily based on \citeApp{tartakovsky2014TartarBook} and references therein. We here provide theorems without proofs, which are given in \citeApp{tartakovsky2014TartarBook}. 

\subsection{Notation and Basic Assumptions}
First, we introduce mathematical notation and several basic assumptions. Let $X^{(0,T)} := \{ x^{(t)} \}_{0\leq t \leq T}$ be a stochastic process. We assume that $ K (\in \mbn)$ densities $p_k (X^{(0,T)})$ ($k \in [K] := \{ 1,2, ... ,K \}$) are distinct. Let $y (\in \mcy := [K])$ be a parameter of the densities. Our task is to test $K$ hypotheses $H_k: y = k$ ($k \in [K]$); 
i.e., to identify which of the $K$ densities $p_k$ is the true one through consecutive observations of $x{(t)}$.

Let $(\Om, \mc{F}, \{\mc{F}_t\}_{t \geq 0} , P)$ for $t \in \mbz_{\geq 0} := \{0,1,2,...\}$ or $t \in \mbr_{\geq0} := [0, \infty)$ be a filtered probability space. The sub-$\sigma$-algebra $\mc{F}_t$ of $\mc{F}$ is assumed to be generated by $X^{(0,t)}$. Our target hypotheses are $H_k : P=P_k$ ($k \in [K]$) where $P_k$ are probability measures that are assumed to be locally mutually absolutely continuous. Let $\mbe_k$ denote the expectation under $H_k$ ($k \in [K]$); e.g., $\mbe_k [f(X^{(0,t)}] = \int f(X^{(0,t)}) d P_k (X^{(0,t)})$ for a function $f$. We define the likelihood ratio matrix as
\begin{align}
    \La_{kl}(t) := \f{d P_k^t}{d P_l^t}(X^{(0,t)}) \hspace{10pt} (t \geq 0)\, ,
\end{align}
where $\La_{kl} (0) = 1$ $P_k$-a.s. and $P_k^t$ is the restriction of $P_k$ to $\mc{F}_t$. Therefore, the LLR matrix is defined as 
\begin{equation}
    \la_{kl} (t) := \log \La_{kl}(t) \hspace{10pt} (t \geq 0) \, ,
\end{equation}
where $\la_{kl} (0) = 0$ $P_k$-a.s. The LLR matrix plays a crucial role in the MSPRT, as seen in the main text and in the following.

We define a multihypothesis sequential test as $\de := (d, \ta)$. $d := d (X^{(0,t)})$ is an $\mc{F}_t$-measurable terminal decision function that takes values in $[K]$. $\ta$ is a stopping time with respect to $\{ \mc{F}_t \}_{t \geq 0}$ and takes values in $[0, \infty)$. Therefore, $\{\om \in \Om | d = k \} = \{\om \in \Om |  \ta < \infty , \de \,\,\text{accepts}\,\, H_k \}$. In the following, we solely consider statistical tests with $\mbe_k [\ta] < \infty$ ($k \in [K]$).

\paragraph{Error probabilities.}
We define three types of error probabilities:
\begin{align}
    &\al_{kl} (\de) := P_k(d=l) \hspace{10pt} (k \neq l, \,\,\,\,\, k,l \in [K]) \, , \nn
    &\al_{k} (\de) := P_k (d \neq k) = \sum_{l (\neq k)} \al_{kl} (\de) \hspace{10pt} (k \in [K]) \, ,\nn
    &\be_l (\de) := \sum_{k \in [K]} w_{kl} P_k (d = l) \hspace{10pt} \hspace{10pt} (l \in [K]) \, ,
\end{align}
where $w_{kl} > 0$ except for the zero diagonal entries. We further define $\al_{\rm max} := \max_{k,l} \al_{kl}$. Whenever $\al_{\rm max} \rightarrow 0$, we hereafter assume that for all $k,l,m,n \in [K]$ ($k \neq l$, $m \neq n$),
\begin{equation*}
    \lim_{\al_{\rm max} \rightarrow 0} \f{\log\al_{kl}}{\log\al_{mn}} = c_{klmn} \, ,
\end{equation*}
where $0< c_{klmn} < \infty$. This technical assumption means that $\al_{kl}$ does not go to zero at an exponentially faster or slower rate than the others.

\paragraph{Classes of tests.}
We define the corresponding sets of statistical tests with bounded error probabilities.
\begin{align}    
    &C(\{ \al \}) := \{ \, \de \, | \,\, \al_{kl} (\de) \leq \al_{kl}, \,\,\,\, k \neq l, \,\,\,\, k,l \in [K] \} \, ,\nn
    &C(\bm{\al}) := \{ \, \de \, | \,\, \al_{k} (\de) \leq \al_{k}, \,\,\,\, k \in [K] \} \, , \nn
    &C(\bm{\be}) := \{ \, \de \, | \,\, \be_{l} (\de) \leq \be_{l}, \,\,\,\, l \in [K] \} \, .
\end{align}

\paragraph{Convergence of random variables.}
We introduce the following two types of convergence for later convenience.
\begin{definition}[Almost sure convergence (convergence with probability one)]
    Let $\{ x^{(t)} \}_{t \geq 0}$ denote a stochastic process. We say that stochastic process $\{x^{(t)}\}_{t\geq 0}$ converges to a constant $c$ almost surely as $t\rightarrow \infty$ (symbolically, $x^{(t)} \xrightarrow[t\rightarrow\infty]{P-a.s.} c$), if
    \begin{equation*}
        P \lt( \underset{t\rightarrow\infty}{\lim} x^{(t)} = c \rt) = 1 \, .
    \end{equation*}
\end{definition}

\begin{definition}[$r$-quick convergence]
    Let $\{ x^{(t)} \}_{t\geq 0}$ be a stochastic process. Let $\mathcal{T}_\ep(\{ x^{(t)} \}_{t\geq 0})$ be the last entry time of $\{ x^{(t)} \}_{t\geq 0}$ into the region $(-\infty, -\ep) \cup (\ep, \infty)$; i.e.,
    \begin{equation}
        \mathcal{T}_\ep(\{ x^{(t)} \}_{t\geq 0}) = \underset{t\geq 0}{\rm sup} \big\{  t \,|\,\,\, |x^{(t)}|>\ep \, \big\}, \quad \mr{sup}\{\emptyset\} := 0 \,.
    \end{equation}
    Then, we say that stochastic process $\{ x^{(t)} \}_{t\geq 0}$ converges to zero r-quickly, or 
    \begin{equation}
        x^{(t)} \underset{t\rightarrow \infty}{\xrightarrow{r-{\rm quickly}}} 0 \, ,
    \end{equation}
    for some $r>0$, if
    \begin{equation}
        \mbe[(\mathcal{T}_\ep(\{ x^{(t)} \}_{t\geq 0}))^r] < \infty \quad \textrm{for all $\ep>0$}\,.
    \end{equation}
\end{definition}
$r$-quick convergence ensures that the last entry time into the large-deviation region ($\mathcal{T}_\ep(\{ x^{(t)} \}_{t \geq  0})$) is finite in expectation.

\subsection{MSPRT: Matrix Sequential Probability Ratio Test}
Formally, the MSPRT is defined as follows:
\begin{definition}[Matrix sequential probability ratio test]
    Define a threshold matrix $a_{kl} \in \mbr$ ($k,l \in [K]$), where the diagonal elements are immaterial and arbitrary, e.g., 0. The MSPRT $\de^*$ of multihypothesis $H_k: P = P_k$ ($k \in [K]$) is defined as
    \begin{align}
        &\de^* := (d^*, \ta^*) \nn
        &\ta^* := \min \{ \ta_k | k \in [K] \} \nn
        &d^* := k \hspace{10pt} \text{if} \hspace{10pt} \ta^* = \ta_k \nn
        &\ta_k := \inf \{ 
            t \geq 0 
            | \underset{\substack{l \in [K] \\ l (\neq k)} }{\min} \{
                \la_{kl}(t) - a_{lk}
            \} \geq 0
        \} \hspace{5pt} (k \in [K]) \no \, .    
    \end{align}
\end{definition}
In other words, the MSPRT stops at the smallest $t$ such that for a number of $k \in [K]$, $\la_{kl}(t) \geq a_{lk}$ for all $l (\neq k)$. Note that the uniqueness of such $k$ is ensured by the anti-symmetry of $\la_{kl}$. 
In our experiment, we use a single-valued threshold for simplicity. A general threshold matrix may improve performance, especially when the dataset is class-imbalanced \citeApp{longadge2013class, ali2015classification, hong2016dealing}. We occasionally use $A_{lk} := e^{a_{lk}}$ in the following.

The following lemma determines the relationship between the thresholds and error probabilities.
\begin{lemma}[General error bounds of the MSPRT \citeApp{tartakovsky1998TartarBook455}] \label{lem:General error bounds of MSPRT}
    The following inequalities hold:
    \begin{enumerate}
        \item $\al_{kl}^* \leq e^{-a_{kl}}$ \hspace{10pt} for $\,\,k,l\in[K]$\,\, ($k \neq l$),
        \item $\al_k^* \leq \sum_{l(\neq k)} e^{-a_{kl}}$ \hspace{10pt} for $\,\,k\in[K]$,
        \item $\be_l^* \leq \sum_{k(\neq l)} w_{kl} e^{-a_{kl}}$ \hspace{10pt} for $\,\,l\in[K]$.
    \end{enumerate}
\end{lemma}
Therefore,
\begin{align}
    & a_{lk} \geq \log(\f{1}{\al_{lk}}) \Longrightarrow \de^* \in C(\{\al\}) \, , \nn
    & a_{lk} \geq a_l = \log( \f{K - 1}{\al_{lk}} ) \Longrightarrow \de^* \in C(\bm{\al}) \, ,\nn
    & a_{lk} \geq a_k = \log( \f{ \sum_{m(\neq k)} w_{mk} }{\be_k} ) \Longrightarrow \de^* \in C(\bm{\be}) \, . \no
\end{align}

\subsection{Asymptotic Optimality of MSPRT in I.I.D. Cases}
\subsubsection{Under First Moment Condition}
Given the true distribution, one can derive a dynamic programming recursion; its solution is the optimal stopping time. However, that recursion formula is intractable in general due to its composite integrals to calculate expectation values\citeApp{tartakovsky2014TartarBook}. Thus, we cannot avoid several approximations unless the true distribution is extremely simple.

To avoid the complications above, we focus on the asymptotic behavior of the MSPRT, where the error probabilities go to zero. In this region, the stopping time and thresholds typically approach infinity because more evidence is needed to make such careful, perfect decisions.

First, we provide the lower bound of the stopping time. Let us define the first moment of the LLR: $I_{kl} := \mbe_k [\la_{kl}(1)]$. Note that $I_{kl}$ is the Kullback-Leibler divergence and hence $I_{kl} \geq 0$.
\begin{lemma}[Lower bound of the stopping time \citeApp{tartakovsky2014TartarBook}]
    Assume that $I_{lk}$ is positive and finite for all $k, l \in [K]$ ($k \neq l$). If $\sum_{k \in [K]}\al_{k} \leq 1$, then for all $k \in [K]$,
    \begin{equation}
        \underset{\de \in C(\{ \al \})}{\mr{inf}} \mbe_k [\ta] 
        \geq \max \lt[ \f{1}{I_{kl}} \sum_{m \in [K]} \al_{km} \log (\f{\al_{km}}{\al_{lm}}) \rt] \, . \no
    \end{equation}
\end{lemma}
The proof follows from Jensen's inequality and Wald's identity $\mbe_k [\la_{kl}(\ta)] = I_{kl} \mbe_k [\ta]$. However, the lower bound is unattainable in general\footnote{
    We can show that when $K=2$, the SPRT can attain the lower bound if there are no overshoots of the LLR over the thresholds.
}. In the following, we show that the MSPRT asymptotically satisfies the lower bound.

\begin{lemma}[Asymptotic lower bounds (i.i.d. case) \citeApp{tartakovsky1998TartarBook455}] \label{lem:Asymptotic lower bounds (i.i.d.)}
    Assume the first moment condition $0 < I_{kl} < \infty$ for all $k, l \in [K]$ ($k \neq l$). The following inequalities hold for all $m > 0$ and $k \in [K]$,
    \begin{enumerate}
        \item As $\al_{\rm max} \rightarrow 0$,
            \begin{equation}
                \underset{\de \in C(\{\al\})}{\mr{inf}} \mbe_k [\ta]^m \geq \underset{l (\neq k)}{\max} \lt[ 
                    \f{|\log \al_{lk}|}{I_{kl}}
                \rt]^m (1 + o(1)) \, . \no
            \end{equation}
        \item As $\max_{k} \al_k \rightarrow 0$,
            \begin{equation}
                \underset{\de \in C(\bm{\al})}{\mr{inf}} \mbe_k [\ta]^m \geq \underset{l (\neq k)}{\max} \lt[ 
                    \f{|\log \al_{l}|}{I_{kl}}
                \rt]^m (1 + o(1)) \, . \no
            \end{equation}
        \item As $\max_{k} \be_k \rightarrow 0$,
            \begin{equation}
                \underset{\de \in C(\bm{\be})}{\mr{inf}} \mbe_k [\ta]^m \geq \underset{l (\neq k)}{\max} \lt[ \f{|\log \be_k|}{I_{kl}} \rt]^m (1 + o(1)) \, . \no
            \end{equation}
    \end{enumerate}
\end{lemma}

\begin{theorem}[Asymptotic optimality of the MSPRT with the first moment condition (i.i.d. case) \citeApp{tartakovsky1998TartarBook455, dragalin1999TartarBook128_MSPRT-I, tartakovsky2014TartarBook}] \label{thm: Asymptotic optimality with first moment condition}
    Assume the first moment condition $0 < I_{kl} < \infty$ for all $k, l \in [K]$ ($k \neq l$).
    \begin{enumerate}
        \item 
            If $a_{lk} = \log(\f{1}{\al_{lk}})$ for $k.l \in [K]$ ($k \neq l$), then $\de^* \in C (\{\al\})$, and for all $m > 0$ and $k \in [K]$,
            \begin{equation}
                \underset{\de \in C(\{ \al \})}{\mr{inf}} \mbe_k [\ta]^m 
                \sim \mbe_k [\ta^*]^m
                \sim \underset{\substack{l \in [K] \\ l (\neq k)}}{\max} \lt[
                    \f{|\log \al_{lk}|}{I_{kl}} 
                \rt]^m  
            \end{equation}
            as $\al_{\rm max} \rightarrow 0$. If $a_{lk} \neq \log (\f{1}{\al_{lk}})$, the above inequality holds when $a_{lk} \sim \log (\f{1}{\al_{lk}})$ and $\al_{kl} (\de^*) \leq \al_{kl}$.\footnote{
                Recall that $\al_{kl} (\de^*) \leq \al_{kl}$ is automatically satisfied whenever $a_{lk} \geq \log(1/\al_{lk})$ for general distribution because of Lemma \ref{lem:General error bounds of MSPRT}. Similar arguments follow for 2. and 3. in Theorem \ref{thm: Asymptotic optimality with first moment condition}.
            }
        \item 
            If $a_{lk} = \log(\f{K-1}{\al_{l}})$ for $k.l \in [K]$ ($k \neq l$), then $\de^* \in C (\bm{\al})$, and for all $m > 0$ and $k \in [K]$,
            \begin{equation}
                \underset{\de \in C(\bm{\al})}{\mr{inf}} \mbe_k [\ta]^m 
                \sim \mbe_k [\ta^*]^m
                \sim \underset{\substack{l \in [K] \\ l (\neq k)}}{\max} \lt[
                    \f{|\log \al_{l}|}{I_{kl}} 
                \rt]^m 
            \end{equation}
            as $\max_k \al_{k} \rightarrow 0$. If $a_{l} \neq \log(\f{K-1}{\al_{l}})$, the above inequality holds when $a_{lk} \sim \log(\f{K-1}{\al_{l}})$ and $\al_{l} (\de^*) \leq \al_{l}$.
        \item
            If $a_{lk} = \log(\f{\sum_{n (\neq k)} w_{nk} }{\be_k})$ for $k.l \in [K]$ ($k \neq l$), then $\de^* \in C (\bm{\be})$, and for all $m > 0$ and $k \in [K]$,
            \begin{equation}
                \underset{\de \in C(\bm{\be})}{\mr{inf}} \mbe_k [\ta]^m 
                \sim \mbe_k [\ta^*]^m
                \sim \underset{\substack{l \in [K] \\ l (\neq k)}}{\max} \lt[
                    \f{|\log \be_{k}|}{I_{kl}} 
                \rt]^m 
            \end{equation}
            as $\max_k \be_{k} \rightarrow 0$. If $a_{lk} \neq \log(\f{1}{\be_{k}})$, the above inequality holds when $a_{lk} \sim \log(\f{1}{\be_{k}})$ and $\be_{k} (\de^*) \leq \be_{k}$.
    \end{enumerate}
\end{theorem}
Therefore, we conclude that the MSPRT asymptotically minimizes all positive moments of the stopping time including $m=1$; i.e., asymptotically, the MSPRT makes the quickest decision in expectation among all the algorithms with bounded error probabilities.

\subsubsection{Under Second Moment Condition}
The \textit{second moment condition} 
\begin{equation}
    \mbe_k [ \la_{kl} (1) ]^2 < \infty \hspace{10pt} (k,l \in [K])
\end{equation}
strengthens the optimality. We define the \textit{cumulative flaw matrix} as
\begin{equation}
    \Up_{kl} := \exp (
        - \sum_{t=1}^{\infty} \f{1}{t} [ 
            P_l ( \la_{kl} (t) > 0 ) 
            + P_k ( \la_{kl} (t) \leq 0 )
        ]
    ) 
\end{equation}
($k,l \in [K]$, $k \neq l$). Note that $0 < \Up_{kl} = \Up_{lk} \leq 1$.

\begin{theorem}[Asymptotic optimality of the MSPRT with the second moment condition (i.i.d. case) \citeApp{lorden1977TartarBook275_GaryMSPRT, tartakovsky2014TartarBook}] \label{thm: Asymptotic optimality with second moment condition}
    Assume that the threshold $A_{kl} = A_{kl}(c)$ is a function of a small parameter $c > 0$. Then, the error probabilities of the MSPRT are also functions of $c$; i.e., $\al_{kl}(\de^*) =: \al_{kl}^*(c)$, $\al_{k}(\de^*) =: \al_{k}^*(c)$, and $\be_{k}(\de^*) =: \be_{k}^*(c)$, and $A_{lk}(c) \xrightarrow{c \rightarrow 0} \infty$ indicates $\al_{kl}^*(c), \al_{k}^*(c), \be_{k}^* (c) \xrightarrow{c \rightarrow 0} 0$.
    \begin{enumerate}
        \item
            Let $A_{lk} = B_{lk} / c$ for any $B_{kl} > 0$ ($k \neq l$). Then, as $c \rightarrow 0$, 
            \begin{align}
                &\mbe_k \ta^*(c) = \underset{\de \in C(\{ \al^*(c) \})}{\mr{inf}} \mbe_k \ta + o(1) \\
                &\mbe_k \ta^*(c) = \,\,\underset{\de \in C(\bm{\al}^*(c))}{\mr{inf}} \,\, \mbe_k \ta + o(1) 
            \end{align}
            for all $k \in [K]$, where $\bm{\al}^*(c) := (\al_{1}^*(c), ..., \al_{K}^*(c))$.
        \item
            Let $A_{lk} (c) = w_{lk} \Up_{kl} / c$. Then, as $c \rightarrow 0$,
            \begin{equation}
                \mbe_k \ta^* (c) = \underset{\de \in C(\{ \bm{\be}^*(c) \})}{\mr{inf}} \mbe_k \ta + o(1)
            \end{equation}
            for all $k \in [K]$, where $\bm{\be}^*(c) := (\be_{1}^*(c), ..., \be_{K}^*(c))$.
    \end{enumerate}
\end{theorem}
Therefore, the MSPRT $\de^*$ asymptotically minimizes the expected stopping time among all tests whose error probabilities are less than or equal to those of $\de^*$. We can further generalize Theorem \ref{thm: Asymptotic optimality with second moment condition} by introducing different costs $c_k$ for each hypothesis $H_k$ to allow different rates (see \citeApp{tartakovsky2014TartarBook}).

\subsection{Asymptotic Optimality of MSPRT in General Non-I.I.D. Cases}
\begin{lemma}[Asymptotic lower bounds \citeApp{tartakovsky1998TartarBook455}]\label{lem:Asymptotic lower bounds (non-i.i.d.)}
    Assume that there exists a non-negative increasing function $\ps (t)$ ($\ps(t) \xrightarrow{t \rightarrow \infty} \infty$) and positive finite constants $I_{lk}$ ($k,l \in [K]$, $k \neq l$) such that for all $\ep > 0$ and $k,l \in [K]$ ($k \neq l$),
    \begin{equation}
        \underset{T\rightarrow\infty}{\lim} P_k \lt(
            \underset{0 \leq t \leq T}{\sup} \f{\la_{kl}(t)}{\ps(T)} \geq (1+\ep) I_{kl}
        \rt) = 1 \, . \label{eq:increasing assumption}
    \end{equation}
    Then, for all $m > 0$ and $k \in [K]$,
    \begin{align}
        &\underset{\de \in C(\{\al\})}{\inf} \mbe_k [\ta]^m \geq \Ps \lt(
            \underset{\substack{l \in [K] \\ l (\neq k)}}{\max} \f{|\log \al_{lk}|}{I_{kl}}
        \rt)^m (1 + o(1))  
        \hspace{10pt} \text{as} \hspace{5pt} \al_{\rm max} \rightarrow 0 \, , \\
        &\underset{\de \in C(\bm{\al})}{\inf} \mbe_k [\ta]^m \geq \Ps \lt(
            \underset{\substack{l \in [K] \\ l (\neq k)}}{\max} \f{|\log \al_{l}|}{I_{kl}}
        \rt)^m (1 + o(1)) 
        \hspace{10pt} \text{as} \hspace{5pt} \max_n \al_n \rightarrow 0 \, , \\
        &\underset{\de \in C(\bm{\be})}{\inf} \mbe_k [\ta]^m \geq \Ps \lt(
            \underset{\substack{l \in [K] \\ l (\neq k)}}{\max} \f{|\log \be_{k}|}{I_{kl}}
        \rt)^m (1 + o(1)) 
        \hspace{10pt} \text{as} \hspace{5pt} \max_n \be_n \rightarrow 0 \, .
    \end{align}
    where $\Ps$ is the inverse function of $\ps$.
\end{lemma}
Note that if for all $0< T < \infty$
\begin{align}
    P_k \lt( \underset{0\leq t \leq T}{\sup} |\la_{kl}(t)| < \infty \rt) = 1
\end{align}
and if 
\begin{equation}
    \f{\la_{kl}(t)}{\ps(t)} \xrightarrow[t\rightarrow\infty]{P_k-a.s.} I_{kl} 
    \hspace{10pt} (k,l \in [K], \,\,\, k \neq l) \, ,
\end{equation}
then (\ref{eq:increasing assumption}) holds.

\begin{theorem}[Asymptotic optimality of the MSPRT \citeApp{tartakovsky1998TartarBook455}]\label{thm: Asymptotic optimality of MSPRT (non-i.i.d.)}
    Assume that that there exists a non-negative increasing function $\ps(t)$ ($\ps(t)\xrightarrow{t\rightarrow\infty} \infty$) and positive finite constants $I_{kl}$ ($k,l \in [K]$, $k \neq l$) such that for some $r > 0$,
    \begin{equation}
        \f{\la_{kl}(t)}{\ps(t)} \xrightarrow[t\rightarrow\infty]{P_k-r-quickly} I_{kl}
    \end{equation}
    for all $k,l \in [K]$ ($k \neq l$). Let $\Ps$ be the inverse function of $\ps$. Then, 
    \begin{enumerate}
        \item 
            If $a_{lk} \sim \log(1 / \al_{lk})$ and $\al_{kl} (\de^*) \leq \al_{kl}$ ($k,l \in [K]$, $k \neq l$)\footnote{
                Recall that $\al_{kl} (\de^*) \leq \al_{kl}$ is automatically satisfied whenever $a_{lk} \geq \log(1/\al_{lk})$ for general distribution because of Lemma \ref{lem:General error bounds of MSPRT}. Similar arguments follow for 2. and 3. in Theorem \ref{thm: Asymptotic optimality of MSPRT (non-i.i.d.)}.
            }, then for all $m \in (0, r]$ and $k \in [K]$,
            \begin{align}
                &\underset{\de \in C (\{ \al \})}{ \inf } \mbe_k [\ta]^m
                \sim \mbe_k [\ta^*]^m 
                \sim \Ps \lt(
                    \underset{\substack{l \in [K] \\ l (\neq k)}}{ \max } \f{| \log \al_{lk} |}{I_{kl}}
                \rt)^m 
                \hspace{10pt} \text{as} \,\,\,\, \al_{\rm max} \rightarrow 0 \, .
            \end{align}
        \item
            If $a_{lk} \sim \log( (K-1) / \al_{l})$ and $\al_{k} (\de^*) \leq \al_{k}$ ($k,l \in [K]$, $k \neq l$), then for all $m \in (0, r]$ and $k \in [K]$,
            \begin{align}
                &\underset{\de \in C (\bm{\al})}{ \inf } \mbe_k [\ta]^m
                \sim \mbe_k [\ta^*]^m 
                \sim \Ps \lt(
                    \underset{\substack{l \in [K] \\ l (\neq k)}}{ \max } \f{| \log \al_{l} |}{I_{kl}}
                \rt)^m 
                \hspace{10pt} \text{as} \,\,\,\, \max_k \al_k \rightarrow 0 \, .
            \end{align}
        \item
            If $a_{lk} \sim \log( \sum_{n(\neq k)} w_{n k} / \be_{k})$ and $\be_{k} (\de^*) \leq \be_{k}$ ($k,l \in [K]$, $k \neq l$), then for all $m \in (0, r]$ and $k \in [K]$,
            \begin{align}
                &\underset{\de \in C (\bm{\be})}{ \inf } \mbe_k [\ta]^m
                \sim \mbe_k [\ta^*]^m 
                \sim \Ps \lt(
                    \underset{\substack{l \in [K] \\ l (\neq k)}}{ \max } \f{| \log \be_{k} |}{I_{kl}}
                \rt)^m  
                &\hspace{10pt} \text{as} \,\,\,\, \max_k \be_k \rightarrow 0 \, .
            \end{align}
    \end{enumerate}
\end{theorem}
Therefore, combining Lemma \ref{lem:Asymptotic lower bounds (non-i.i.d.)} and Theorem \ref{thm: Asymptotic optimality of MSPRT (non-i.i.d.)}, we conclude that the MSPRT asymptotically minimizes the moments of the stopping time; i.e., asymptotically, the MSPRT makes the quickest decision in expectation among all the algorithms with bounded error probabilities, even without the i.i.d. assumption.


%% file: supp_Related.tex
\section{Supplementary Related Work} \label{app: Supplementary Related Work}
We provide additional references.
Our work is an interdisciplinary study and potentially bridges various research areas, such as early classification of time series, sequential hypothesis testing, sequential decision making, classification with abstention, and DRE.

\paragraph{Early classification of time series.}
Many methods have been proposed for early classification of time series: non-deep models are \citeApp{xing2011extracting_localized_class_signature, mcgovern2011identifying, Xing2012ECTS, ghalwash2012early_localized_class_signature, ghalwash2013extraction_localized_class_signature, ghalwash2014utilizing_localized_class_signature, Mori2016reliable_ECDIRE_non-deep_famous, karim2019framework_localized_class_signature, schafer2020TEASER}; deep models are \citeApp{LSTM_ms, wang2016earliness_EAConvNets, suzuki2018anticipating_AdaLEA, russwurm2019end_E2E-ECTS}; reinforcement learning-based models are \citeApp{EARLIEST, martinez2020EARLIEST-like_RL-based_ECTS, Wang2020CVPR_RL-based}.
There is a wide variety of real-world applications of such models:
length adaptive text classification \citeApp{huang2017length_adaptive_text_classification}, 
early text classification for sexual predator detection and depression detection on social media documents \citeApp{lopez2018early}, 
early detection of thermoacoustic instability from high-speed videos taken from a combustor \citeApp{gangopadhyay20213d_combustion_system_thermoacoustic_instability}, 
and early river classification through real-time monitoring of water quality \citeApp{gupta2019game_theory_river_classification}.

\paragraph{Early exit problem.}
The \textit{overthinking problem} \citeApp{kaya2019Shallow-Deep_Networks_overthinking_original_early_exit} occurs when a DNN can reach correct predictions before its final layer. Early exit from forward propagation mitigates wasteful computation and circumvents overfitting. \citeApp{kaya2019Shallow-Deep_Networks_overthinking_original_early_exit} proposes the Shallow-Deep Networks, which is equipped with internal layerwise classifiers and observes internal layerwise predictions to trigger an early exit. The early exit mechanism has been applied to Transformer \citeApp{vaswani2017NIPS_Transformer_origianal} and BERT \citeApp{devlin2019ACL_BERT_original}; e.g., see \citeApp{dehghani2019ICLR_Universal_Transformer_early_exit_plus_Transformer, zhou2020NeurIPS_BERT_early_exit}. 
Owing to early exiting, \citeApp{amir2021FrameExit_CVPR_oral_early_exit} sets a new state of the art for efficient video understanding on the HVU benchmark.
However, early exit algorithms are typically given by heuristics. MSPRT-TANDEM can be both the internal classifier and early exit algorithm itself with the theoretically sound background.

\paragraph{Classification with a reject option.}
Classification with a reject option is also referred to as classification with an abstain option, classification with abstention, classification with rejection, and selective classification.
Sequential classification with a reject option (to postpone the classification) can be regarded as early classification of time series \citeApp{ECTS_CWRO}.





\paragraph{SPRT.}
The SPRT for two-hypothesis testing (``binary SPRT'') is optimal for i.i.d. distributions \cite{WaldWolfowitz1948TartarBook496, tartakovsky2014TartarBook}. There are many different proofs: e.g., \cite{burkholder1963TartarBook90, matches1963TartarBook289, tartakovsky1991TartarBook452, lehmann2006TartarBook261, shiryaev2007TartarBook420, ferguson2014TartarBook143}. 
The Bayes optimality of the binary SPRT for i.i.d. distributions is proved in \citeApp{arrow1949TartarBook15, ferguson2014TartarBook143}.
The generalization of the i.i.d. MSPRT to non-stationary processes with independent increments is made in \citeApp{tartakovskij1981TartarBook449, golubev1984TartarBook168, tartakovsky1991TartarBook452, verdenskaya1992TartarBook484, tartakovsky1998TartarBook457}.

\paragraph{Density ratio estimation.}
A common method of estimating the density ratio is to train a machine learning model to classify two types of examples in a training dataset and extract the density ratio from the optimal classifier \cite{sugiyama2012Density_Ratio_Estimation_in_Machine_Learning, gutmann2012bregman, menon2016linking}.

%% file: supp_ProofConsistency.tex
\section{Proof of Theorem \ref{thm:consistency of LSEL}} \label{app:Proof of Consistency Theorem of LSEL}
In this appendix, we provide the proof of Theorem \ref{thm:consistency of LSEL}.

We define the target parameter set as $ \Th^*  := \{ \bm{\th}^* \in \mbr^{d_\th} | \hat{\la}(X^{(1,t)}; \bm{\th}^*) = \la(X^{(1,t)}) \,\, (\forall t \in [T]) \}$, and we assume $\Th^* \neq \emptyset$ throughout this paper. For instance, sufficiently large neural networks can satisfy this assumption. We additionally assume that each $\bm{\th}^*$ is separated in $\Th^*$; i.e., $\exists \de > 0$ such that $B ( \bm{\th}^* ; \de ) \cap B ( \bm{\th}^{*\prime}; \de ) = \emptyset$ for arbitrary $\bm{\th}^*, \bm{\th}^{*\prime} \in \Th^*$, where $B(\bm{\th}; \de)$ denotes the open ball at center $\bm{\th}$ with radius $\de$.\footnote{
    This assumption is for simplicity of the proofs. When the assumption above is not true, we conjecture that the consistency holds by assuming the positivity of a projected Hessian of $\hat{\la}_{kl}$ w.r.t. $\bm{\th}$ at $\partial \Th^*$ (the boundary of $\Th^*$). The projection directions may depend on the local curvature of $\partial \Th^{*}$ and whether $\Th^{*}$ is open or closed. We omit further discussions here because they are too complicated but maintain the basis of our statements.
} 
\begin{theorem}[Consistency of the LSEL]
    Let $L(\bm{\th})$ and $\hat{L}_S (\bm{\th})$ denote $L_{\rm \text{LSEL}} [\hat{\la}]$ and $\hat{L}_{\rm \text{LSEL}} (\bm{\th} ;S)$, respectively. Assume the following three conditions:
    \begin{itemize}
        \item[(a)] $\forall k, l \in [K]$, $\forall t \in [T]$, $p(X^{(1,t)} | k) = 0 \Longleftrightarrow p(X^{(1,t)} | l) = 0$.
        
        \item[(b)] $\mr{sup}_{\bm{\th}} | \hat{L}_S (\bm{\th}) - L (\bm{\th}) | \xrightarrow[M\rightarrow\infty]{P} 0$; i.e., $\hat{L}_S (\bm{\th})$ converges in probability uniformly over $\bm{\th}$ to $L(\bm{\th})$.\footnote{
            More specifically,
            $\forall \ep > 0, P( 
                \mr{sup}_{\bm{\th}} | \hat{L}_S (\bm{\th}) - L (\bm{\th}) | > \ep
            ) \xrightarrow{M \rightarrow \infty} 0$.
        }

        \item[(c)] For all $\th^* \in \Th^*$, there exist $t \in [T]$, $k\in[K]$ and $l\in[K]$, such that the following $d_\th \times d_\th$ matrix is full-rank:
        \begin{align}
            \int d X^{(1,t)} p( X^{(1,t)} | k )
            \nabla_{\bm{\th}^*} \hat{\la}_{k l}(X^{(1,t)}; \bm{\th}^*)
            \nabla_{\bm{\th}^*} \hat{\la}_{k l}(X^{(1,t)}; \bm{\th}^*)^{\top} \, .
        \end{align}
    \end{itemize}
    Then,  $P( \hat{\bm{\th}}_S \notin \Th^* ) \xrightarrow[]{M\rightarrow\infty}0$; i.e., $\hat{\bm{\th}}_S$ converges in probability into $\Th^*$.
\end{theorem}

First, we prove Lemma \ref{lem:Non-parametric estimation: LSEL}, which is then used in Lemma \ref{lem:Th* minimizes LSEL}. Using Lemma \ref{lem:Th* minimizes LSEL}, we prove Theorem \ref{thm:consistency of LSEL}. Our proofs are partly inspired by \citeApp{gutmann2012noise_NCE}. Note that for simplicity, we prove all the statements only for an arbitrary $t \in [T]$. The result can be straightforwardly generalized to the sum of the losses with respect to $t \in [T]$. Therefore, we omit $\f{1}{T}\sum_{t \in [T]}$ from $L$ and $\hat{L}_{S}$ in the following.

\begin{lemma}[Non-parametric estimation] \label{lem:Non-parametric estimation: LSEL}
    Assume that for all $k, l \in [K]$ , $p(X^{(1,t)} | k) = 0 \Longleftrightarrow p(X^{(1,t)} | l) = 0$. Then, $L[\ti{\la}]$ attains the unique minimum at $\ti{\la} = \la$.
\end{lemma}
\begin{proof}
Let $\ph ( X^{(1,t)} ) = ( \ph_{kl} (X^{(1,t)}) )_{k.l \in [K]}$ be an arbitrary perturbation function to $\tilde{\la}$. $\ph_{kl}$ satisfies $\ph_{kl} = - \ph_{lk}$ and is not identically zero if $k \neq l$. For an arbitrarily small $\ep > 0$,
\begin{align}
    &L[\tilde{\la} + \ep \ph] = L[\tilde{\la}]
    + \f{1}{K} \sum_{k \in [K]} \int d X^{(1,t)} p ( X^{(1,t)} | k )
    \bigg[ \ep \f{
            - \sum_{l (\neq k)}  e^{- \tilde{\la}_{kl} (X^{(1,t)})} \ph_{kl} (X^{(1,t)})
        }{
            \sum_{m \in [K]} e^{ - \tilde{\la}_{km} (X^{(1,t)}) }
        } \nn
    &+ \f{ \ep^2 }{ 2 ( \sum_{m \in [K]} e^{ - \tilde{\la}_{km} (X^{(1,t)}) } )^2 } \nn
    & \times \Big\{
            \sum_{m \in [K]} e^{ - \tilde{\la}_{km} (X^{(1,t)}) } 
            \sum_{l (\neq k)}  e^{ - \tilde{\la}_{kl} ( X^{(1,t)} )} \ph_{kl}^2 (X^{(1,t)})
    - ( \sum_{l (\neq k)}  e^{- \tilde{\la}_{kl} ( X^{(1,t)} ) } \ph_{kl} (X^{(1,t)}) )^2
        \Big\} \bigg] \nn
    &+ \mc{O}( \ep^3 ) \, . \label{eq:expansion}
\end{align}
A necessary condition for the optimality is that the first order terms vanish for arbitrary $\ph$. Because
\begin{align}
    & \text{(first order)} = - \f{\ep}{K} \int d X^{(1,t)} \sum_{ k > l } \ph_{kl} ( X^{(1,t)} ) \nn
    &\times \bigg[
        \f{ p(X|k) }{ \sum_{m \in [K]} e^{ - \tilde{\la}_{km} (X^{(1,t)}) } }
        e^{ -\tilde{\la}_{kl} (X^{(1,t)}) } 
    - \f{ p(X|l) }{ \sum_{m \in [K]} e^{ - \tilde{\la}_{lm} (X^{(1,t)}) } }
        e^{ -\tilde{\la}_{lk} (X^{(1,t)}) }
    \bigg] \, ,
\end{align}
and $p(X^{(1,t)} | k) = 0 \Leftrightarrow p(X^{(1,t)} | l) = 0$, the following equality holds at the unique extremum:
\begin{align}
    & \,\,\,\,\, \f{ p(X|k) }{ \sum_{m \in [K]} e^{ - \tilde{\la}_{km} (X^{(1,t)}) } }
         e^{ -\tilde{\la}_{kl} (X^{(1,t)}) }
    = \f{ p(X|l) }{ \sum_{m \in [K]} e^{ - \tilde{\la}_{lm} (X^{(1,t)}) } }
         e^{ -\tilde{\la}_{lk} (X^{(1,t)}) } \nn
    \Longleftrightarrow & \,\,\,\,\, 
         e^{\la_{kl}} \sum_{m \in [K]} \tilde{\La}_{ml} 
        = \tilde{\La}_{kl}^2 \sum_{m \in [K]} \tilde{\La}_{mk} \nn
    &\lt( = \tilde{\La}_{kl} \sum_{m\in[K]} \tilde{\La}_{mk} \tilde{\La}_{kl} = \tilde{\La}_{kl} \sum_{m\in[K]} \tilde{\La}_{ml} \rt) \nn
    \Longleftrightarrow & \,\,\,\,\, 
        e^{\la_{kl}} = \tilde{\La}_{kl} \nn   
    \Longleftrightarrow & \,\,\,\,\, 
        \tilde{\la}_{kl} ( X^{(1,t)} ) = \la_{kl} ( X^{(1,t)} ) \no \, ,
\end{align} where we defined $\tilde{\La}_{kl} := e^{\tilde{\la}_{kl}}$ and used $\tilde{\La}_{mk} \tilde{\La}_{kl} = \tilde{\La}_{ml}$. Next, we prove that $\tilde{\la}_{kl} = \la_{kl}$ is the minimum by showing that the second order of (\ref{eq:expansion}) is positive-definite:
\begin{align}
    & \text{(second order)}
    = \f{\ep^2}{2} \f{1}{K}\sum_{k \in [K]} \int d X^{(1,t)} 
        \f{p(X^{(1,t)} | k) }{ ( \sum_{m \in [K]} e^{ - \tilde{\la}_{km} (X^{(1,t)}) } )^2 } \nn
    & \times \bigg\{ 
        \sum_{m \in [K]} e^{ - \tilde{\la}_{km} (X^{(1,t)}) } 
        \sum_{l (\neq k) }\ph_{kl}^2 ( X^{(1,t)} ) e^{ - \tilde{\la}_{kl} ( X^{(1,t)} ) }
    - ( \sum_{m (\neq k)} \ph_{km} ( X^{(1,t)} ) e^{ - \tilde{\la}_{km} ( X^{(1,t)} ) } )^2        
    \bigg\} \nn
    = &  \f{\ep^2}{2} \f{1}{K} \sum_{k \in [K]} \int d X^{(1,t)} 
        \f{p(X^{(1,t)} | k) }{(\sum_{m \in [K]} e^{ - \tilde{\la}_{km} (X^{(1,t)}) })^2} \nn
    & \times \bigg\{ 
        \sum_{ l ( \neq k) } \ph_{kl}^2 (X^{(1,t)}) e^{- \tilde{\la}_{kl} (X^{(1,t)}) } \nn
    & + \sum_{\substack{m > n \\ m,n \neq k}} ( \ph_{km}(X^{(1,t)}) - \ph_{kn}(X^{(1,t)}) )^2 
    e^{- \tilde{\la}_{km} (X^{(1,t)}) }
        e^{- \tilde{\la}_{kn} (X^{(1,t)}) }
    \bigg\}  \nn 
    & > 0 \, .\no 
\end{align}
\end{proof}

\begin{lemma}[$\Th^*$ minimizes $L$]\label{lem:Th* minimizes LSEL}
    Assume that for all $\th^* \in \Th^*$, there exist $k^*\in[K]$ and $l^*\in[K]$, such that the following $d_\th \times d_\th$ matrix is full-rank:
        \begin{align}
            \int d X^{(1,t)} p( X^{(1,t)} | k^* )
            \nabla_{\bm{\th}^*} \hat{\la}_{k^* l^*}(X^{(1,t)}; \bm{\th}^*)
            \nabla_{\bm{\th}^*} \hat{\la}_{k^* l^*}(X^{(1,t)}; \bm{\th}^*)^{\top} \, .
        \end{align}
    Then, for any $\bm{\th} \notin \Th^*$,
    \begin{equation*}
        L(\bm{\th}) > L(\bm{\th}^*) \,\,\,\,\, (\forall \bm{\th}^* \in \Th^*) \, ,
    \end{equation*}
    meaning that $\Th^* = \mr{argmin}_{\bm{\th}} L(\bm{\th})$.
\end{lemma}
\begin{proof}
Let $\bm{\th}^*$ be an arbitrary element in $\Th^*$. For an arbitrarily small $\ep > 0$, let $\bm{\varphi} \in \mbr^{d_\th}$ be an arbitrary vector such that $\bm{\varphi} \neq \bm{0}$. Then, in a neighborhood of $\bm{\th}^*$,
\begin{align}
    & L[\hat{\la}(X^{(1,t)}; \bm{\th}^* + \ep \bm{\varphi})] 
    = L[\hat{\la}(X^{(1,t)}; \bm{\th}^*)] \nn
    & + \f{1}{K} \sum_{k \in [K]} \int d X^{(1,t)} p ( X^{(1,t)} | k )
    \bigg[ \ep \f{
            - \sum_{l (\neq k)} e^{- \hat{\la}_{kl} (X^{(1,t)}; \bm{\th}^*)} \rh_{kl} (X^{(1,t)})
        }{
            \sum_{m \in [K]} e^{ - \hat{\la}_{km} (X^{(1,t)}; \bm{\th}^*) }
        } \nn
    &+ \f{ \ep^2 }{ 2 ( \sum_{m \in [K]} e^{ - \hat{\la}_{km} (X^{(1,t)}; \bm{\th}^*) } )^2 }
    \Big\{
            -\sum_{m \in [K]} e^{ - \hat{\la}_{km} (X^{(1,t)}; \bm{\th}^*) } 
    \sum_{l (\neq k)} e^{ - \hat{\la}_{kl} ( X^{(1,t)}; \bm{\th}^* )} \om_{kl} (X^{(1,t)})
            \nn
            &+ \sum_{m \in [K]} e^{ - \hat{\la}_{km} (X^{(1,t)}; \bm{\th}^*) } 
            \sum_{l (\neq k)} e^{ - \hat{\la}_{kl} ( X^{(1,t)}; \bm{\th}^* )} \rh_{kl}^2 (X^{(1,t)})
    - ( \sum_{l (\neq k)} e^{- \hat{\la}_{kl} ( X^{(1,t)}; \bm{\th}^* ) } \rh_{kl} (X^{(1,t)}) )^2
        \Big\} \bigg] \nn
    &+ \mc{O}( \ep^3 ) \, , \label{eq:expansion with theta}
\end{align}
where $\rh_{kl}(X^{(1,t)}) := \bm{\varphi}^{\top} \cdot \nabla_{\bm{\th}} \hat{\la}_{kl} (X^{(1,t)}; \bm{\th}^*) $ and $\om_{kl} (X^{(1,t)}) := \bm{\varphi}^{\top} \cdot \nabla^2_{\bm{\th}} \hat{\la}_{kl} (X^{(1,t)}; \bm{\th}^*) \cdot \bm{\varphi} $. By definition of $\Th^*$, $\hat{\la}_{kl}(X^{(1,t)}; \bm{\th}^*) = \la_{kl} (X^{(1,t)}) = \log ( p( X^{(1,t)} | k )/p( X^{(1,t)} | l ) )$. Substituting this into (\ref{eq:expansion with theta}), we can see that the first order terms and the second order terms that contain $\om_{kl}$ are identically zero because of the asymmetry of $\hat{\la}_{kl}$. Therefore, 
\begin{align}
    & L[\hat{\la}(X^{(1,t)}; \bm{\th}^* + \ep \bm{\varphi})] 
    = L[\hat{\la}(X^{(1,t)}; \bm{\th}^*)] \nn
    &+ \f{1}{K} \sum_{k \in [K]} \int d X^{(1,t)} p ( X^{(1,t)} | k ) 
    \f{ \ep^2 }{ 2 ( \sum_{m \in [K]} e^{ - \hat{\la}_{km} (X^{(1,t)}; \bm{\th}^*) } )^2 } \nn
    &\times \Big\{
            \sum_{m \in [K]}  e^{ - \hat{\la}_{km} (X^{(1,t)}; \bm{\th}^*) } 
    \sum_{l (\neq k)} e^{ - \hat{\la}_{kl} ( X^{(1,t)}; \bm{\th}^* )} \rh_{kl}^2 (X^{(1,t)})
    - ( \sum_{l (\neq k)} e^{- \hat{\la}_{kl} ( X^{(1,t)}; \bm{\th}^* ) } \rh_{kl} (X^{(1,t)}) )^2
        \Big\} \nn
    &+ \mc{O}( \ep^3 ) \, . \label{eq:expansion with theta reduced}
\end{align}

Next, we define
\begin{align}
    I_k := 
    \sum_{m \in [K]} e^{ - \hat{\la}_{km} (X^{(1,t)}; \bm{\th}^*) } 
            \sum_{l (\neq k)} e^{ - \hat{\la}_{kl} ( X^{(1,t)}; \bm{\th}^* )} \rh_{kl}^2 (X^{(1,t)}) 
    - ( \sum_{l (\neq k)} e^{- \hat{\la}_{kl} ( X^{(1,t)}; \bm{\th}^* ) } \rh_{kl} (X^{(1,t)}) )^2 \, ,
\end{align}
so that 
\begin{align}
    &\hspace{10pt} L[\hat{\la}(X^{(1,t)}; \bm{\th}^* + \ep \bm{\varphi})] \nn
    &= L[\hat{\la}(X^{(1,t)}; \bm{\th}^*)]
    + \f{\ep^2}{2} \f{1}{K} \sum_{k \in [K]} \int d X^{(1,t)} p ( X^{(1,t)} | k )
    \f{ 
            I_k 
        }{ 
            ( \sum_{m \in [K]} e^{ - \hat{\la}_{km} (X^{(1,t)}; \bm{\th}^*) } )^2 } 
        + \mc{O}(\ep^3) \nn
    &=: L[\hat{\la}(X^{(1,t)}; \bm{\th}^*)] + \f{\ep^2}{2} J[\hat{\la}(X^{(1,t)}; \bm{\th}^*)] 
        + \mc{O}(\ep^3) \, . \label{eq:L=L+e^2J}
\end{align}
Here, we defined 
\begin{align}
    &J[\hat{\la}(X^{(1,t)}; \bm{\th}^*)] 
        := \f{1}{K} \sum_{k \in [K]} \int d X^{(1,t)} p ( X^{(1,t)} | k )
        \f{ 
            I_k 
        }{ 
            ( \sum_{m \in [K]} e^{ - \hat{\la}_{km} (X^{(1,t)}; \bm{\th}^*) } )^2 } \, .
\end{align}
In the following, we show that $J$ is positive to obtain $L[\hat{\la}(X^{(1,t)}; \bm{\th}^* + \ep \bm{\varphi})] > L[\hat{\la}(X^{(1,t)}; \bm{\th}^*)] $. We first see that $J$ is non-negative. Because
\begin{align}
    I_k &= \sum_{l (\neq k)} 
        ( \bm{\varphi}^{\top} \cdot \nabla_{\bm{\th}} \hat{\la}_{kl} (X^{(1,t)}; \bm{\th}^*) )^2
        e^{ -\hat{\la}_{kl} (X^{(1,t)}; \bm{\th}^*)}    
        \nn
    &+ \lt\{ 
            \bm{\varphi}^{\top} \cdot ( 
                \nabla_{\bm{\th}} \hat{\la}_{kl} (X^{(1,t)}; \bm{\th}^*) 
                - \nabla_{\bm{\th}} \hat{\la}_{kl^\prime} (X^{(1,t)}; \bm{\th}^*) )
        \rt\}^2  
        \sum_{\substack{l > l^\prime \\ l, l^\prime \neq k}}
        e^{ -\hat{\la}_{kl} (X^{(1,t)}; \bm{\th}^*) }
        e^{ - \hat{\la}_{kl^\prime} (X^{(1,t)}; \bm{\th}^*) } \nn
    &\geq \sum_{l (\neq k)} 
        ( \bm{\varphi}^{\top} \cdot \nabla_{\bm{\th}} \hat{\la}_{kl} (X^{(1,t)}; \bm{\th}^*) )^2
        e^{ -\hat{\la}_{kl} (X^{(1,t)}; \bm{\th}^*)} \, ,
\end{align}
we can bound $J$ from below:
\begin{align}
    & \hspace{10pt} J[\hat{\la}(X^{(1,t)}; \bm{\th}^*)] \nn 
    &\geq
    \f{1}{K} \sum_{k \in [K]} \sum_{l (\neq k)} 
        \int d X^{(1,t)} p ( X^{(1,t)} | k ) 
    \f{ e^{ -\hat{\la}_{kl} (X^{(1,t)}; \bm{\th}^*)}}{ ( \sum_{m \in [K]} e^{ - \hat{\la}_{km} (X^{(1,t)}; \bm{\th}^*) } )^2 }  
        ( \bm{\varphi}^{\top} \cdot \nabla_{\bm{\th}} \hat{\la}_{kl} (X^{(1,t)}; \bm{\th}^*) )^2 \label{eq:J bounded} \\
    &\,\,\,\, (\geq 0 )\, . \no
\end{align}
Note that each term in (\ref{eq:J bounded}) is non-negative; therefore, $J$ is non-negative. We next show that $J$ is non-zero to prove that $J>0$. By assumption, $\exists k^*,l^* \in [K]$ such that $\forall \bm{\varphi} \neq \bm{0}$,
\begin{align}
    & \bm{\varphi}^{\top} \cdot 
        \int d X^{(1,t)} p( X^{(1,t)} | k^* ) 
        \nabla_{\bm{\th}^*} \hat{\la}_{k^* l^*}(X^{(1,t)}; \bm{\th}^*)
        \nabla_{\bm{\th}^*} \hat{\la}_{k^* l^*}(X^{(1,t)}; \bm{\th}^*)^{\top}
        \cdot \bm{\varphi} \nn
    =& \int d X^{(1,t)} p( X^{(1,t)} | k^* ) ( \bm{\varphi}^{\top} \cdot \nabla_{\bm{\th}^*} \hat{\la}_{k^* l^*} (X^{(1,t)}; \bm{\th}^* ) )^2 \neq 0 \, . \nn
    \therefore& \,\,
        \int d X^{(1,t)} p ( X^{(1,t)} | k^* )
        \f{ 
            e^{ -\hat{\la}_{k^* l^*} (X^{(1,t)}; \bm{\th}^*)}
        }{
            ( \sum_{m \in [K]} e^{ - \hat{\la}_{k^* m} (X^{(1,t)}; \bm{\th}^*) } )^2 
        }  
        ( \bm{\varphi}^{\top} \cdot \nabla_{\bm{\th}} \hat{\la}_{k^* l^*} (X^{(1,t)}; \bm{\th}^*) )^2
        \neq 0 \, ,
\end{align}
because 
\begin{equation*}
    \f{
        e^{ -\hat{\la}_{k^* l^*} (X^{(1,t)}; \bm{\th}^*)}
    }{
        ( \sum_{m \in [K]} e^{ - \hat{\la}_{k^* m} (X^{(1,t)}; \bm{\th}^*) } )^2     
    } > 0 \, .    
\end{equation*}
Therefore, at least one term in (\ref{eq:J bounded}) is non-zero, meaning $(\ref{eq:J bounded}) \neq 0$ and $J [\hat{\la}(X^{(1,t)}; \bm{\th}^*)] > 0$. Thus, we conclude that $L[\hat{\la}(X^{(1,t)}; \bm{\th}^* + \ep \bm{\varphi})] > L[\hat{\la}(X^{(1,t)}; \bm{\th}^*)] $ via (\ref{eq:L=L+e^2J}).

Now, we have proven that $ L(\bm{\th}) > L(\bm{\th}^*)$ in the vicinity of $\bm{\th}^*$. For the other $\bm{\th} (\notin \Th^*)$, the inequality $ L(\bm{\th}) > L(\bm{\th}^*)$ immediately follows from Lemma \ref{lem:Non-parametric estimation: LSEL} because $\hat{\la}$ is not equal to $\la$ for such $\bm{\th} \notin \Th^*$ and $\la$ is the unique minimum of $L[\hat{\la}]$. This concludes the proof.
\end{proof}

Finally, we prove Theorem \ref{thm:consistency of LSEL} with the help of Lemma \ref{lem:Th* minimizes LSEL}.
\begin{proof}
To prove the consistency, we show that $P ( \hat{\bm{\th}}_S \notin \Th^* ) (= P ( \{ \om \in \Om | \hat{\bm{\th}}_S (\om) \notin \Th^* \} )) \xrightarrow[]{M\rightarrow \infty} 0$, where $ \hat{\bm{\th}}_S$ is the empirical risk minimizer, $M$ is the sample size, $P$ is the probability measure, and $\Om$ is the sample space of the underlying probability space. By Lemma \ref{lem:Th* minimizes LSEL}, if $\bm{\th} \notin \Th^*$, then there exists $\de > 0$ such that $L (\bm{\th}) > L (\bm{\th}^*) + \de(\bm{\th})$. Therefore,
\begin{align}
    &\hspace{10pt} \{ \om \in \Om | \hat{\bm{\th}}_S (\om) \notin \Th^* \} 
    \subset \{ \om \in \Om | L ( \hat{\bm{\th}}_S (\om)) > L (\bm{\th}^*) + \de (\hat{\bm{\th}}_S) \} \nn
    &\therefore \,\,\,P \lt(\hat{\bm{\th}}_S \notin \Th^* \rt) \leq P \lt( L ( \hat{\bm{\th}}_S ) > L (\bm{\th}^*) + \de ( \hat{\bm{\th}}_S ) \rt)
        \label{eq:ineq1} \, .
\end{align}
We bound the right-hand side in the following.
\begin{align}
    \hspace{10pt} L ( \hat{\bm{\th}}_S ) - L (\bm{\th}^*) 
    &= L ( \hat{\bm{\th}}_S ) 
        - \hat{L}_S (\bm{\th}^*)
        + \hat{L}_S (\bm{\th}^*)
        - L (\bm{\th}^*) \nn
    &\leq L ( \hat{\bm{\th}}_S ) 
        - \hat{L}_S (\hat{\bm{\th}}_S)
        + \hat{L}_S (\bm{\th}^*)
        - L (\bm{\th}^*) \no \, .
\end{align}
Therefore,
\begin{align}
    L ( \hat{\bm{\th}}_S ) - L (\bm{\th}^*) 
    &= | L ( \hat{\bm{\th}}_S ) - L (\bm{\th}^*) |\nn
    &\leq |L ( \hat{\bm{\th}}_S ) 
        - \hat{L}_S (\hat{\bm{\th}}_S) |
        + | \hat{L}_S (\bm{\th}^*)
        - L (\bm{\th}^*) | \nn 
    &\leq \,2 \, \mr{sup}_{\bm{\th}} \lt| L ( \bm{\th} ) 
        - \hat{L}_S (\bm{\th}) \rt| \no \, .
\end{align}
Thus,
\begin{align}
    \de ( \hat{\bm{\th}}_S ) < L (\hat{\bm{\th}}_S) - L ( \bm{\th}^* )
    \Longrightarrow& \de ( \hat{\bm{\th}}_S ) 
        <  \,2 \, \mr{sup}_{\bm{\th}} \lt| L ( \bm{\th} ) - \hat{L}_S (\bm{\th}) \rt| \no \, .
\end{align}
Hence,
\begin{align}
    P \Big( L ( \hat{\bm{\th}}_S ) > L (\bm{\th}^*) + \de ( \hat{\bm{\th}}_S ) \Big) 
    \leq P \lt( \de ( \hat{\bm{\th}}_S ) 
        <  \,2 \, \mr{sup}_{\bm{\th}} \lt| L ( \bm{\th} ) - \hat{L}_S (\bm{\th}) \rt| \rt) \, . \no
\end{align}
By the assumption that $\hat{L}_S (\bm{\th})$ converges in probability uniformly over $\bm{\th}$ to $L (\bm{\th})$, the right-hand side is bounded above by an arbitrarily small $ \ep > 0$ for sufficiently large $M$:
\begin{equation}
    P \lt(
        \de ( \hat{\bm{\th}}_S ) 
        <  \,2 \, \mr{sup}_{\bm{\th}} \lt| L ( \bm{\th} ) - \hat{L}_S (\bm{\th}) \rt| 
    \rt) < \ep \label{eq:ineq2} \, .
\end{equation}
Combining (\ref{eq:ineq1}) and (\ref{eq:ineq2}), we conclude that $\forall \ep >0, \exists n \in \mbn$ s.t. $\forall M > n, P (\hat{\bm{\th}}_S \notin \Th^*) < \ep$.
\end{proof}

%% file: supp_Variants.tex
\section{Modified LSEL and Logistic Loss} \label{app: Modified LSEL and Logistic Loss}
In this appendix, we first discuss the effect of the prior ratio term $\log ( \hat{p}( y = k ) / \hat{p}( y = l ) ) = : \log \hat{\nu}_{kl}$ in the M-TANDEM and M-TANDEMwO formula (Appendix \ref{app: Modified LSEL and Its Consistency}). We then define the logistic loss used in the main text (Appendix \ref{app: Logistic Loss and Its Consistency}). 

\subsection{Modified LSEL and Consistency} \label{app: Modified LSEL and Its Consistency}
In the main text, we ignore the prior ratio term $\log \hat{\nu}_{kl}$ (Section \ref{sec: Overall Architecture: MSPRT-TANDEM}). Strictly speaking, this is equivalent to the definition of the following \textit{modified LSEL (modLSEL)}: 
\begin{align} \label{eq:modLSEL}
    L_{\rm modLSEL} [\tilde{\la}] 
    &:= 
    \f{1}{T}\sum_{t\in[T]} \underset{\substack{( X^{(1,t)}, y ) \\ \sim P( X^{(1,t)}, y )}}{\mbe} \bigg[ \log ( 
        1 + \sum_{ k ( \neq y ) } \nu^{-1}_{y k} e^{ - \tilde{\la}_{y k} ( X^{(1,t)} ) }
    ) \bigg] \nn
    &= \f{1}{T}\sum_{t\in[T]} \sum_{y \in [K]} \int d X^{(1,t)} p ( X^{(1,t)} | y ) p ( y ) 
    \log ( 
        1 + \sum_{ k ( \neq y ) } \nu^{-1}_{y k} e^{ - \tilde{\la}_{y k} ( X^{(1,t)} ) }
    )
\end{align}
where $\nu_{kl} = p(y=k) / p(y=l)$ $(k,l \in [K])$ is the prior ratio matrix. The empirical approximation of $L_{\rm modLSEL}$ is
\begin{align} \label{eq:modLSEL with empirical approx.}
    &\hat{L}_{\rm modLSEL} (\bm{\th}; S)
    := \f{1}{MT} \sum_{i \in [M]} \sum_{t\in[T]}  \log ( 
        1 + \sum_{ k ( \neq y_i ) } \hat{\nu}^{-1}_{y_i k} e^{ - \hat{\la}_{y_i k} ( X_i^{(1,t)}; \bm{\th} ) }
    ) \, .
\end{align}
where $\hat{\nu}_{kl} := M_k / M_l $ ($k,l \in [K]$). $M_k$ denotes the sample size of class $k$, i.e., $M_k := |\{ i \in [M] | y_i = k \}|$. (\ref{eq:modLSEL}) is a generalization of the logit adjustment \citeApp{menon2021ICLR_long-tail_logit_adjustment} to the LSEL and helps us to train neural networks on imbalanced datasets. 

We can prove the consistency even for the modified LSEL, given an additional assumption (d):
\input{supp_Variants_snip_ConsistencyLogisticLoss}

\subsection{Logistic Loss and Consistency} \label{app: Logistic Loss and Its Consistency}
We use the following logistic loss for DRME in the main text:
\begin{align}
    &\hat{L}_{\rm logistic} ( \bm{\th} ; S ) 
    := \f{1}{KT} \sum_{k \in [K]} \sum_{t \in [T]} \f{1}{M_k} \sum_{i \in I_k} 
        \f{1}{K-1} \sum_{l (\neq k)} 
        \log (
            1 + e^{- \hat{\la}_{kl} (X^{(1,t)}_i; \bm{\th})}
        )  \label{eq: logistic loss with empirical approximation} \, .
\end{align}
Note that $\hat{L}_{\operatorname{logistic}}$ resembles the LSEL but is defined as the sum of the logarithm of the exponential (sum-log-exp), not log-sum-exp. We can prove the consistency and the proof is completely parallel to, and even simpler than, that of Theorem \ref{thm:consistency of LSEL}; therefore, we omit the proof to avoid redundancy. $\hat{L}_{\rm logistic}$ approaches
\begin{align}
    L_{\rm logistic} [ \la ] 
    = \f{1}{KT} \sum_{k \in [K]} \sum_{t \in [T]} 
    \underset{\substack{X^{(1,t)} \\ \sim p (X^{(1,t)}| y=k) }}{\mbe}
        \lt[
            \f{1}{K-1} \sum_{l (\neq k)} \log ( 1 + e^{- \la_{k l} (X^{(1,t)})} )
        \rt] 
\end{align}
as $M \rightarrow \infty$. 

Additionally, we can define the modified logistic loss as
\begin{align}
    \hat{L}_{\rm modlogistic} ( \bm{\th} ; S ) = 
    \f{1}{MT} \sum_{i \in [M]} \sum_{t \in [T]}    
        \f{1}{K-1} \sum_{l (\neq y_i)} 
        \log ( 
            1 + \hat{\nu}_{y_i l}^{-1} e^{ - \hat{\la}_{y_i l} (X^{(1,t)}_i ; \bm{\th}) } 
        )  \label{eq: modlogistic loss with empirical approximation}
\end{align}
We can prove the consistency in a similar way to Theorem \ref{thm:consistency of modLSEL}. $\hat{L}_{\rm modlogistic}$ approaches
\begin{align}
    L_{\rm modlogistic} [ \la ] = \f{1}{T} \sum_{t \in [T]}
    \underset{\substack{( X^{(1,t)}, y ) \\ \sim p ( X^{(1,t)}, y )}}{\mbe} 
        \lt[
            \f{1}{K-1} \sum_{l (\neq y)} \log ( 1 + \nu_{yl}^{-1} e^{-\la_{yl} (X^{(1,t)}) } )
        \rt] \, .
\end{align}
as $M \rightarrow \infty$. 

Our empirical study shows that the LSEL is better than the logistic loss (Figure \ref{fig: DRE Losses vs LSEL} and Appendix \ref{app: Performance Comparison of LSELwith logistic}), potentially because the LSEL weighs hard classes more than the logistic loss (Section \ref{sec: Hard class weighting effect of the LSEL}).

\clearpage
\section{Performance Comparison of LSEL and Logistic Loss} \label{app: Performance Comparison of LSELwith logistic}
Figure \ref{fig: LSEL vs logistic} provides the performance comparison of the LSEL (\ref{eq:LLLR-E with empirical approx.}) and the logistic loss (\ref{eq: logistic loss with empirical approximation}). The LSEL is consistently better than or comparable with the logistic loss, which is potentially because of the hard class weighting effect.
\begin{figure}[H]
    \vskip 0.2in
    \begin{minipage}[b]{0.5\linewidth}
        \centering
        \includegraphics[width=\columnwidth, keepaspectratio]
        {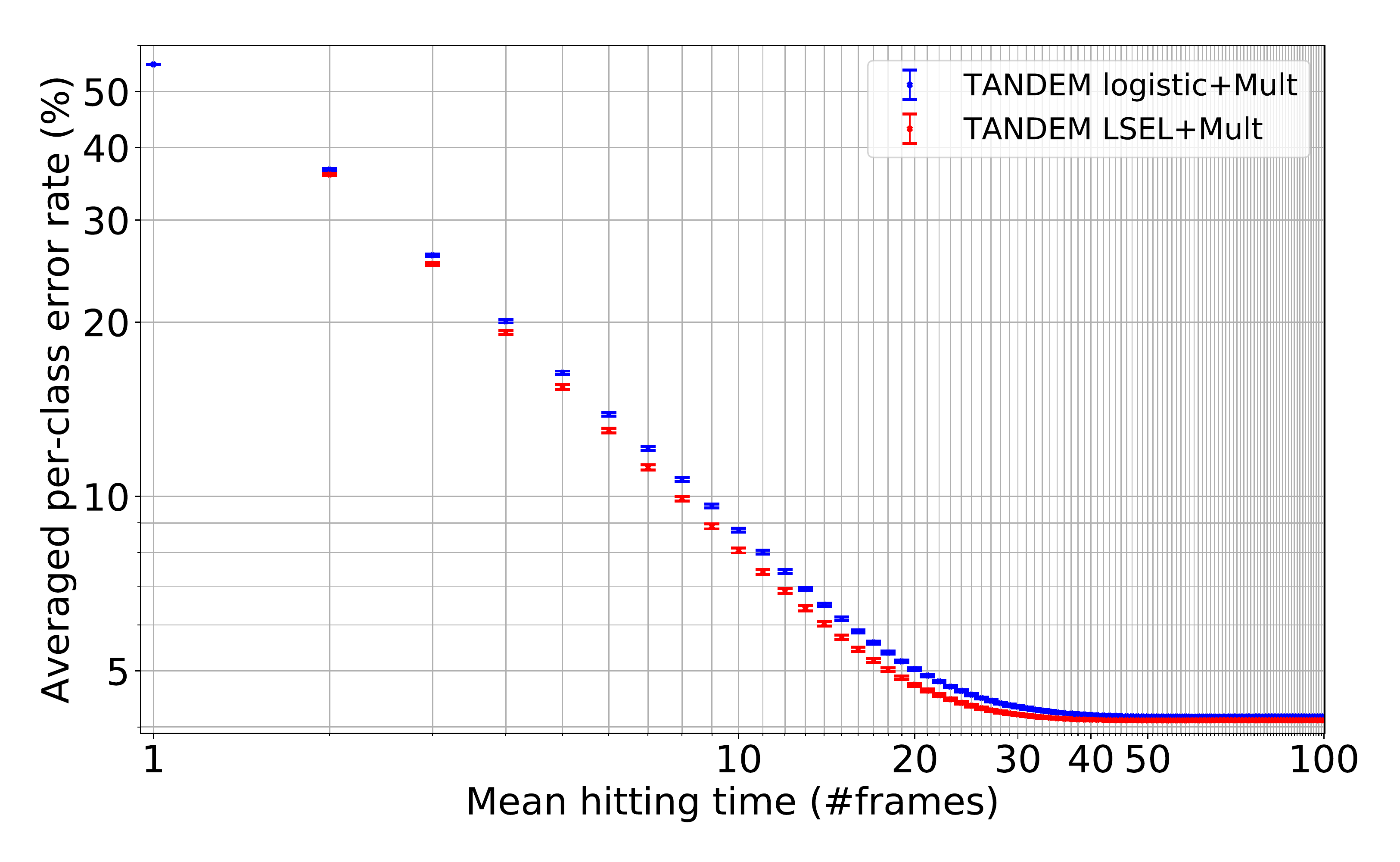} 
    \end{minipage}
    \begin{minipage}[b]{0.5\linewidth}
        \centering
        \includegraphics[width=\columnwidth, keepaspectratio]
        {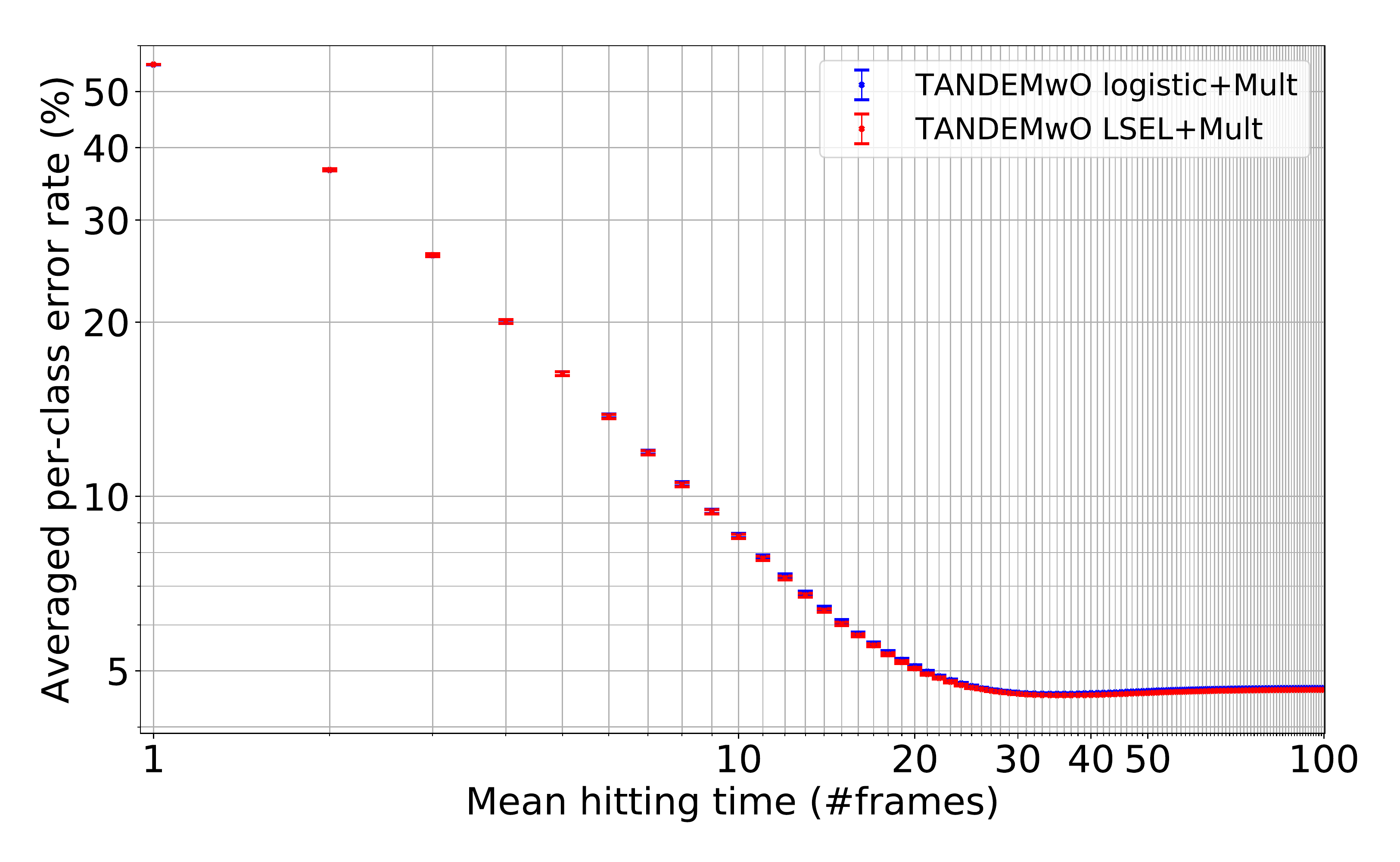} 
    \end{minipage}
    \begin{minipage}[b]{0.5\linewidth}
        \centering
        \includegraphics[width=\columnwidth, keepaspectratio]
        {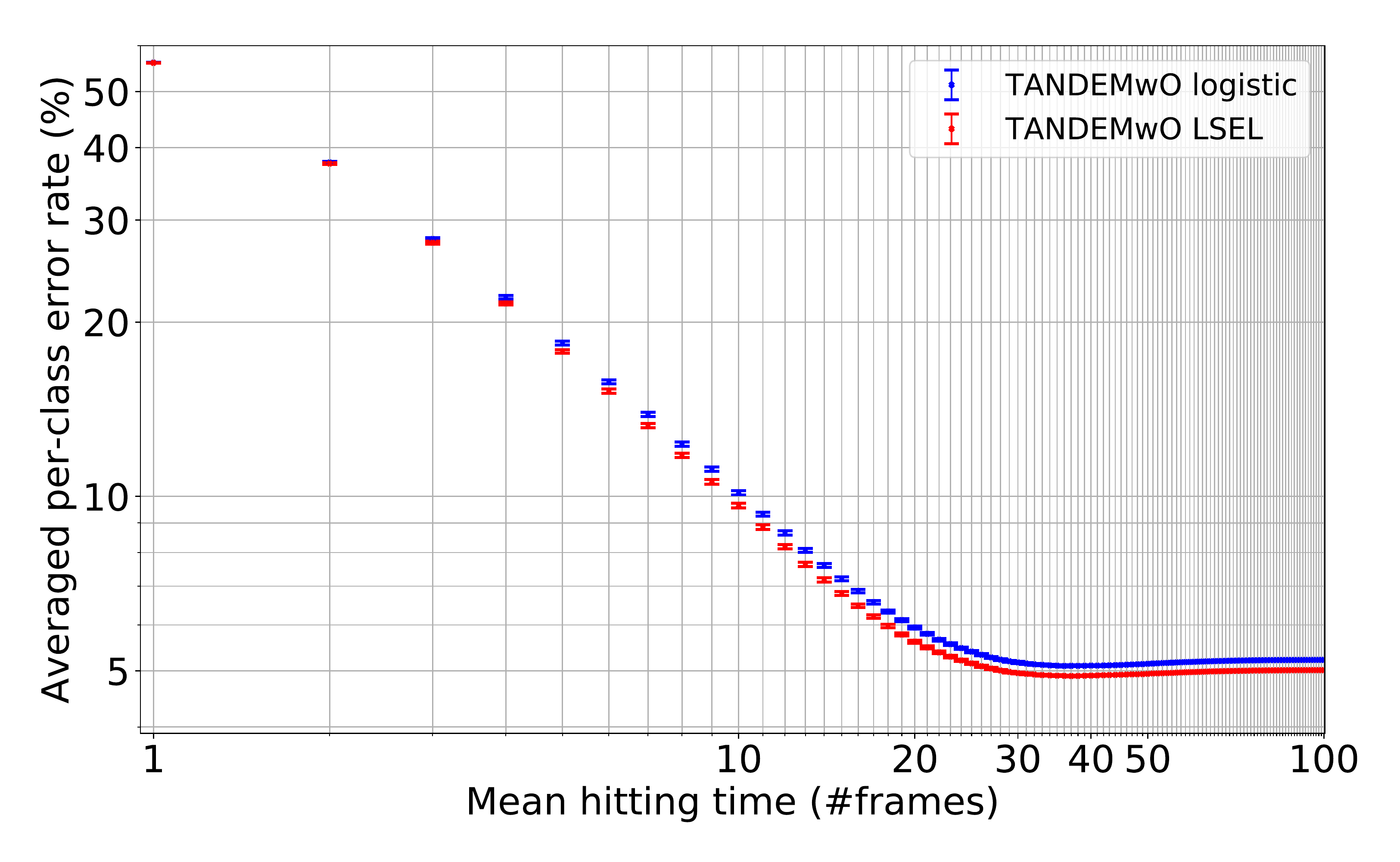} 
    \end{minipage}
    \begin{minipage}[b]{0.5\linewidth}
        \centering
        \includegraphics[width=\columnwidth, keepaspectratio]
        {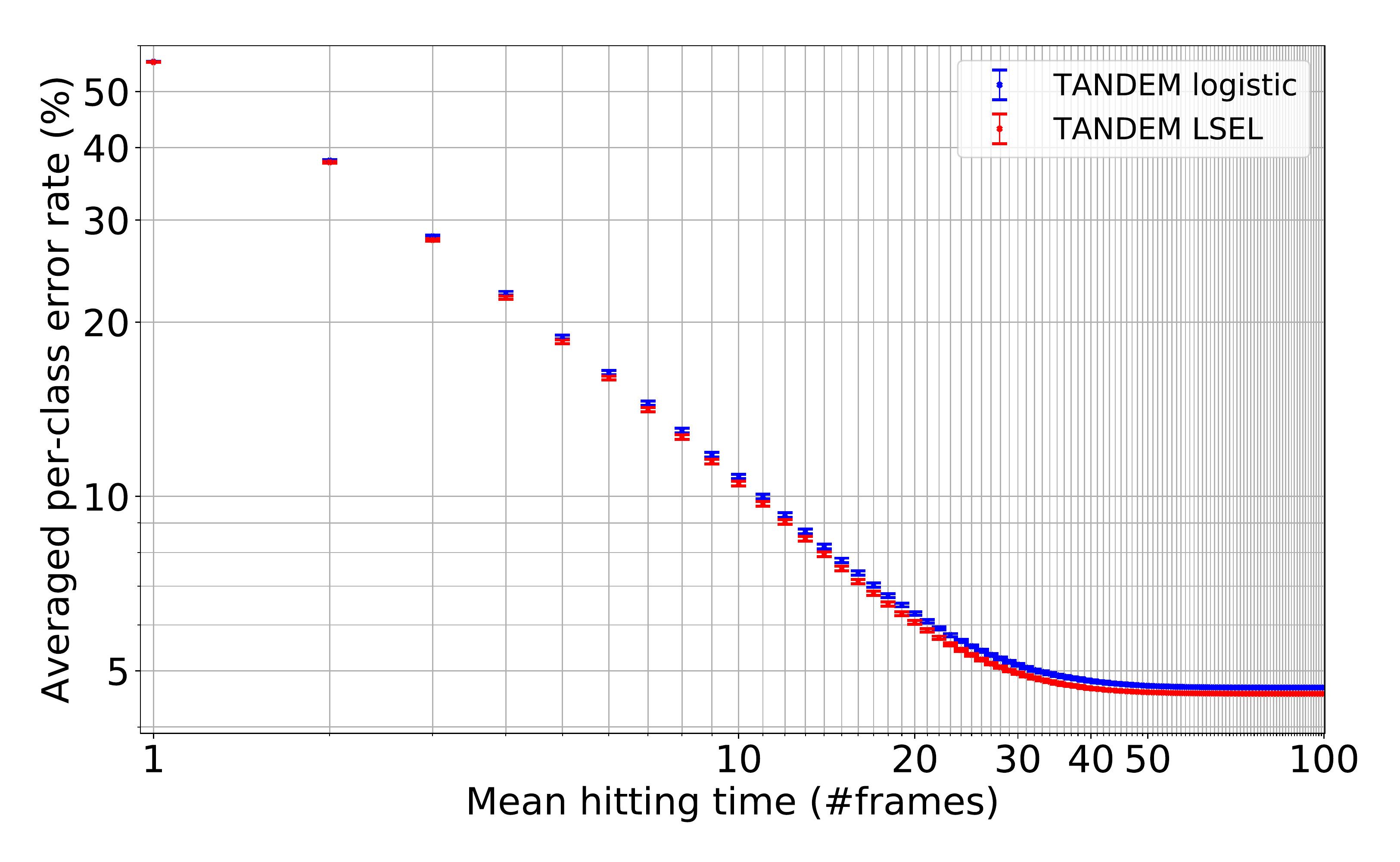} 
    \end{minipage}
    \caption{
        \textbf{LSEL v.s. Logistic Loss.} The LSEL is consistently better than or at least comparable with the logistic loss. The dataset is NMNIST-100f. The error bar is the SEM. ``TANDEM'' means that the model is trained with the M-TANDEM formula, ``TANDEMwO'' means that the model is trained with the M-TANDEMwO formula, and ``Mult'' means that the multiplet loss is simultaneously used.
    }\label{fig: LSEL vs logistic}
    \vskip -0.2in
\end{figure}

%% file: supp_Variants_snip_ConsistencyLogisticLoss.tex
\begin{theorem}[Consistency of the modLSEL]\label{thm:consistency of modLSEL}
    Let $L(\bm{\th})$ and $\hat{L}_S (\bm{\th})$ denote $L_{\rm \text{modLSEL}} [\hat{\la} (\cdot; \bm{\th})]$ and $\hat{L}_{\rm \text{modLSEL}} (\bm{\th} ;S)$, respectively. 
    Let $\hat{\bm{\th}}_S$ be the empirical risk minimizer of $\hat{L}_S$; namely, $\hat{\bm{\th}}_S := \mr{argmin}_{\bm{\th}} \hat{L}_S (\bm{\th})$.
    Let $\Th^* := \{ \bm{\th}^* \in \mbr^{d_{\th}} | \hat{\la} ( X^{(1,t)} ; \bm{\th}^* ) = \la (X^{(1,t)}) \,\, (\forall t \in [T]) \}$ be the target parameter set. Assume, for simplicity of proof, that each $\bm{\th}^*$ is separated in $\Th^*$; i.e., $\exists \de > 0$ such that $B ( \bm{\th}^* ; \de ) \cap B ( \bm{\th}^{*\prime}; \de ) = \emptyset$ for arbitrary $\bm{\th}^*$ and $\bm{\th}^{*\prime} \in \Th^*$, where $B(\bm{\th}; \de)$ denotes an open ball at center $\bm{\th}$ with radius $\de$.
    Define
    \begin{equation}
        \hat{L}_S^\prime (\bm{\th}) := \f{1}{MT} \sum_{i \in [M]} \sum_{t\in[T]}  \log \lt( 
            1 + \sum_{ k ( \neq y_i ) } \nu^{-1}_{y_i k} e^{ - \hat{\la}_{y_i k} ( X_i^{(1,t)}; \bm{\th} ) }
        \rt) \, 
    \end{equation}
    ($\hat{\nu}$ is replaced by $\nu$ in $\hat{L}_S$). Assume the following three conditions:
    \begin{itemize}
        \item[(a)] $\forall k, l \in [K]$, $\forall t\in[T]$, $p(X^{(1,t)} | k) = 0 \Longleftrightarrow p(X^{(1,t)} | l) = 0$.
        
        \item[(b$^\prime$)] $\mr{sup}_{\bm{\th}} | \hat{L}_S^\prime (\bm{\th}) - L (\bm{\th}) | \xrightarrow[M\rightarrow\infty]{P} 0$; i.e., $\hat{L}_S^\prime (\bm{\th})$ converges in probability uniformly over $\bm{\th}$ to $L(\bm{\th})$.\footnote{
            Specifically,
            $\forall \ep > 0, P( 
                \mr{sup}_{\bm{\th}} | \hat{L}_S^\prime (\bm{\th}) - L (\bm{\th}) | > \ep
            ) \xrightarrow{M \rightarrow \infty} 0$.
        }

        \item[(c)] For all $\th^* \in \Th^*$, there exist $t \in [T]$, $k\in[K]$ and $l\in[K]$, such that the following $d_\th \times d_\th$ matrix is full-rank:        
        \begin{align}
            &\int d X^{(1,t)} p( X^{(1,t)} | k )
            \nabla_{\bm{\th}^*} \hat{\la}_{k l}(X^{(1,t)}; \bm{\th}^*)
                \nabla_{\bm{\th}^*} \hat{\la}_{k l}(X^{(1,t)}; \bm{\th}^*)^{\top} \, .
        \end{align}
        
        \item[(d)] $\mr{sup}_{\bm{\th}} | \hat{L}_S^\prime (\bm{\th}) - \hat{L}_S (\bm{\th}) | \xrightarrow[M\rightarrow\infty]{P} 0$.
        
    \end{itemize}
    Then, $P( \hat{\bm{\th}}_S \notin \Th^* ) \xrightarrow{M\rightarrow\infty}0$; i.e., $\hat{\bm{\th}}_S$ converges in probability into $\Th^*$.
\end{theorem}
(b) is now modified to (b$^\prime$) and (d) is added to the assumptions of Theorem \ref{thm:consistency of LSEL}. (b$^\prime$) can be satisfied under the standard assumptions of the uniform law of large numbers. (d) may be proven under some appropriate assumptions, but we simply accept it here.

\begin{proof}
We prove all the statements only for an arbitrary $t\in[T]$ and omit $\f{1}{T}\sum_{t\in[T]}$ from $L$ and $\hat{L}_S$, as is done in the proof of Theorem \ref{thm:consistency of LSEL}.
In the same way as Appendix \ref{app:Proof of Consistency Theorem of LSEL}, we first provide two lemmas:
\begin{lemma}[Non-parametric estimation] \label{lem:Non-parametric estimation: modLSEL}
    Assume that for all $k, l \in [K]$ , $p(X^{(1,t)} | k) = 0 \Longleftrightarrow p(X^{(1,t)} | l) = 0$. Then, $L[\ti{\la}]$ attains the unique minimum at $\ti{\la} = \la$.
\end{lemma}
\begin{lemma}[$\Th^*$ minimizes $L$]\label{lem:Th* minimizes modLSEL}
    Assume that for all $\th^* \in \Th^*$, there exist $k^*\in[K]$ and $l^*\in[K]$, such that the following $d_\th \times d_\th$ matrix is full-rank:
        \begin{align}
            &\int d X^{(1,t)} p( X^{(1,t)} | k^* ) 
            \nabla_{\bm{\th}^*} \hat{\la}_{k^* l^*}(X^{(1,t)}; \bm{\th}^*)
                \nabla_{\bm{\th}^*} \hat{\la}_{k^* l^*}(X^{(1,t)}; \bm{\th}^*)^{\top} \, .
        \end{align}
    Then, for any $\bm{\th} \notin \Th^*$,
    \begin{equation*}
        L(\bm{\th}) > L(\bm{\th}^*) \,\,\,\,\, (\forall \bm{\th}^* \in \Th^*) \, ,
    \end{equation*}
    meaning that $\Th^* = \mr{argmin}_{\bm{\th}} L(\bm{\th})$.
\end{lemma}
We skip the proofs because they are completely parallel to the proof of Lemma \ref{lem:Non-parametric estimation: LSEL} and Lemma \ref{lem:Th* minimizes LSEL}.

To prove the consistency, we show that $P ( \hat{\bm{\th}}_S \notin \Th^* ) (= P ( \{ \om \in \Om | \hat{\bm{\th}}_S (\om) \notin \Th^* \} )) \xrightarrow[]{M\rightarrow \infty} 0$, where $ \hat{\bm{\th}}_S$ is the empirical risk minimizer on the random training set $S$, $M$ is the sample size, $P$ is the probability measure, and $\Om$ is the sample space of the underlying probability space. By Lemma \ref{lem:Th* minimizes modLSEL}, if $\bm{\th} \notin \Th^*$, then there exists $\de > 0$ such that $L (\bm{\th}) > L (\bm{\th}^*) + \de(\bm{\th})$. Therefore,
\begin{align}
    &\{ \om \in \Om | \hat{\bm{\th}}_S (\om) \notin \Th^* \}
    \subset \{ \om \in \Om | L ( \hat{\bm{\th}}_S (\om)) > L (\bm{\th}^*) + \de (\hat{\bm{\th}}_S) \} \nn
    \therefore \hspace{5pt} & P \lt(\hat{\bm{\th}}_S \notin \Th^* \rt) \leq P \lt( L ( \hat{\bm{\th}}_S ) > L (\bm{\th}^*) + \de ( \hat{\bm{\th}}_S ) \rt)
        \label{eq:ineq1_} \, .
\end{align}
We bound the right-hand side in the following.
\begin{align}
    L ( \hat{\bm{\th}}_S ) - L (\bm{\th}^*)
    &= L ( \hat{\bm{\th}}_S ) 
        - \hat{L}_S (\bm{\th}^*)
        + \hat{L}_S (\bm{\th}^*)
        - L (\bm{\th}^*) \nn
    &\leq L ( \hat{\bm{\th}}_S ) 
        - \hat{L}_S (\hat{\bm{\th}}_S)
        + \hat{L}_S (\bm{\th}^*)
        - L (\bm{\th}^*) \nn
    &= L ( \hat{\bm{\th}}_S ) 
        - \hat{L}_S^\prime (\hat{\bm{\th}}_S)
        + \hat{L}_S^\prime (\hat{\bm{\th}}_S)
        - \hat{L}_S (\hat{\bm{\th}}_S) \nn
    &\hspace{10pt}+ \hat{L}_S (\bm{\th}^*)
        - \hat{L}_S^\prime (\bm{\th}^*)
        + \hat{L}_S^\prime (\bm{\th}^*)
        - L (\bm{\th}^*) \no \, .
\end{align}
Therefore,
\begin{align} 
    L ( \hat{\bm{\th}}_S ) - L (\bm{\th}^*) 
    &= | L ( \hat{\bm{\th}}_S ) - L (\bm{\th}^*) | \nn
    &\leq | L ( \hat{\bm{\th}}_S ) 
        - \hat{L}_S^\prime (\hat{\bm{\th}}_S) |
        + | \hat{L}_S^\prime (\hat{\bm{\th}}_S)
        - \hat{L}_S (\hat{\bm{\th}}_S) | \nn
    &\hspace{10pt}+ | \hat{L}_S (\bm{\th}^*)
        - \hat{L}_S^\prime (\bm{\th}^*) |
        + | \hat{L}_S^\prime (\bm{\th}^*)
        - L (\bm{\th}^*) | \nn
    &\leq \,2 \, \mr{sup}_{\bm{\th}} \lt| L ( \bm{\th} ) 
        - \hat{L}_S (\bm{\th}) \rt|
        + 2 \, \mr{sup}_{\bm{\th}} \lt| \hat{L}_S^\prime (\bm{\th}) 
        - \hat{L}_S (\bm{\th}) \rt| \no \, .
\end{align}
Thus,
\begin{align}
    &\hspace{10pt}\de ( \hat{\bm{\th}}_S ) < L (\hat{\bm{\th}}_S) - L ( \bm{\th}^* )
    \Longrightarrow \de ( \hat{\bm{\th}}_S ) 
        <  \,2 \, \mr{sup}_{\bm{\th}} \lt| L ( \bm{\th} ) - \hat{L}_S (\bm{\th}) \rt|
    + 2 \, \mr{sup}_{\bm{\th}} \lt| \hat{L}_S^\prime (\bm{\th}) 
    - \hat{L}_S (\bm{\th}) \rt|\no \, .
\end{align}
Hence,
\begin{align}
    P ( L ( \hat{\bm{\th}}_S ) > L (\bm{\th}^*) + \de ( \hat{\bm{\th}}_S ) )
    \leq P \Big( \de ( \hat{\bm{\th}}_S ) 
        <  \,2 \, \mr{sup}_{\bm{\th}} \lt| L ( \bm{\th} ) - \hat{L}_S (\bm{\th}) \rt|
    + 2 \, \mr{sup}_{\bm{\th}} \lt| \hat{L}_S^\prime (\bm{\th}) 
        - \hat{L}_S (\bm{\th}) \rt| \Big) \, . \no
\end{align}
Recall that by assumption, $2 \, \mr{sup}_{\bm{\th}} \lt| L ( \bm{\th} ) - \hat{L}_S (\bm{\th}) \rt|$ and $2 \, \mr{sup}_{\bm{\th}} \lt| \hat{L}_S^\prime (\bm{\th}) - \hat{L}_S (\bm{\th}) \rt|$ converge in probability to zero; hence, $2 \, \mr{sup}_{\bm{\th}} \lt| L ( \bm{\th} ) - \hat{L}_S (\bm{\th}) \rt| + 2 \, \mr{sup}_{\bm{\th}} \lt| \hat{L}_S^\prime (\bm{\th}) - \hat{L}_S (\bm{\th}) \rt|$ converge in probability to zero because in general,
\begin{equation*}
    a_n \xrightarrow[n \rightarrow \infty]{P} 0 \,\,\, \text{and} \,\,\, b_n \xrightarrow[n \rightarrow \infty]{P} 0 
    \Longrightarrow a_n + b_n \xrightarrow[n \rightarrow \infty]{P} 0 \, ,
\end{equation*}
where $\{a_n\}$ and $\{ b_n \}$ are sequences of random variables. By definition of convergence in probability, for sufficiently large sample sizes $M$,
\begin{align}
    P \Big(
        \de ( \hat{\bm{\th}}_S ) 
        <  \, 2 \, \mr{sup}_{\bm{\th}} \lt| L ( \bm{\th} ) - \hat{L}_S (\bm{\th}) \rt|
    + 2 \, \mr{sup}_{\bm{\th}} \lt| \hat{L}_S^\prime (\bm{\th}) - \hat{L}_S (\bm{\th}) \rt|
        \Big) < \ep \label{eq:ineq2_} \, .
\end{align}
Combining (\ref{eq:ineq1_}) and (\ref{eq:ineq2_}), we conclude that $\forall \ep >0, \exists n \in \mbn$ s.t. $\forall M > n, P (\hat{\bm{\th}}_S \notin \Th^*) < \ep$.
\end{proof}

%% file: supp_TANDEMvsOblivion.tex
\section{M-TANDEM vs. M-TANDEMwO Formulae} \label{app: TANDEM vs TANDEMwO}
The M-TANDEM and M-TANDEMwO formulae enable to efficiently train RNNs on long sequences, which often cause vanishing gradients \cite{Hochreiter_GradientVanishing}. In addition, if a class signature is localized within a short temporal interval, not all frames can be informative \cite{xing2011extracting_localized_class_signature, mcgovern2011identifying, ghalwash2012early_localized_class_signature, ghalwash2013extraction_localized_class_signature, ghalwash2014utilizing_localized_class_signature, karim2019framework_localized_class_signature}. The M-TANDEM and M-TANDEMwO formulae alleviate these problems.

Figure \ref{fig: TANDEM and TANDEMwO illustrate} highlights the differences between the M-TANDEM and M-TANDEMwO formulae. The M-TANDEM formula covers all the timestamps, while the M-TANDEMwO formula only covers the last $N+1$ timestamps. In two-hypothesis testing, the M-TANDEM formula is the canonical generalization of Wald's i.i.d. SPRT, because for $N=0$ (i.i.d.), the M-TANDEM formula reduces to $\hat{\la}_{1,2} (X^{(1,T)}) = \sum_{t=1}^{T} \log ( p(x^{(t)} | 1 ) / p(x^{(t)} | 2 ) )$, which is used in the classical SPRT \citeApp{wald1945TartarBook492}, while the M-TANDEMwO formula reduces to a sum of frame-by-frame scores when $N=0$.

Figure \ref{fig: TANDEM vs TANDEMwO} compares the performance of the M-TANDEM and M-TANDEMwO formulae on three datasets: NMNIST, NMNIST-H, and NMNIST-100f. NMNIST \citeApp{SPRT-TANDEM} is similar to NMNIST-H but has much weaker noise. On relatively short sequences (NMNIST and NMNIST-H), the M-TANDEMwO formula is slightly better than or much the same as the M-TANDEM formula. On longer sequences (NMNIST-100f), the M-TANDEM formula outperforms the M-TANDEMwO formula; the latter slightly and gradually increases the error rate in the latter half of the sequences. 
In summary, the performance of the M-TANDEM and M-TANDEMwO formulae depends on the sequence length of the training datasets, and we recommend using the M-TANDEM formula as the first choice for long sequences ($\gtrsim 100$ frames) and the M-TANDEM formula for short sequences ($\sim 10$ frames).

\begin{figure*}[ht]
    \begin{center}
    \centerline{\includegraphics[width=400pt]
    {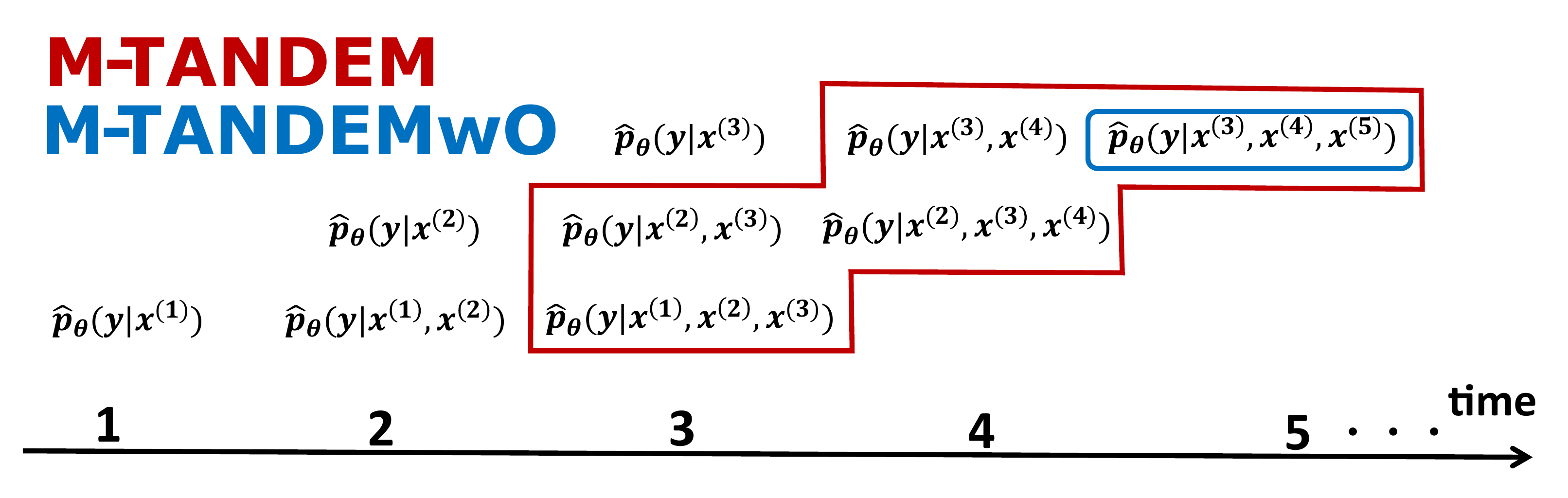}}
    \caption{
        \textbf{M-TANDEM v.s. M-TANDEMwO with $\bm{N=2}$.} 
        The posterior densities encircled in red and blue are used in the M-TANDEM and M-TANDEMwO formulae, respectively. We can see that the M-TANDEM formula covers all the frames, while the M-TANDEMwO formula covers only the last $N+1$ frames.
        } \label{fig: TANDEM and TANDEMwO illustrate}
    \end{center}
\end{figure*}

\begin{figure}[ht]
    \centering
    \begin{minipage}[b]{0.8\linewidth}
        \centering
        \includegraphics[width=\columnwidth, keepaspectratio]
        {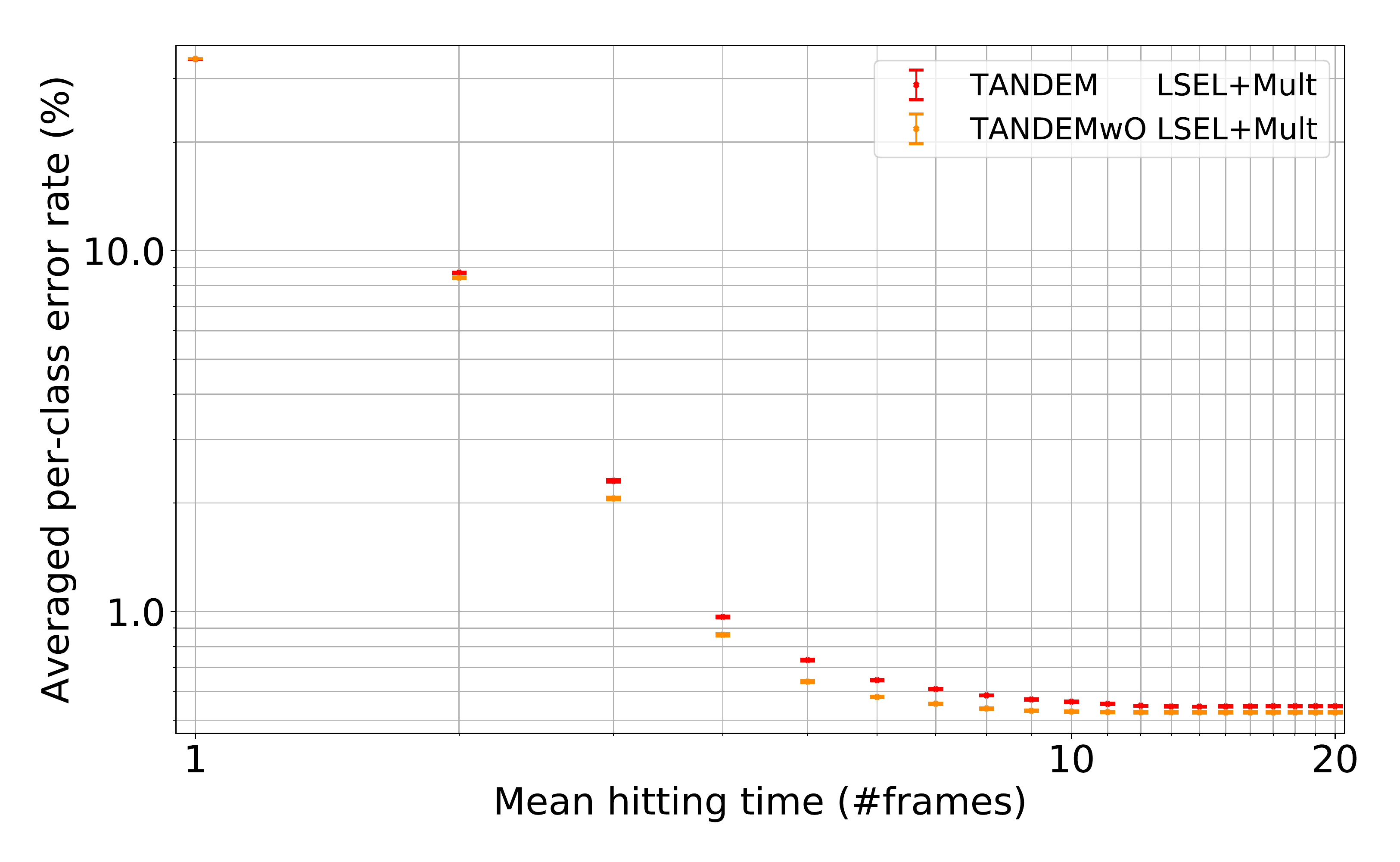} 
    \end{minipage}
    \begin{minipage}[b]{0.8\linewidth}
        \centering
        \includegraphics[width=\columnwidth, keepaspectratio]
        {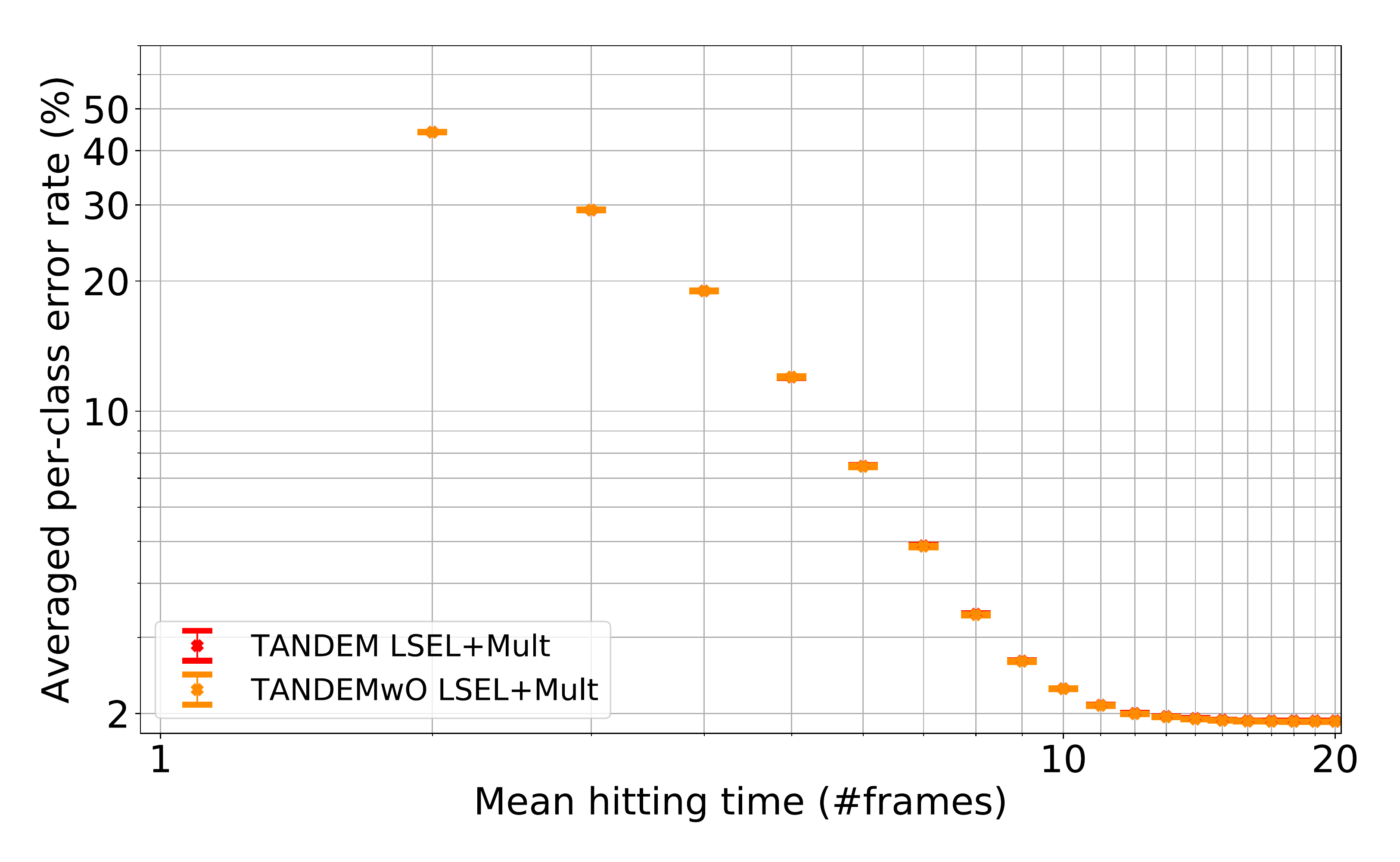} 
    \end{minipage}
    \begin{minipage}[b]{0.8\linewidth}
        \centering
        \includegraphics[width=\columnwidth, keepaspectratio]
        {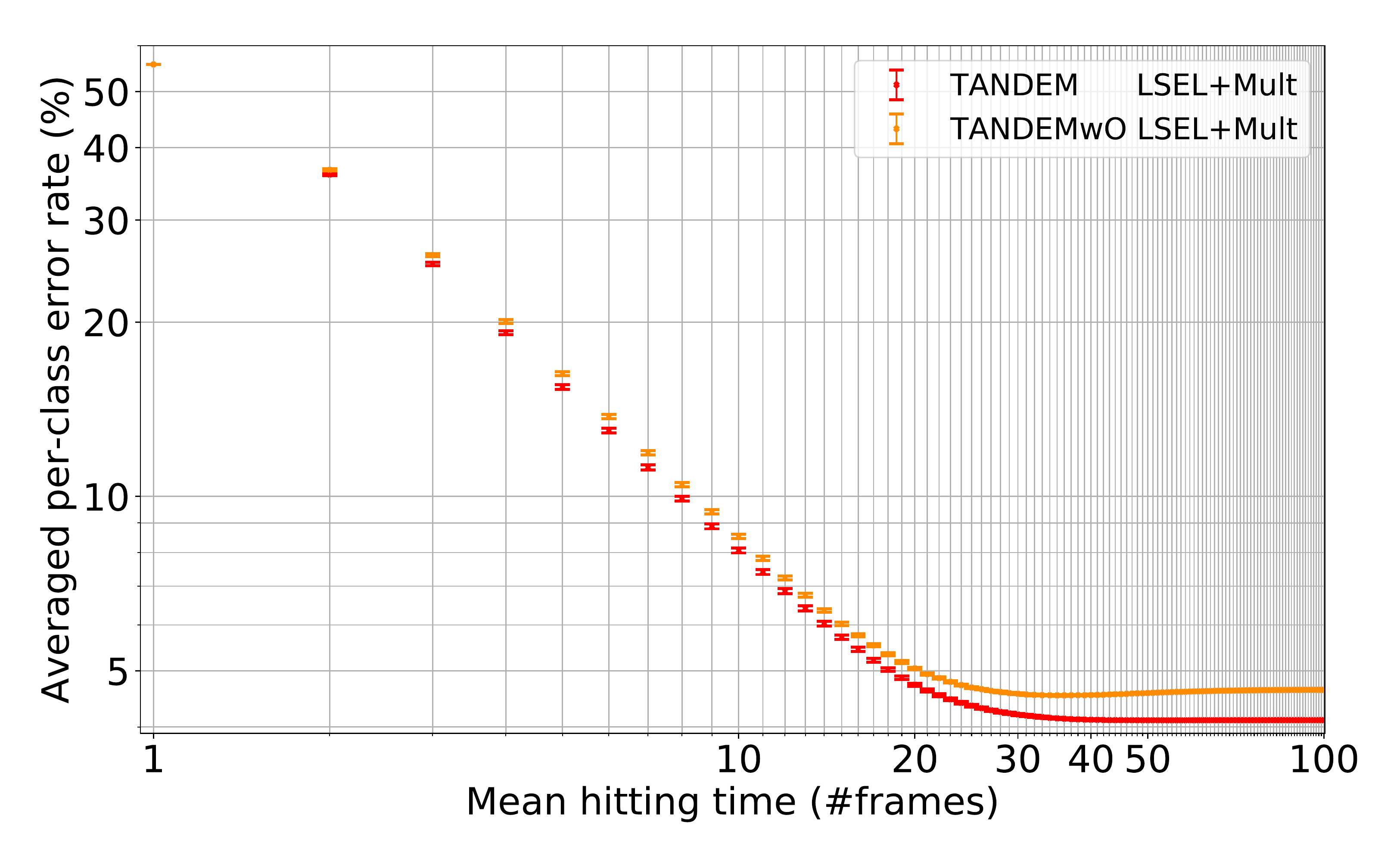} 
    \end{minipage}
    \caption{\textbf{M-TANDEM vs. M-TANDEMwO. Top: NMNIST. Middle: NMNIST-H. Bottom: NMNIST-100f.} TANDEM and TANDEMwO means that the model is trained with the M-TANDEM and M-TANDEMwO formulae, respectively. Mult means that the multiplet loss is simultaneously used.}
    \label{fig: TANDEM vs TANDEMwO}
\end{figure}

%% file: supp_GuessAversion.tex
\section{Proofs Related to Guess-Aversion} \label{app: Proofs of Guess-Aversion}

\subsection{Proof of Theorem \ref{thm: CLSEL is guess-averse}} \label{app: Proof of Theorem thm: CLSEL is guess-averse}
\begin{proof}
    For any $k,l \in [K]$ ($k \neq l$) and any $\bm{s} \in \mc{S}_k$, $e^{- (s_k - s_l)}$ is less than $1$ by definition of $\mc{S}_k$. Therefore, for any $k \in [K]$, any $\bm{s} \in \mc{S}_k$, any $\bm{s}^\prime \in \mc{A}$, and any cost matrix $C$,
    \begin{align}
        \ell (\bm{s}, k; C) 
            := C_{k} \log (1 + \sum_{l (\neq k)} e^{- (s_k - s_l)})
        < C_k \log (1 + \sum_{l (\neq k)} 1) 
            = \ell(\bm{s}^\prime, k; C) \, . \no
    \end{align}
\end{proof}

\subsection{NGA-LSEL Is Not Guess-Averse} \label{app: NGA-LSEL Is Not Guess-Averse}
The NGA-LSEL is $\ell (\bm{s}, y; C) = \sum_{k (\neq y)} C_{y,l} \log (1 + \sum_{l (\neq k)} e^{s_l - s_k})$ (Section \ref{app: Cost-Sensive LSEL and Guess-Aversion}). We prove that the NGA-LSEL is not guess-averse by providing a counter example.
\begin{proof}
    Assume that $K=3$, $C_{kl} = 1$ ($k \neq l$), $\bm{s}(X_i^{(1,t)}) = (3, 2, -100)^\top$, and $y_i=1$. Then,
    \begin{align}
        &\ell (\bm{s}(X_i^{(1,t)}), y_i; C) 
        = \log (1 + e^{ s_1 - s_2 } + e^{ s_3 - s_2 } ) + \log( 1 + e^{ s_1 - s_3 } + e^{ s_2 - s_3 } ) \nn
        =& \log (1 + e^{ 1 } + e^{ -102 } ) + \log( 1 + e^{ 103 } + e^{ 102 } ) \nn
        > & \log (3) + \log (3) =  \ell (\bm{0}, y_i; C) \, .\no
    \end{align}
\end{proof}

\subsection{Another cost-sensitive LSEL}
Alternatively to $\hat{L}_{\rm CLSEL}$, we can define
\begin{align}
    \hat{L}_{\rm LSCEL} (\bm{\th}, C ; S)
    :=\f{1}{MT} \sum_{i=1}^M \sum_{t=1}^T \log ( 
            1 + \sum_{l (\neq y_i)} C_{y_i l} e^{ - \hat{\la}_{y_i l} (X^{(1,t)}_i; \bm{\th}) }
        ) \label{eq: LSCEL} \, .
\end{align}
$\hat{L}_{\rm LSCEL}$ reduces to $\hat{L}_{\rm \text{modLSEL}}$ when $C_{kl} = \hat{\nu}_{kl}^{-1}$. The following theorem shows that $\hat{L}_{\rm \text{LSCEL}}$ is guess-averse:
\begin{theorem} \label{thm: LSCEL is guess-averse}
    $\hat{L}_{\operatorname{LSCEL}}$ is guess-averse, provided that the log-likelihood vector 
    \begin{align}
        \Big( \log \hat{p}_{\bm{\th}} (X^{(1,t)} | y=1), \log \hat{p}_{\bm{\th}} (X^{(1,t)} | y=2),
        \, ..., \, \log \hat{p}_{\bm{\th}} (X^{(1,t)} | y=K) ) \Big)^\top
            \in \mbr^{K} \no
    \end{align}
    is regarded as the score vector $\bm{s}(X^{(1,t)})$.
\end{theorem}
\begin{proof}
    We use Lemma 1 in \cite{beijbom2014guess-averse}:
    \begin{lemma}[Lemma 1 in \cite{beijbom2014guess-averse}] \label{lem: LSCEL is guess-averse}
        Let $\ell (\bm{s}, y; C) = \ga ( \sum_{k \in [K]} C_{yk} \ph ( s_y - s_k ) )$, where $\ga: \mbr \rightarrow \mbr$ is a monotonically increasing function and $\ph: \mbr \rightarrow \mbr$ is a function such that for any $v > 0$, $\ph(v) < \ph (0)$. Then, $\ell$ is guess-averse.
    \end{lemma}
    The statement of Theorem \ref{thm: LSCEL is guess-averse} immediately follows by substituting $\ga (v) = \log (1 + v)$ and $\ph(v) = e^{-v}$ into Lemma \ref{lem: LSCEL is guess-averse}.
\end{proof}

\subsection{Cost-Sensitive Logistic Losses Are Guess-Averse}
We additionally prove that the cost-sensitive logistic losses defined below are also guess-averse, which may be of independent interest.
We define a cost-sensitive logistic loss as
\begin{align}
    &\hat{L}_{\operatorname{C-logistic}} (\bm{\th}, C; S) := 
    \f{1}{MT} \sum_{i=1}^M \sum_{t=1}^T \f{1}{K-1} \sum_{l (\neq y_i)} C_{y_i l} \log \lt(
            1 + e^{ - \hat{\la}_{y_i l} (X^{(1,t)}_i; \bm{\th}) }
        \rt) \label{eq: C-logistic}
\end{align}
$\hat{L}_{\operatorname{C-logistic}}$ reduces to $\hat{L}_{\rm logistic}$ (defined in Appendix \ref{app: Logistic Loss and Its Consistency}) if $C_{kl} = C_{k} = M / K M_k$. $\hat{L}_{\operatorname{C-logistic}}$ is guess-averse:
\begin{theorem} \label{thm: cost-sensitive LLLR-D ver 2 is guess-averse}
    $\hat{L}_{\rm \text{C-logistic}}$ is guess-averse, provided that the log-likelihood vector 
    \begin{align}
        &\Big( \log \hat{p}(X^{(1,t)} | y=1), \log \hat{p}(X^{(1,t)} | y=2), 
        \, ..., \, \log \hat{p}(X^{(1,t)} | y=K) ) \Big)^\top
            \in \mbr^{K}
    \end{align}
    is regarded as the score vector $\bm{s}(X^{(1,t)})$.
\end{theorem}
\begin{proof}
    The proof is parallel to that of Theorem \ref{thm: CLSEL is guess-averse}. For any $k, l \in [K]$ ($k \neq l$) and any $\bm{s} \in \mc{S}_k$, $e^{- (s_k - s_l)}$ is less than $1$ by definition of $\mc{S}_k$. Therefore, for any $k,l \in [K]$, any $\bm{s} \in \mc{S}_k$, any $\bm{s}^\prime \in \mc{A}$, and any cost matrix $C$,
    \begin{align}
        \ell (\bm{s}, k; C) := \f{1}{K-1} \sum_{l (\neq k)} \log ( 1 + e^{- (s_k - s_l)} )^{C_{kl}}
        < \f{1}{K-1} \sum_{l (\neq k)} \log (1 + 1)^{C_{kl}} = \ell (\bm{s}^\prime, k; C) \, . \no
    \end{align}
\end{proof}

We also define
\begin{align}
    \hat{L}_{\operatorname{logistic-C}} (\bm{\th}, C; S) :=
    \f{1}{MT} \sum_{i=1}^M \sum_{t=1}^T \f{1}{K-1} \sum_{l (\neq y_i)} \log \lt( 
            1 + C_{y_i l} e^{ - \hat{\la}_{y_i l} (X^{(1,t)}_i; \bm{\th}) }
        \rt) \, . \label{eq: logistic-C}
\end{align}
$\hat{L}_{\operatorname{logistic-C}}$ reduces to $\hat{L}_{\rm modlogistic}$ if $C_{k l} = \hat{\nu}_{kl}^{-1}$. $\hat{L}_{\operatorname{logistic-C}}$ is guess-averse:
\begin{theorem} \label{thm: cost-sensitive LLLR-D ver 1 is guess-averse}
    $\hat{L}_{\rm \text{logistic-C}}$ is guess-averse, provided that the log-likelihood vector 
    \begin{align}
        \Big( \log \hat{p}(X^{(1,t)} | y=1), \log \hat{p}(X^{(1,t)} | y=2), 
        \, ..., \, \log \hat{p}(X^{(1,t)} | y=K) ) \Big)\top
            \in \mbr^{K}
    \end{align}
    is regarded as the score vector $\bm{s}(X^{(1,t)})$.
\end{theorem}
To prove Theorem \ref{thm: cost-sensitive LLLR-D ver 1 is guess-averse}, we first show the following lemma:
\begin{lemma} \label{lem: cost-sensitive LLLR-D ver 1 is guess-averse}
    Let 
    \begin{align}
        \ell (\bm{s}, k, ; C) = \ga \Big( \prod_{l \in [K]} \lt( 1 + C_{k l} \ph ( s_k - s_l ) \rt) \Big) \, , \no
    \end{align}
    where $\ga: \mbr \rightarrow \mbr$ is a monotonically increasing function and $\ph: \mbr \rightarrow \mbr$ is a function such that for any $v > 0$, $\ph(v) < \ph(0)$. Then, $\ell$ is guess-averse.
\end{lemma}
\begin{proof}
    For any $\bm{s} \in \mc{S}_k$ and $l \in [K]$, 
    \begin{equation}
        \ph(s_k - s_l) < \ph(0) \, , \no
    \end{equation}
    because $\ph(v) < \ph(0)$ and $s_k > s_l$ for all $v > 0$ and $l \in [K]$ ($l \neq k$). Therefore,
    \begin{equation}
        \prod_{l \in [K]} ( 1 + C_{k l} \ph (s_k - s_l) ) 
            < \prod_{l \in [K]} ( 1 + C_{k l} \ph(0) ) \, , \no
    \end{equation}
    because $C_{k l} \geq 0$ for all $k, l \in [K]$ and $C_{k l} \neq 0$ for at least one $l (\neq k)$. Hence, for any $k \in [K]$, any $\bm{s} \in \mc{S}_k$, any $\bm{s}^\prime \in \mc{A}$, and any cost matrix $C$, the monotonicity of $\ga$ shows that 
    \begin{align}
        \ell ( \bm{s}, k; C ) 
            = \ga \lt(\prod_{k \in [K]} ( 1 + C_{y k} \ph (s_y - s_k) ) \rt)
        < \ga \lt(\prod_{k \in [K]} ( 1 + C_{y k} \ph(0) ) \rt)
            = \ell (\bm{s}^\prime, k; C) \, . \no
    \end{align}
\end{proof}
\begin{proof}[Proof of Theorem \ref{thm: cost-sensitive LLLR-D ver 1 is guess-averse}]
    The statement immediately follows from Lemma \ref{lem: cost-sensitive LLLR-D ver 1 is guess-averse} by substituting $\ga(v) = \log(v)$ and $\ph(v) = e^{-v}$.
\end{proof}

%% file: supp_AblationMultLSEL.tex
\section{Ablation Study of Multiplet Loss and LSEL} \label{app: Ablation Study: Multiplet Loss and LSEL}

Figure \ref{fig: Ablation study Mult vs LSEL vs Combo} shows the ablation study comparing the LSEL and the multiplet loss. The combination of the LSEL and the multiplet loss is statistically significantly better than either of the two losses (Appendix \ref{app: Statistical Tests}).
The multiplet loss also performs better than the LSEL. However, the independent use of the multiplet loss has drawbacks: The multiplet loss optimizes \textit{all} the posterior densities output from the temporal integrator (magenta circles in Figure \ref{fig: MSPRT-TANDEM}), while the LSEL uses the \textit{minimum} posterior densities required to calculate the LLR matrix via the M-TANDEM or M-TANDEMwO formula. Therefore, the multiplet loss can lead to a suboptimal minimum for estimating the LLR matrix. In addition, the multiplet loss tends to suffer from the overconfidence problem \citeApp{guo2017ICML_on_calibration}, causing extreme values of LLRs.


\begin{figure}[ht]
    \begin{center}
    \centerline{\includegraphics[width=200pt]
    {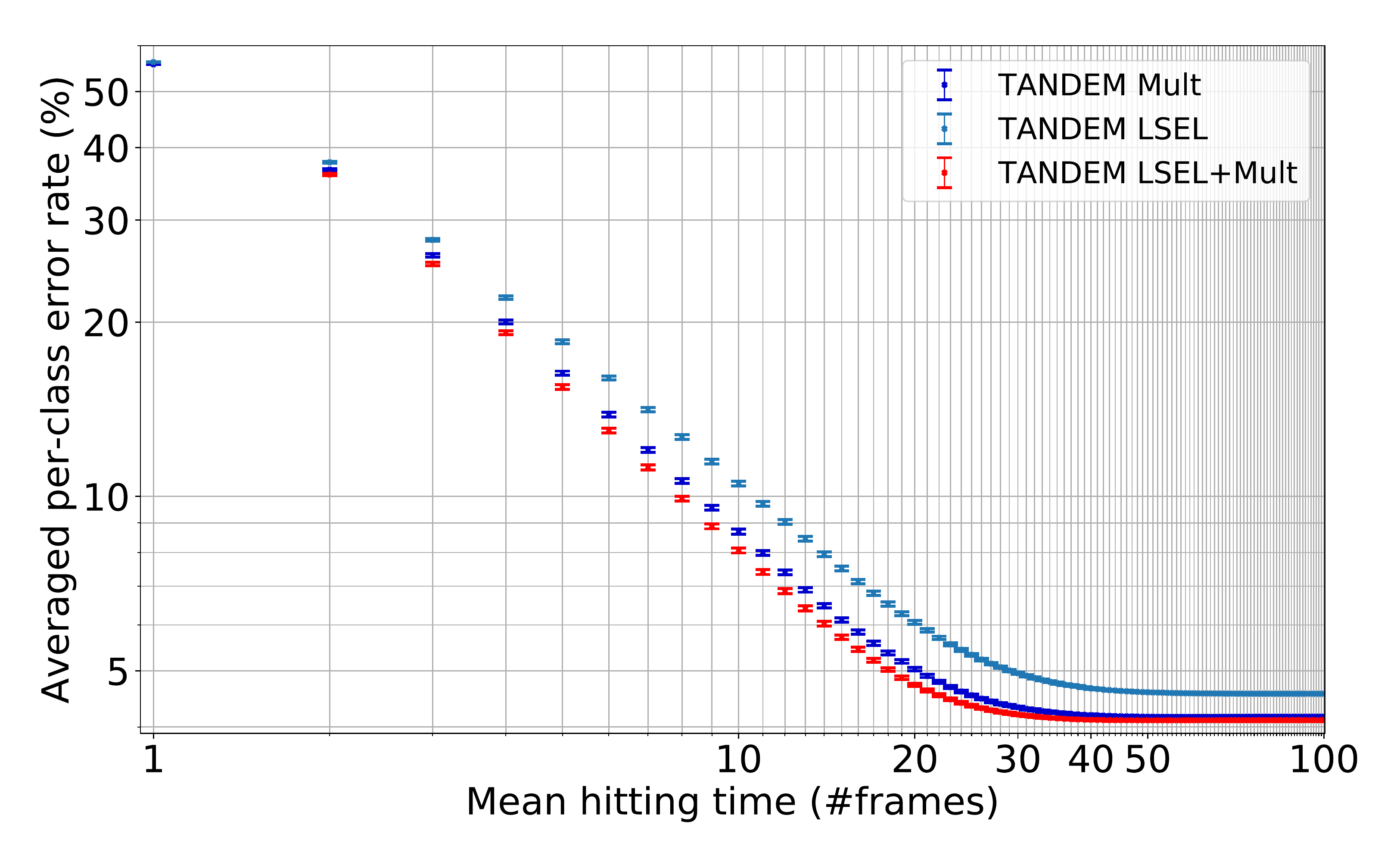}}
    \caption{
        \textbf{Ablation Study of LSEL and Multiplet Loss on NMNIST-100f.} The error bars show SEM. The combination of the LSEL with the multiplet loss 
        gives the best result. The other two curves represent the models trained with the LSEL only and the multiplet loss only. Details of the experiment and the statistical tests are given in Appendices \ref{app: Details of Experiment and More Results} and \ref{app: Statistical Tests}.
        } \label{fig: Ablation study Mult vs LSEL vs Combo}
    \end{center}
\end{figure}

%% file: supp_DetailExperiment.tex
\section{Details of Experiment and More Results}\label{app: Details of Experiment and More Results}
Our computational infrastructure is DGX-1.
The fundamental libraries used in the experiment are Numpy \citeApp{numpy}, Scipy \citeApp{scipy}, Tensorflow 2.0.0 \citeApp{tensorflow} and PyTorch 1.2 \citeApp{pytorch}.

The train/test splitting of NMNIST-H and NMNIST-100f follows the standard one of MNIST \citeApp{MNIST}. The validation set is separated from the last 10,000 examples in the training set. 
The train/test splitting of UCF101 and HMDB51 follows the official splitting \#1. The validation set is separated from the training set, keeping the class frequency. All the videos in UCF101 and HMDB51 are clipped or repeated to make their temporal length equal (50 and 79, respectively). See also our official code.
All the pixel values are divided by 127.5 and then subtracted by 1 before training the feature extractor. 

Hyperparameter tuning is performed with the TPE algorithm \citeApp{TPE}, the default algorithm of Optuna \citeApp{Optuna}. For optimizers, we use Adam, Momentum, \citeApp{loshchilov2018AdamW_SGDW} or RMSprop \citeApp{RMSprop}. Note that Adam and Momentum are not the originals (\citeApp{Adam} and \citeApp{rumelhart1986Momentum_original}), but AdamW and SGDW \citeApp{loshchilov2018AdamW_SGDW}, which have a decoupled weight decay from the learning rate.


To obtain arbitrary points of the SAT curve, we compute the thresholds of MSPRT-TANDEM as follows. 
First, we compute all the LLR trajectories of the test examples. Second, we compute the maximum and minimum values of $|\hat{\la} (X_i^{(1,t)})|$ with respect to $i \in [M]$ and $t \in [T]$.
Third, we generate the thresholds between the maximum and minimum. The thresholds are linearly uniformly separated.
Forth, we run the MSPRT and obtain a two-dimensional point for each threshold ($x = $ mean hitting time, $y = $ averaged per-class error rate).
Finally, we plot them on the speed-accuracy plane and linearly interpolate between two points with neighboring mean hitting times. 
If all the frames in a sequence are observed, the threshold of MSPRT-TANDEM is immediately collapsed to zero to force a decision.

\subsection{Common Feature Extractor} \label{app: Common Feature Extractor}
\paragraph{NMNIST-H and NMNIST-100f}
We first train the feature extractor to extract the bottleneck features, which are then used to train the temporal integrator, LSTM-s/m, and EARLIEST. Therefore, all the models in Figure \ref{fig: SATCs} (MSPRT-TANDEM, NP test, LSTM-s/m, and EARLIEST) share exactly the same feature extractor, ResNet-110 \citeApp{ResNetV1, ResNetV2} with the bottleneck feature dimensions set to 128. The total number of trainable parameters is 6,904,608. 

Tables \ref{tab: Hparams of feature extractor on NMNIST-H} and \ref{tab: Hparams of feature extractor on NMNIST-100f} show the search spaces of hyperparameters. The batch sizes are 64 and 50 for NMNIST-H and NMNIST-100f, respectively. The numbers of training iterations are 6,200 and 40,000 for NMNIST-H and NMNIST-100f, respectively. For each \textit{tuning trial}, we train ResNet and evaluate its averaged per-class accuracy on the validation set per 200 training steps, and after all training iterations, we save the best averaged per-class accuracy in that tuning trial. After all tuning trials, we choose the best hyperparameter combination, which is shown in cyan letters in Tables \ref{tab: Hparams of feature extractor on NMNIST-H} and \ref{tab: Hparams of feature extractor on NMNIST-100f}. We train ResNet with those hyperparameters to extract 128-dimensional bottleneck features, which are then used for the temporal integrator, LSTM-s/m, and EARLIEST. The averaged per-class accuracies of the feature extractors trained on NMNIST-H and on NMNIST-100f are $\sim 84 \%$ and $43 \%$, respectively. The approximated runtime of a single tuning trial is 4.5 and 35 hours for NMNIST-H and NMNIST-100f, respectively, and the GPU consumption is 31 GBs for both datasets.

\paragraph{UCF101 and HMDB51}
We use a pre-trained model without fine-tuning (Microsoft Vision ResNet-50 version 1.0.5 \citeApp{Pretrained_ResNet50_microsoftvision}). The final feature dimensions are 2048.

\begin{table}[htb]
    \caption{
        \textbf{Hyperparameter Search Space of Feature Extractor in Figure \ref{fig: SATCs}: NMNIST-H.} The best hyperparameter combination is highlighted in cyan.
    } \label{tab: Hparams of feature extractor on NMNIST-H}
    \begin{center}
    \begin{small}
    \begin{sc}
    \begin{tabular}{ll}
        \toprule
        Learning rate & \{ $10^{-2}$, 5*$10^{-3}$, $10^{-3}$, \textcolor{cyan}{5*$10^{-4}$} , $10^{-4}$ \} \\
        Weight decay & \{ $10^{-3}$, $10^{-4}$, \textcolor{cyan}{$10^{-5}$} \} \\
        Optimizer & \{ \textcolor{cyan}{Adam}, Momentum, RMSprop \} \\
        \midrule
        \# tuning trials & 96 \\
        \bottomrule
    \end{tabular}
    \end{sc}
    \end{small}
    \end{center}
\end{table}

\begin{table}[htb]
    \caption{
        \textbf{Hyperparameter Search Space of Feature Extractor in Figure \ref{fig: SATCs}: NMNIST-100f.} The best hyperparameter combination is highlighted in cyan.
    } \label{tab: Hparams of feature extractor on NMNIST-100f}
    \begin{center}
    \begin{small}
    \begin{sc}
    \begin{tabular}{ll}
        \toprule
        Learning rate & \{ $10^{-2}$, \textcolor{cyan}{5*$10^{-3}$}, $10^{-3}$, 5*$10^{-4}$, $10^{-4}$, $10^{-5}$ \} \\
        Weight decay & \{ $10^{-3}$, $10^{-4}$, \textcolor{cyan}{$10^{-5}$} \} \\
        Optimizer & \{ \textcolor{cyan}{Adam}, Momentum, RMSprop \} \\
        \midrule
        \# tuning trials & 7 \\
        \bottomrule
    \end{tabular}
    \end{sc}
    \end{small}
    \end{center}
\end{table}


\subsection{Figure \ref{fig: SATCs}: NMNIST-H} \label{app: Figure fig: SATCs: NMNIST-H}
\paragraph{MSPRT-TANDEM and NP test}
The approximation formula is the M-TANDEMwO formula. The loss function consists of the LSEL and the multiplet loss.
The temporal integrator is Peephole LSTM \citeApp{PeepholeLSTM} with the hidden state dimensions set to 128 followed by a fully-connected layer to output logits for classification. The temporal integrator has 133,760 trainable parameters.
The batch size is fixed to 500. The number of training iterations is 5,000.

Table \ref{tab: Hparams of MSPRT on NMNIST-H in fig SATC} shows the search space of hyperparameters. For each tuning trial, we train the temporal integrator and evaluate its mean averaged per-class accuracy\footnote{
    1. Compute LLRs for all frames; 2. Run MSPRT with threshold = 0 and compute framewise averaged per-class accuracy; 3. Compute the arithmetic mean of the framewise averaged per-class accuracy.
} per every 50 training iterations. After all training iterations, we save the best mean averaged per-class accuracy. After all tuning trials, we select the best combination of the hyperparameters, which is shown in Table \ref{tab: Hparams of MSPRT on NMNIST-H in fig SATC} in cyan letters.

After fixing the hyperparameters, we then train LSTM arbitrary times. During each statistics trial, we train LSTM with the best fixed hyperparameters and evaluate its mean averaged per-class accuracy at every 50 training iterations. After all training iterations, we save the best weight parameters in terms of the mean averaged per-class accuracy. After all statistics trials, we can plot the SAT ``points'' with integer hitting times. The approximated runtime of one statistic trial is 3.5 hours, and the GPU consumption is 1.0 GB.

\begin{table}[htb]
    \caption{
        \textbf{Hyperparameter Search Space of MSPRT-TANDEM in Figure \ref{fig: SATCs}: NMNIST-H.} The best hyperparameter combination is highlighted in cyan. $\ga$ is defined as $L_{\mr{total}} = L_{\mathrm{mult}} + \ga L_{\mathrm{LSEL}}$.
    } \label{tab: Hparams of MSPRT on NMNIST-H in fig SATC}
    \begin{center}
    \begin{small}
    \begin{sc}
    \begin{tabular}{ll}
        \toprule
        Order  & \{0,1,5,10, \textcolor{cyan}{$15$},19 \} \\
        Learning rate & \{ $10^{-2}$, $10^{-3}$, \textcolor{cyan}{$10^{-4}$} \} \\
        Weight decay & \{ $10^{-3}$, \textcolor{cyan}{$10^{-4}$}, $10^{-5}$ \} \\
        $\ga$ & \{ $10^{-1}$, $1$, \textcolor{cyan}{$10$}, $10^{2}$, $10^{3}$\} \\
        Optimizer & \{ Adam, \textcolor{cyan}{RMSprop} \} \\
        \midrule
        \# tuning trials & 500 \\
        \bottomrule
    \end{tabular}
    \end{sc}
    \end{small}
    \end{center}
\end{table}

\paragraph{LSTM-s/m}
The backbone model is Peephole LSTM with the hidden state dimensions set to 128 followed by a fully-connected layer to output logits for classification. LSTM has 133,760 trainable parameters. The batch size is fixed to 500. The number of training iterations is 5,000.

Tables \ref{tab: Hparams of LSTM-s on NMNIST-H in fig SATC} and \ref{tab: Hparams of LSTM-m on NMNIST-H in fig SATC} show the search spaces of hyperparameters. For each tuning trial, we train LSTM and evaluate its mean averaged per-class accuracy per every 50 training iterations. After all training iterations, we save the best mean averaged per-class accuracy. After all tuning trials, we select the best combination of the hyperparameters, which is shown in Tables \ref{tab: Hparams of LSTM-s on NMNIST-H in fig SATC} and \ref{tab: Hparams of LSTM-m on NMNIST-H in fig SATC} in cyan letters.

After fixing the hyperparameters, we then train LSTM arbitrary times. During each statistics trial, we train LSTM with the best fixed hyperparameters and evaluate its mean averaged per-class accuracy per every 50 training iterations. After all training iterations, we save the best weight parameters in terms of the mean averaged per-class accuracy. After all statistics trials, we can plot the SAT ``points'' with integer hitting times. The approximated runtime of one statistics trial is 3.5 hours, and the GPU consumption is 1.0 GB.

\begin{table}[htb]
    \caption{
        \textbf{Hyperparameter Search Space of LSTM-s in Figure \ref{fig: SATCs}: NMNIST-H.} The best hyperparameter combination is highlighted in cyan. $\ga$ controls the strength of monotonicity and is defined as $L_{\mr{total}} = L_{\mathrm{cross\mbox{-}entropy}} + \ga L_{\mathrm{ranking}}$ \protect\citeApp{LSTM_ms}.
    } \label{tab: Hparams of LSTM-s on NMNIST-H in fig SATC}
    \begin{center}
    \begin{small}
    \begin{sc}
    \begin{tabular}{ll}
        \toprule
        Learning rate & \{ $10^{-2}$, $10^{-3}$, \textcolor{cyan}{$10^{-4}$}, $10^{-5}$ \} \\
        Weight decay & \{ $10^{-3}$, \textcolor{cyan}{$10^{-4}$}, $10^{-5}$ \} \\
        $\ga$ & \{ \textcolor{cyan}{$10^{-2}$}, $10^{-1}$, $1$, $10$, $10^{2}$ \} \\
        Optimizer & \{ Adam, \textcolor{cyan}{RMSprop} \} \\
        \midrule
        \# tuning trials & 500 \\
        \bottomrule
    \end{tabular}
    \end{sc}
    \end{small}
    \end{center}
\end{table}

\begin{table}[htb]
    \caption{
        \textbf{Hyperparameter Search Space of LSTM-m in Figure \ref{fig: SATCs}: NMNIST-H.} The best hyperparameter combination is highlighted in cyan. $\ga$ controls the strength of monotonicity and is defined as $L_{\mr{total}} = L_{\mathrm{cross\mbox{-}entropy}} + \ga L_{\mathrm{ranking}}$ \protect\citeApp{LSTM_ms}.
    } \label{tab: Hparams of LSTM-m on NMNIST-H in fig SATC}
    \begin{center}
    \begin{small}
    \begin{sc}
    \begin{tabular}{ll}
        \toprule
        Learning rate & \{ $10^{-2}$, $10^{-3}$, \textcolor{cyan}{$10^{-4}$}, $10^{-5}$ \} \\
        Weight decay & \{ $10^{-3}$, $10^{-4}$, \textcolor{cyan}{$10^{-5}$} \} \\
        $\ga$ & \{ \textcolor{cyan}{$10^{-2}$}, $10^{-1}$, $1$, $10$, $10^{2}$ \} \\
        Optimizer & \{ Adam, \textcolor{cyan}{RMSprop} \} \\
        \midrule
        \# tuning trials & 500 \\
        \bottomrule
    \end{tabular}
    \end{sc}
    \end{small}
    \end{center}
\end{table}

\paragraph{EARLIEST}
The main backbone is LSTM \citeApp{LSTM} with the hidden state dimensions set to 128. The whole architecture has 133,646 trainable parameters. The batch size is 1000. The number of training iterations is 20,000. EARLIEST has a parameter $\la$ (Not to be confused with the LLR matrix) that controls the speed-accuracy tradeoff. A larger $\la$ gives faster and less accurate decisions, and a smaller $\la$ gives slower and more accurate decisions.
We train EARLIEST with two different $\la$'s: $10^{-2}$ and $10^2$.

Tables \ref{tab: Hparams of EARLIEST lam-2 on NMNIST-H in fig SATC} and \ref{tab: Hparams of EARLIEST lam2 on NMNIST-H in fig SATC} show the search spaces of hyperparameters. For each tuning trial, we train EARLIEST and evaluate its averaged per-class accuracy per every 500 training iterations. After all training iterations, we save the best averaged per-class accuracy. After all tuning trials, we select the best combination of the hyperparameters, which is shown in Tables \ref{tab: Hparams of EARLIEST lam-2 on NMNIST-H in fig SATC} and \ref{tab: Hparams of EARLIEST lam2 on NMNIST-H in fig SATC} in cyan letters.

After fixing the hyperparameters, we then train EARLIEST arbitrary times. During each statistics trial, we train EARLIEST with the best fixed hyperparameters and evaluate its mean averaged per-class accuracy per every 500 training iterations. After all training iterations, we save the best weight parameters in terms of the mean averaged per-class accuracy. After all statistics trials, we can plot the SAT ``points.'' Note that EARLIEST cannot change the decision policy after training, and thus one statistics trial gives only one point on the SAT graph; therefore, several statistics trials give only one point with an error bar. The approximated runtime of one statistics trial is 12 hours, and the GPU consumption is 1.4 GBs. 

\begin{table}[htb]
    \caption{
        \textbf{Hyperparameter Search Space of EARLIEST with $\bm{\la=10^{-2}}$ in Figure \ref{fig: SATCs}: NMNIST-H.} The best hyperparameter combination is highlighted in cyan. 
    } \label{tab: Hparams of EARLIEST lam-2 on NMNIST-H in fig SATC}
    \begin{center}
    \begin{small}
    \begin{sc}
    \begin{tabular}{ll}
        \toprule
        Learning rate & \{ $10^{-1}$, $10^{-2}$, \textcolor{cyan}{$10^{-3}$}, $10^{-4}$, $10^{-5}$ \} \\
        Weight decay & \{ $10^{-3}$, \textcolor{cyan}{$10^{-4}$}, $10^{-5}$ \} \\
        Optimizer & \{ Adam, \textcolor{cyan}{RMSprop} \} \\
        \midrule
        \# tuning trials & 500 \\
        \bottomrule
    \end{tabular}
    \end{sc}
    \end{small}
    \end{center}
\end{table}

\begin{table}[htb]
    \caption{
        \textbf{Hyperparameter Search Space of EARLIEST with $\bm{\la=10^2}$ in Figure \ref{fig: SATCs}: NMNIST-H.} The best hyperparameter combination is highlighted in cyan.
    } \label{tab: Hparams of EARLIEST lam2 on NMNIST-H in fig SATC}
    \begin{center}
    \begin{small}
    \begin{sc}
    \begin{tabular}{ll}
        \toprule
        Learning rate & \{ $10^{-1}$, $10^{-2}$, \textcolor{cyan}{$10^{-3}$}, $10^{-4}$, $10^{-5}$ \} \\
        Weight decay & \{ \textcolor{cyan}{$10^{-3}$}, $10^{-4}$, $10^{-5}$ \} \\
        Optimizer & \{ Adam, \textcolor{cyan}{RMSprop} \} \\
        \midrule
        \# tuning trials & 500 \\
        \bottomrule
    \end{tabular}
    \end{sc}
    \end{small}
    \end{center}
\end{table}

\subsection{Figure \ref{fig: SATCs}: NMNIST-100f} \label{app: Figure fig: SATCs NMNIST-100f}
\paragraph{MSPRT-TANDEM and NP test}
The approximation formula is the M-TANDEM formula. The loss function consists of the LSEL and the multiplet loss. The temporal integrator is Peephole LSTM \citeApp{PeepholeLSTM} with the hidden state dimensions set to 128 followed by a fully-connected layer to output logits for classification. The temporal integrator has 133,760 trainable parameters. The batch size is fixed to 100. The number of training iterations is 5,000.

Table \ref{tab: Hparams of MSPRT on NMNIST-100f in fig SATC} shows the search space of hyperparameters. For each tuning trial, we train the temporal integrator and evaluate its mean averaged per-class accuracy per every 200 training iterations. After all training iterations, we save the best mean averaged per-class accuracy. After all tuning trials, we select the best combination of the hyperparameters, which is shown in Table \ref{tab: Hparams of MSPRT on NMNIST-100f in fig SATC} in cyan letters. The approximated runtime of one statistics trial is 1 hour, and the GPU consumption is 8.7 GBs. 

\begin{table}[htb]
    \caption{
        \textbf{Hyperparameter Search Space of MSPRT-TANDEM in Figure \ref{fig: SATCs}: NMNIST-100f.} The best hyperparameter combination is highlighted in cyan. $\ga$ is defined as $L_{\mr{total}} = L_{\mathrm{mult}} + \ga L_{\mathrm{LSEL}}$.
    } \label{tab: Hparams of MSPRT on NMNIST-100f in fig SATC}
    \begin{center}
    \begin{small}
    \begin{sc}
    \begin{tabular}{ll}
        \toprule
        Order  & \{ 0, \textcolor{cyan}{25}, 50, 75, 99 \} \\
        Learning rate & \{ \textcolor{cyan}{$10^{-2}$}, $10^{-3}$, $10^{-4}$ \} \\
        Weight decay & \{ $10^{-3}$, $10^{-4}$, \textcolor{cyan}{$10^{-5}$} \} \\
        $\ga$ & \{ \textcolor{cyan}{$10^{-1}$}, $1$, $10$, $10^{2}$, $10^{3}$\} \\
        Optimizer & \{ \textcolor{cyan}{Adam}, RMSprop \} \\
        \midrule
        \# tuning trials & 200 \\
        \bottomrule
    \end{tabular}
    \end{sc}
    \end{small}
    \end{center}
\end{table}

\paragraph{LSTM-s/m}
The backbone model is Peephole LSTM with the hidden state dimensions set to 128 followed by a fully-connected layer to output logits for classification. LSTM has 133,760 trainable parameters. The batch size is fixed to 500. The number of training iterations is 5,000.

Tables \ref{tab: Hparams of LSTM-s on NMNIST-100f in fig SATC} and \ref{tab: Hparams of LSTM-m on NMNIST-100f in fig SATC} show the search spaces of hyperparameters. For each tuning trial, we train LSTM and evaluate its mean averaged per-class accuracy per every 100 training iterations. After all training iterations, we save the best mean averaged per-class accuracy. After all tuning trials, we select the best combination of the hyperparameters, which is shown in Tables \ref{tab: Hparams of LSTM-s on NMNIST-100f in fig SATC} and \ref{tab: Hparams of LSTM-m on NMNIST-100f in fig SATC} in cyan letters.

The approximated runtime of one statistics trial is 5 hours, and the GPU consumption is 2.6 GBs. 

\begin{table}[htb]
    \caption{
        \textbf{Hyperparameter Search Space of LSTM-s in Figure \ref{fig: SATCs}: NMNIST-100f.} The best hyperparameter combination is highlighted in cyan. $\ga$ controls the strength of monotonicity and is defined as $L_{\mr{total}} = L_{\mathrm{cross\mbox{-}entropy}} + \ga L_{\mathrm{ranking}}$ \protect\citeApp{LSTM_ms}.
    } \label{tab: Hparams of LSTM-s on NMNIST-100f in fig SATC}
    \begin{center}
    \begin{small}
    \begin{sc}
    \begin{tabular}{ll}
        \toprule
        Learning rate & \{ $10^{-2}$, \textcolor{cyan}{$10^{-3}$}, $10^{-4}$, $10^{-5}$ \} \\
        Weight decay & \{ $10^{-3}$, $10^{-4}$, \textcolor{cyan}{$10^{-5}$} \} \\
        $\ga$ & \{ $10^{-2}$, \textcolor{cyan}{$10^{-1}$}, $1$, $10$, $10^{2}$ \} \\
        Optimizer & \{ \textcolor{cyan}{Adam}, RMSprop \} \\
        \midrule
        \# tuning trials & 200 \\
        \bottomrule
    \end{tabular}
    \end{sc}
    \end{small}
    \end{center}
\end{table}

\begin{table}[htb]
    \caption{
        \textbf{Hyperparameter Search Space of LSTM-m in Figure \ref{fig: SATCs}: NMNIST-100f.} The best hyperparameter combination is highlighted in cyan. $\ga$ controls the strength of monotonicity and is defined as $L_{\mr{total}} = L_{\mathrm{cross\mbox{-}entropy}} + \ga L_{\mathrm{ranking}}$ \protect\citeApp{LSTM_ms}.
    } \label{tab: Hparams of LSTM-m on NMNIST-100f in fig SATC}
    \begin{center}
    \begin{small}
    \begin{sc}
    \begin{tabular}{ll}
        \toprule
        Learning rate & \{ $10^{-2}$, \textcolor{cyan}{$10^{-3}$}, $10^{-4}$, $10^{-5}$ \} \\
        Weight decay & \{ $10^{-3}$, $10^{-4}$, \textcolor{cyan}{$10^{-5}$} \} \\
        $\ga$ & \{ \textcolor{cyan}{$10^{-2}$}, $10^{-1}$, $1$, $10$, $10^{2}$ \} \\
        Optimizer & \{ Adam, \textcolor{cyan}{RMSprop} \} \\
        \midrule
        \# tuning trials & 200 \\
        \bottomrule
    \end{tabular}
    \end{sc}
    \end{small}
    \end{center}
\end{table}

\paragraph{EARLIEST}
The main backbone is LSTM \citeApp{LSTM} with the hidden state dimensions set to 128. The whole architecture has 133,646 trainable parameters. The batch size is 1000. The number of training iterations is 20,000. We train EARLIEST with two different $\la$'s: $10^{-2}$ and $10^{-4}$.

Tables \ref{tab: Hparams of EARLIEST lam-2 on NMNIST-100f in fig SATC} and \ref{tab: Hparams of EARLIEST lam-4 on NMNIST-100f in fig SATC} show the search spaces of hyperparameters. For each tuning trial, we train EARLIEST and evaluate its averaged per-class accuracy per every 500 training iterations. After all training iterations, we save the best averaged per-class accuracy. After all tuning trials, we select the best combination of the hyperparameters, which is shown in Tables \ref{tab: Hparams of EARLIEST lam-2 on NMNIST-100f in fig SATC} and \ref{tab: Hparams of EARLIEST lam-4 on NMNIST-100f in fig SATC} in cyan letters. The approximated runtime of a single tuning trial is 14 hours, and the GPU consumption is 2.0 GBs.

\begin{table}[htb]
    \caption{
        \textbf{Hyperparameter Search Space of EARLIEST with $\la=10^{-2}$ in Figure \ref{fig: SATCs}: NMNIST-100f.} The best hyperparameter combination is highlighted in cyan. 
    } \label{tab: Hparams of EARLIEST lam-2 on NMNIST-100f in fig SATC}
    \begin{center}
    \begin{small}
    \begin{sc}
    \begin{tabular}{ll}
        \toprule
        Learning rate & \{ $10^{-1}$, $10^{-2}$, $10^{-3}$, $10^{-4}$, \textcolor{cyan}{$10^{-5}$} \} \\
        Weight decay & \{ \textcolor{cyan}{$10^{-3}$}, $10^{-4}$, $10^{-5}$ \} \\
        Optimizer & \{ \textcolor{cyan}{Adam}, RMSprop \} \\
        \midrule
        \# tuning trials & 200 \\
        \bottomrule
    \end{tabular}
    \end{sc}
    \end{small}
    \end{center}
\end{table}

\begin{table}[htb]
    \caption{
        \textbf{Hyperparameter Search Space of EARLIEST with $\la=10^{-4}$ in Figure \ref{fig: SATCs}: NMNIST-100f.} The best hyperparameter combination is highlighted in cyan.
    } \label{tab: Hparams of EARLIEST lam-4 on NMNIST-100f in fig SATC}
    \begin{center}
    \begin{small}
    \begin{sc}
    \begin{tabular}{ll}
        \toprule
        Learning rate & \{ $10^{-1}$, $10^{-2}$, \textcolor{cyan}{$10^{-3}$}, $10^{-4}$, $10^{-5}$ \} \\
        Weight decay & \{ $10^{-3}$, \textcolor{cyan}{$10^{-4}$}, $10^{-5}$ \} \\
        Optimizer & \{ Adam, \textcolor{cyan}{RMSprop} \} \\
        \midrule
        \# tuning trials & 200 \\
        \bottomrule
    \end{tabular}
    \end{sc}
    \end{small}
    \end{center}
\end{table}

\subsection{Figure \ref{fig: SATCs}: UCF101}
\paragraph{MSPRT-TANDEM and NP test}
The approximation formula is the M-TANDEM formula. The loss function consists of the LSEL and the multiplet loss. The temporal integrator is Peephole LSTM \citeApp{PeepholeLSTM} with the hidden state dimensions set to 256 followed by a fully-connected layer to output logits for classification. The temporal integrator has 2,387,456 trainable parameters. The batch size is fixed to 256. The number of training iterations is 10,000.
We use the effective number \citeApp{cui2019CVPR_class-balanced_loss_effective_number} as the cost matrix of the LSEL, instead of $1/M_k$, to avoid over-emphasizing the minority class and to simplify the parameter tuning (only one extra parameter $\beta$ is introduced). 

Table \ref{tab: Hparams of MSPRT on UCF101 in fig SATC} shows the search space of hyperparameters. For each tuning trial, we train the temporal integrator and evaluate its mean averaged per-class accuracy per every 200 training iterations. After all training iterations, we save the best mean averaged per-class accuracy. After all tuning trials, we select the best combination of the hyperparameters, which is shown in Table \ref{tab: Hparams of MSPRT on UCF101 in fig SATC} in cyan letters. The approximated runtime of one statistics trial is 8 hours, and the GPU consumption is 16--32 GBs. 

\begin{table}[htb]
    \caption{
        \textbf{Hyperparameter Search Space of MSPRT-TANDEM in Figure \ref{fig: SATCs}: UCF101.} The best hyperparameter combination is highlighted in cyan. $\ga$ is defined as $L_{\mr{total}} = L_{\mathrm{mult}} + \ga L_{\mathrm{LSEL}}$. $\be$ controls the cost matrix \protect\citeApp{cui2019CVPR_class-balanced_loss_effective_number}.
    } \label{tab: Hparams of MSPRT on UCF101 in fig SATC}
    \begin{center}
    \begin{small}
    \begin{sc}
    \begin{tabular}{ll}
        \toprule
        Order  & \{ 0, \textcolor{cyan}{10}, 25, 40, 49 \} \\
        Learning rate & \{ $10^{-3}$, \textcolor{cyan}{$10^{-4}$}, $10^{-5}$ \} \\
        Weight decay & \{ \textcolor{cyan}{$10^{-3}$}, $10^{-4}$, $10^{-5}$ \} \\
        $\ga$ & \{ $10^{-1}$, \textcolor{cyan}{$1$}, $10$, $10^{2}$ \} \\
        Optimizer & \{ \textcolor{cyan}{Adam}, RMSprop \} \\
        $\be$ & \{ 0.99, 0.999, 0.9999, \textcolor{cyan}{0.99999}, 1. \} \\
        \midrule
        \# tuning trials & 100 \\
        \bottomrule
    \end{tabular}
    \end{sc}
    \end{small}
    \end{center}
\end{table}

\paragraph{LSTM-s/m}
The backbone model is Peephole LSTM with the hidden state dimensions set to 256 followed by a fully-connected layer to output logits for classification. LSTM has 2,387,456 trainable parameters. The batch size is fixed to 256. The number of training iterations is 5,000.

Tables \ref{tab: Hparams of LSTM-s on UCF101 in fig SATC} and \ref{tab: Hparams of LSTM-m on UCF101 in fig SATC} show the search spaces of hyperparameters. For each tuning trial, we train LSTM and evaluate its mean averaged per-class accuracy per every 200 training iterations. After all training iterations, we save the best mean averaged per-class accuracy. After all tuning trials, we select the best combination of the hyperparameters, which is shown in Tables \ref{tab: Hparams of LSTM-s on UCF101 in fig SATC} and \ref{tab: Hparams of LSTM-m on UCF101 in fig SATC} in cyan letters.
The approximated runtime of one statistics trial is 3 hours, and the GPU consumption is 2.6 GBs.

\begin{table}[htb]
    \caption{
        \textbf{Hyperparameter Search Space of LSTM-s in Figure \ref{fig: SATCs}: UCF101.} The best hyperparameter combination is highlighted in cyan. $\ga$ controls the strength of monotonicity and is defined as $L_{\mr{total}} = L_{\mathrm{cross\mbox{-}entropy}} + \ga L_{\mathrm{ranking}}$ \protect\citeApp{LSTM_ms}.
    } \label{tab: Hparams of LSTM-s on UCF101 in fig SATC}
    \begin{center}
    \begin{small}
    \begin{sc}
    \begin{tabular}{ll}
        \toprule
        Learning rate & \{ $10^{-2}$, \textcolor{cyan}{$10^{-3}$}, $10^{-4}$, $10^{-5}$ \} \\
        Weight decay & \{ $10^{-3}$, \textcolor{cyan}{$10^{-4}$}, {$10^{-5}$} \} \\
        $\ga$ & \{ $10^{-2}$, {$10^{-1}$}, $1$, \textcolor{cyan}{$10$}, $10^{2}$ \} \\
        Optimizer & \{ \textcolor{cyan}{Adam}, RMSprop \} \\
        \midrule
        \# tuning trials & 100 \\
        \bottomrule
    \end{tabular}
    \end{sc}
    \end{small}
    \end{center}
\end{table}

\begin{table}[htb]
    \caption{
        \textbf{Hyperparameter Search Space of LSTM-m in Figure \ref{fig: SATCs}: UCF101.} The best hyperparameter combination is highlighted in cyan. $\ga$ controls the strength of monotonicity and is defined as $L_{\mr{total}} = L_{\mathrm{cross\mbox{-}entropy}} + \ga L_{\mathrm{ranking}}$ \protect\citeApp{LSTM_ms}.
    } \label{tab: Hparams of LSTM-m on UCF101 in fig SATC}
    \begin{center}
    \begin{small}
    \begin{sc}
    \begin{tabular}{ll}
        \toprule
        Learning rate & \{ $10^{-2}$, \textcolor{cyan}{$10^{-3}$}, $10^{-4}$, $10^{-5}$ \} \\
        Weight decay & \{ $10^{-3}$, \textcolor{cyan}{$10^{-4}$}, {$10^{-5}$} \} \\
        $\ga$ & \{ {$10^{-2}$}, $10^{-1}$, \textcolor{cyan}{$1$}, $10$, $10^{2}$ \} \\
        Optimizer & \{ \textcolor{cyan}{Adam}, {RMSprop} \} \\
        \midrule
        \# tuning trials & 100 \\
        \bottomrule
    \end{tabular}
    \end{sc}
    \end{small}
    \end{center}
\end{table}

\paragraph{EARLIEST}
The main backbone is LSTM \citeApp{LSTM} with the hidden state dimensions set to 256. The whole architecture has 2,387,817 trainable parameters. The batch size is 256. The number of training iterations is 5,000. We train EARLIEST with two different $\la$'s: $10^{-1}$ and $10^{-10}$.

Tables \ref{tab: Hparams of EARLIEST lam-1 on UCF101 in fig SATC} and \ref{tab: Hparams of EARLIEST lam-10 on UCF101 in fig SATC} show the search spaces of hyperparameters. For each tuning trial, we train EARLIEST and evaluate its averaged per-class accuracy per every 500 training iterations. After all training iterations, we save the best averaged per-class accuracy. After all tuning trials, we select the best combination of the hyperparameters, which is shown in Tables \ref{tab: Hparams of EARLIEST lam-1 on UCF101 in fig SATC} and \ref{tab: Hparams of EARLIEST lam-10 on UCF101 in fig SATC} in cyan letters. The approximated runtime of a single tuning trial is 0.5 hours, and the GPU consumption is 2.0 GBs.

\begin{table}[htb]
    \caption{
        \textbf{Hyperparameter Search Space of EARLIEST with $\la=10^{-1}$ in Figure \ref{fig: SATCs}: UCF101.} The best hyperparameter combination is highlighted in cyan. 
    } \label{tab: Hparams of EARLIEST lam-1 on UCF101 in fig SATC}
    \begin{center}
    \begin{small}
    \begin{sc}
    \begin{tabular}{ll}
        \toprule
        Learning rate & \{ $10^{-1}$, $10^{-2}$, $10^{-3}$, \textcolor{cyan}{$10^{-4}$}, {$10^{-5}$} \} \\
        Weight decay & \{ {$10^{-3}$}, \textcolor{cyan}{$10^{-4}$}, $10^{-5}$ \} \\
        Optimizer & \{ {Adam}, \textcolor{cyan}{RMSprop} \} \\
        \midrule
        \# tuning trials & 100 \\
        \bottomrule
    \end{tabular}
    \end{sc}
    \end{small}
    \end{center}
\end{table}

\begin{table}[htb]
    \caption{
        \textbf{Hyperparameter Search Space of EARLIEST with $\la=10^{-10}$ in Figure \ref{fig: SATCs}: UCF101.} The best hyperparameter combination is highlighted in cyan.
    } \label{tab: Hparams of EARLIEST lam-10 on UCF101 in fig SATC}
    \begin{center}
    \begin{small}
    \begin{sc}
    \begin{tabular}{ll}
        \toprule
        Learning rate & \{ $10^{-1}$, $10^{-2}$, \textcolor{cyan}{$10^{-3}$}, $10^{-4}$, $10^{-5}$ \} \\
        Weight decay & \{ $10^{-3}$, \textcolor{cyan}{$10^{-4}$}, $10^{-5}$ \} \\
        Optimizer & \{ Adam, \textcolor{cyan}{RMSprop} \} \\
        \midrule
        \# tuning trials & 100 \\
        \bottomrule
    \end{tabular}
    \end{sc}
    \end{small}
    \end{center}
\end{table}

\subsection{Figure \ref{fig: SATCs}: HMDB51}
\paragraph{MSPRT-TANDEM and NP test}
The approximation formula is the M-TANDEM formula. The loss function consists of the LSEL and the multiplet loss. The temporal integrator is Peephole LSTM \citeApp{PeepholeLSTM} with the hidden state dimensions set to 256 followed by a fully-connected layer to output logits for classification. The temporal integrator has 2,374,656 trainable parameters. The batch size is fixed to 128. The number of training iterations is 10,000.
We use the effective number \citeApp{cui2019CVPR_class-balanced_loss_effective_number} as the cost matrix of the LSEL, instead of $1/M_k$, to avoid over-emphasising the minority class and to simplify the parameter tuning (only one extra parameter $\beta$ is introduced). 

Table \ref{tab: Hparams of MSPRT on HMDB51 in fig SATC} shows the search space of hyperparameters. For each tuning trial, we train the temporal integrator and evaluate its mean averaged per-class accuracy per every 200 training iterations. After all training iterations, we save the best mean averaged per-class accuracy. After all tuning trials, we select the best combination of the hyperparameters, which is shown in Table \ref{tab: Hparams of MSPRT on HMDB51 in fig SATC} in cyan letters. The approximated runtime of one statistics trial is 8 hours, and the GPU consumption is 16--32 GBs. 

\begin{table}[htb]
    \caption{
        \textbf{Hyperparameter Search Space of MSPRT-TANDEM in Figure \ref{fig: SATCs}: HMDB51.} The best hyperparameter combination is highlighted in cyan. $\ga$ is defined as $L_{\mr{total}} = L_{\mathrm{mult}} + \ga L_{\mathrm{LSEL}}$. $\be$ controls the cost matrix \protect\citeApp{cui2019CVPR_class-balanced_loss_effective_number}.
    } \label{tab: Hparams of MSPRT on HMDB51 in fig SATC}
    \begin{center}
    \begin{small}
    \begin{sc}
    \begin{tabular}{ll}
        \toprule
        Order  & \{ \textcolor{cyan}{0}, 10, 40, 60, 78 \} \\
        Learning rate & \{ $10^{-3}$, \textcolor{cyan}{$10^{-4}$}, $10^{-5}$ \} \\
        Weight decay & \{ {$10^{-3}$}, \textcolor{cyan}{$10^{-4}$}, $10^{-5}$ \} \\
        $\ga$ & \{ $10^{-1}$, {$1$}, $10$, \textcolor{cyan}{$10^{2}$} \} \\
        Optimizer & \{ \textcolor{cyan}{Adam}, RMSprop \} \\
        $\be$ & \{ 0.99, 0.999, \textcolor{cyan}{0.9999}, {0.99999}, 1. \} \\
        \midrule
        \# tuning trials & 100 \\
        \bottomrule
    \end{tabular}
    \end{sc}
    \end{small}
    \end{center}
\end{table}

\paragraph{LSTM-s/m}
The backbone model is Peephole LSTM with the hidden state dimensions set to 256 followed by a fully-connected layer to output logits for classification. LSTM has 2,374,656 trainable parameters. The batch size is fixed to 128. The number of training iterations is 10,000.

Tables \ref{tab: Hparams of LSTM-s on HMDB51 in fig SATC} and \ref{tab: Hparams of LSTM-m on HMDB51 in fig SATC} show the search spaces of hyperparameters. For each tuning trial, we train LSTM and evaluate its mean averaged per-class accuracy per every 200 training iterations. After all training iterations, we save the best mean averaged per-class accuracy. After all tuning trials, we select the best combination of the hyperparameters, which is shown in Tables \ref{tab: Hparams of LSTM-s on HMDB51 in fig SATC} and \ref{tab: Hparams of LSTM-m on HMDB51 in fig SATC} in cyan letters.
The approximated runtime of one statistics trial is 3 hours, and the GPU consumption is 2.6 GBs. 

\begin{table}[htb]
    \caption{
        \textbf{Hyperparameter Search Space of LSTM-s in Figure \ref{fig: SATCs}: HMDB51.} The best hyperparameter combination is highlighted in cyan. $\ga$ controls the strength of monotonicity and is defined as $L_{\mr{total}} = L_{\mathrm{cross\mbox{-}entropy}} + \ga L_{\mathrm{ranking}}$ \protect\citeApp{LSTM_ms}.
    } \label{tab: Hparams of LSTM-s on HMDB51 in fig SATC}
    \begin{center}
    \begin{small}
    \begin{sc}
    \begin{tabular}{ll}
        \toprule
        Learning rate & \{ $10^{-2}$, {$10^{-3}$}, \textcolor{cyan}{$10^{-4}$}, $10^{-5}$ \} \\
        Weight decay & \{ \textcolor{cyan}{$10^{-3}$}, {$10^{-4}$}, {$10^{-5}$} \} \\
        $\ga$ & \{ $10^{-2}$, {$10^{-1}$}, $1$, \textcolor{cyan}{$10$}, $10^{2}$ \} \\
        Optimizer & \{ \textcolor{cyan}{RMSprop} \} \\
        \midrule
        \# tuning trials & 100 \\
        \bottomrule
    \end{tabular}
    \end{sc}
    \end{small}
    \end{center}
\end{table}

\begin{table}[htb]
    \caption{
        \textbf{Hyperparameter Search Space of LSTM-m in Figure \ref{fig: SATCs}: HMDB51.} The best hyperparameter combination is highlighted in cyan. $\ga$ controls the strength of monotonicity and is defined as $L_{\mr{total}} = L_{\mathrm{cross\mbox{-}entropy}} + \ga L_{\mathrm{ranking}}$ \protect\citeApp{LSTM_ms}.
    } \label{tab: Hparams of LSTM-m on HMDB51 in fig SATC}
    \begin{center}
    \begin{small}
    \begin{sc}
    \begin{tabular}{ll}
        \toprule
        Learning rate & \{ $10^{-2}$, {$10^{-3}$}, \textcolor{cyan}{$10^{-4}$}, $10^{-5}$ \} \\
        Weight decay & \{ \textcolor{cyan}{$10^{-3}$}, {$10^{-4}$}, {$10^{-5}$} \} \\
        $\ga$ & \{ {$10^{-2}$}, $10^{-1}$, {$1$}, \textcolor{cyan}{$10$}, $10^{2}$ \} \\
        Optimizer & \{ {Adam}, \textcolor{cyan}{RMSprop} \} \\
        \midrule
        \# tuning trials & 100 \\
        \bottomrule
    \end{tabular}
    \end{sc}
    \end{small}
    \end{center}
\end{table}

\paragraph{EARLIEST}
The main backbone is LSTM \citeApp{LSTM} with the hidden state dimensions set to 256. The whole architecture has 2,374,967 trainable parameters. The batch size is 256. The number of training iterations is 5,000. We train EARLIEST with two different $\la$'s: $10^{-1}$ and $10^{-10}$.

Tables \ref{tab: Hparams of EARLIEST lam-1 on HMDB51 in fig SATC} and \ref{tab: Hparams of EARLIEST lam-10 on HMDB51 in fig SATC} show the search spaces of hyperparameters. For each tuning trial, we train EARLIEST and evaluate its averaged per-class accuracy per every 500 training iterations. After all training iterations, we save the best averaged per-class accuracy. After all tuning trials, we select the best combination of the hyperparameters, which is shown in Tables \ref{tab: Hparams of EARLIEST lam-1 on HMDB51 in fig SATC} and \ref{tab: Hparams of EARLIEST lam-10 on HMDB51 in fig SATC} in cyan letters. The approximated runtime of a single tuning trial is 0.5 hours, and the GPU consumption is 2.0 GBs.

\begin{table}[htb]
    \caption{
        \textbf{Hyperparameter Search Space of EARLIEST with $\la=10^{-1}$ in Figure \ref{fig: SATCs}: HMDB51.} The best hyperparameter combination is highlighted in cyan.
    } \label{tab: Hparams of EARLIEST lam-1 on HMDB51 in fig SATC}
    \begin{center}
    \begin{small}
    \begin{sc}
    \begin{tabular}{ll}
        \toprule
        Learning rate & \{ $10^{-1}$, $10^{-2}$, $10^{-3}$, \textcolor{cyan}{$10^{-4}$}, {$10^{-5}$} \} \\
        Weight decay & \{ \textcolor{cyan}{$10^{-3}$}, {$10^{-4}$}, $10^{-5}$ \} \\
        Optimizer & \{ \textcolor{cyan}{Adam}, {RMSprop} \} \\
        \midrule
        \# tuning trials & 100 \\
        \bottomrule
    \end{tabular}
    \end{sc}
    \end{small}
    \end{center}
\end{table}

\begin{table}[htb]
    \caption{
        \textbf{Hyperparameter Search Space of EARLIEST with $\la=10^{-10}$ in Figure \ref{fig: SATCs}: HMDB51.} The best hyperparameter combination is highlighted in cyan.
    } \label{tab: Hparams of EARLIEST lam-10 on HMDB51 in fig SATC}
    \begin{center}
    \begin{small}
    \begin{sc}
    \begin{tabular}{ll}
        \toprule
        Learning rate & \{ $10^{-1}$, $10^{-2}$, \textcolor{cyan}{$10^{-3}$}, $10^{-4}$, $10^{-5}$ \} \\
        Weight decay & \{ \textcolor{cyan}{$10^{-3}$}, {$10^{-4}$}, $10^{-5}$ \} \\
        Optimizer & \{ Adam, \textcolor{cyan}{RMSprop} \} \\
        \midrule
        \# tuning trials & 100 \\
        \bottomrule
    \end{tabular}
    \end{sc}
    \end{small}
    \end{center}
\end{table}

\subsection{Figure \ref{fig: MSPRT}: LLR Trajectories} \label{app: Figure fig: LLR trajectories of NMNIST-100f}
We plot ten different $i$'s randomly selected from the validation set of NMNIST-100f. 
The base temporal integrator is selected from the models used for NMNIST-100f in Figure \ref{fig: SATCs}. 
More example trajectories are shown in Figure \ref{fig: LLR trajectories of all} in Appendix \ref{app: LLR trajectories}.

\subsection{Figure \ref{fig: Ablation study Mult vs LSEL vs Combo}: Ablation Study of LSEL and Multiplet Loss}
The training procedure is totally similar to that of Figure \ref{fig: SATCs}. Tables \ref{tab: Hparams of TANDEM Mult in fig ablation} and \ref{tab: Hparams of TANDEM LSEL in fig ablation} show the search spaces of hyperparameters. ``TANDEM LSEL+Mult'' is the same model as ``MSPRT-TANDEM'' in Figure \ref{fig: SATCs} (NMNIST-100f).

\begin{table}[htb]
    \caption{
        \textbf{Hyperparameter Search Space of ``TANDEM Mult'' in Figure \ref{fig: Ablation study Mult vs LSEL vs Combo}.} The best hyperparameter combination is highlighted in cyan.
    } \label{tab: Hparams of TANDEM Mult in fig ablation}
    \begin{center}
    \begin{small}
    \begin{sc}
    \begin{tabular}{ll}
        \toprule
        Order  & \{ 0, 25, \textcolor{cyan}{50}, 75, 99 \} \\
        Learning rate & \{ \textcolor{cyan}{$10^{-2}$}, $10^{-3}$, $10^{-4}$ \} \\
        Weight decay & \{ $10^{-3}$, $10^{-4}$, \textcolor{cyan}{$10^{-5}$} \} \\
        $\ga$ & N/A \\
        Optimizer & \{ \textcolor{cyan}{Adam}, RMSprop \} \\
        \midrule
        \# tuning trials & 200 \\
        \bottomrule
    \end{tabular}
    \end{sc}
    \end{small}
    \end{center}
\end{table}

\begin{table}[htb]
    \caption{
        \textbf{Hyperparameter Search Space of ``TANDEM LSEL'' in Figure \ref{fig: Ablation study Mult vs LSEL vs Combo}.} The best hyperparameter combination is highlighted in cyan.
    } \label{tab: Hparams of TANDEM LSEL in fig ablation}
    \begin{center}
    \begin{small}
    \begin{sc}
    \begin{tabular}{ll}
        \toprule
        Order  & \{ 0, 25, 50, 75, \textcolor{cyan}{99} \} \\
        Learning rate & \{ \textcolor{cyan}{$10^{-2}$}, $10^{-3}$, $10^{-4}$ \} \\
        Weight decay & \{ $10^{-3}$, \textcolor{cyan}{$10^{-4}$}, $10^{-5}$ \} \\
        $\ga$ & \{ \textcolor{cyan}{$10^{-1}$}, $1$, $10$, $10^{2}$, $10^{3}$\} \\
        Optimizer & \{ \textcolor{cyan}{Adam}, RMSprop \} \\
        \midrule
        \# tuning trials & 200 \\
        \bottomrule
    \end{tabular}
    \end{sc}
    \end{small}
    \end{center}
\end{table}

\clearpage
\subsection{Figure \ref{fig: DRE Losses vs LSEL} Left: LSEL v.s. Conventional DRE Losses} \label{app: Figure fig: DRE Losses vs LSEL}
We define the loss functions used in Figure \ref{fig: DRE Losses vs LSEL}.
Conventional DRE losses are often biased (LLLR), unbounded (LSIF and DSKL), or numerically unstable (LSIF, LSIFwC, and BARR), especially when applied to multiclass classification, leading to suboptimal performances (Figure \ref{fig: DRE Losses vs LSEL}).)
Because conventional DRE losses are restricted to binary DRE, we modify two LSIF-based and three KLIEP-based loss functions for DRME. The logistic loss we use is introduced in Appendix \ref{app: Modified LSEL and Logistic Loss}.

The original LSIF \citeApp{kanamori2009least_LSIF_original} is based on a kernel method and estimates density ratios via minimizing the mean squared error between $p$ and $\hat{r} q\,$ ( $\hat{r} := \hat{p}/\hat{q}$ ). We define a variant of LSIF for DRME as
\begin{align}
    \hat{L}_{\rm LSIF} := \sum_{t \in [T]} \sum_{\substack{k,l \in [K] \\ (k \neq l)}} \lt[ 
        \f{1}{M_l} \sum_{i \in I_l} \hat{\La}^{2}_{kl} (X_i^{(1,t)})
        - \f{1}{M_k} \sum_{i \in I_k} \hat{\La}_{kl} (X_i^{(1,t)})
    \rt] \, ,
\end{align}
where the likelihood ratio is denoted by $\hat{\La} (X) := e^{\hat{\la} (X)} = \hat{p}(X|k) / \hat{p} (X|l)$. 
Because of the $k,l$-summation, $\hat{L}_{\rm LSIF}$ is symmetric with respect to the denominator and numerator, unlike the original LSIF. Therefore, we can expect that $\hat{L}_{\rm LSIF}$ is more stable than the original one. However, $\hat{L}_{\rm LSIF}$ inherits the problems of the original LSIF; it is unbounded and numerically unstable. The latter is due to dealing with $\hat{\La}$ directly, which easily explodes when the denominator is small. This instability is not negligible, especially when LSIF is used with DNNs. The following LSIF with Constraint (LSIFwC)\footnote{Do not confuse this with cLSIF \citeApp{kanamori2009least_LSIF_original}.} avoids the explosion by adding a constraint:
\begin{align}
    \hat{L}_{\rm LSIFwC} := \sum_{t \in [T]} \sum_{\substack{k,l \in [K] \\ (k \neq l)}} \bigg[& 
        \f{1}{M_l} \sum_{i \in I_l} \hat{\La}^{2}_{kl} (X_i^{(1,t)})
        - \f{1}{M_k} \sum_{i \in I_k} \hat{\La}_{kl} (X_i^{(1,t)}) \nn
        &+ \ga \bigg| 
           \f{1}{M_l} \sum_{i \in I_l} \hat{\La}_{kl} (X_i^{(1,t)}) - 1 
        \bigg|
    \bigg] \, ,
\end{align}
where $\ga > 0$ is a hyperparameter. $\hat{L}_{LSIFwC}$ is symmetric and bounded for sufficiently large $\ga$. However, it is still numerically unstable due to $\hat{\La}$. Note that the constraint term is equivalent to the probability normalization $\int  dx \, p(x|l) \, ( \hat{p}(x|k) / \hat{p}(x|l) )  = 1$.

DSKL \citeApp{khan2019deep_DSKL_BARR_original}, BARR \citeApp{khan2019deep_DSKL_BARR_original}, and LLLR \citeApp{SPRT-TANDEM} are based on KLIEP \citeApp{sugiyama2008direct_KLIEP_original}, which estimates density ratios via minimizing the Kullback-Leibler divergence \citeApp{KLD} between $p$ and $\hat{r} q\,$ ( $\hat{r} := \hat{p}/\hat{q}$ ). We define a variant of DSKL for DRME as
\begin{align}
    \hat{L}_{\rm DSKL} := \sum_{t \in [T]} \sum_{\substack{k,l \in [K] \\ (k \neq l)}} \lt[
        \f{1}{M_l} \sum_{i \in I_l} \hat{\la}_{kl} (X_i^{(1,t)})
        - \f{1}{M_k} \sum_{i \in I_k} \hat{\la}_{kl} (X_i^{(1,t)})
    \rt]
\end{align}
The original DSKL is symmetric, and the same is true for $\hat{L}_{\rm BARR}$, while the original KLIEP is not. $\hat{L}_{\rm BARR}$ is relatively stable compared with $\hat{L}_{\rm LSIF}$ and $\hat{L}_{\rm LSIFwC}$ because it does not include $\hat{\La}$ but $\hat{\la}$; still, $\hat{L}_{\rm DSKL}$ is unbounded and can diverge. BARR for DRME is 
\begin{align}
    \hat{L}_{\rm BARR} := \sum_{t \in [T]} \sum_{\substack{k,l \in [K] \\ (k \neq l)}} \lt[ 
        - \f{1}{M_k} \sum_{i \in I_k} \hat{\la}_{kl} (X_i^{(1,t)}) + \ga \lt|
            \f{1}{M_l} \sum_{i \in I_l} \hat{\La}_{kl} (X_i^{(1,t)}) - 1
        \rt| 
    \rt] \, ,
\end{align}
where $\ga > 0$ is a hyperparameter. $\hat{L}_{\rm BARR}$ is symmetric and bounded because of the second term but is numerically unstable because of $\hat{\La}$. LLLR for DRME is
\begin{align}
    \hat{L}_{\rm LLLR} := \sum_{t\in [T]} \sum_{\substack{k,l \in [K] \\ (k \neq l)}} 
    \f{1}{M_k + M_l} \lt[
        \sum_{i \in I_l} \si ( \hat{\la}_{kl} (X_i^{(1,t)}) ) 
        + \sum_{i \in I_k} ( 1 - \si ( \hat{\la}_{kl} (X_i^{(1,t)}) ) ) 
    \rt] \, ,
\end{align}
where $\si$ is the sigmoid function. LLLR is symmetric, bounded, and numerically stable but is biased in the sense that it does not necessarily converge to the optimal solution $\la$; i.e., the probability normalization constraint, which is explicitly included in the original KLIEP, cannot be exactly satisfied. Finally, the logistic loss is defined as (\ref{eq: logistic loss with empirical approximation}).

All the models share the same feature extractor (Appendix \ref{app: Common Feature Extractor}) and temporal integrator (Appendix \ref{app: Figure fig: SATCs: NMNIST-H}) without the M-TANDEM(wO) approximation or multiplet loss. The search spaces of hyperparameters are given in Tables \ref{tab: Hparams of LSIF in fig DREvsLSEL}--\ref{tab: Hparams of the LSEL in fig DREvsLSEL}. The other setting follows Appendix \ref{app: Figure fig: SATCs: NMNIST-H}.

\begin{table}[htb]
    \caption{
        \textbf{Hyperparameter Search Space of LSIF in Figure \ref{fig: DRE Losses vs LSEL} Left.} The best hyperparameter combination is highlighted in cyan.
    } \label{tab: Hparams of LSIF in fig DREvsLSEL}
    \begin{center}
    \begin{small}
    \begin{sc}
    \begin{tabular}{ll}
        \toprule
        Learning rate & \{ $10^{-3}$, \textcolor{cyan}{$10^{-4}$}, $10^{-5}$, $10^{-6}$ \} \\
        Weight decay & \{ \textcolor{cyan}{$10^{-2}$}, {$10^{-3}$}, {$10^{-4}$}, $10^{-5}$ \} \\
        Optimizer & \{ \textcolor{cyan}{Adam}, RMSprop, Momentum \} \\
        \midrule
        \# tuning trials & 150 \\
        \bottomrule
    \end{tabular}
    \end{sc}
    \end{small}
    \end{center}
\end{table}

\begin{table}[htb]
    \caption{
        \textbf{Hyperparameter Search Space of LSIFwC in Figure \ref{fig: DRE Losses vs LSEL} Left.} The best hyperparameter combination is highlighted in cyan.
    } \label{tab: Hparams of LSIFwC in fig DREvsLSEL}
    \begin{center}
    \begin{small}
    \begin{sc}
    \begin{tabular}{ll}
        \toprule
        Learning rate & \{ $10^{-3}$, \textcolor{cyan}{$10^{-4}$}, $10^{-5}$, $10^{-6}$ \} \\
        Weight decay & \{ \textcolor{cyan}{$10^{-2}$}, {$10^{-3}$}, {$10^{-4}$}, $10^{-5}$ \} \\
        $\ga$ & \{ \textcolor{cyan}{$10^{-4}$}, $10^{-3}$, $10^{-2}$,  {$1$}, $10$ \} \\
        Optimizer & \{ \textcolor{cyan}{Adam}, RMSprop, Momentum \} \\
        \midrule
        \# tuning trials & 150 \\
        \bottomrule
    \end{tabular}
    \end{sc}
    \end{small}
    \end{center}
\end{table}

\begin{table}[htb]
    \caption{
        \textbf{Hyperparameter Search Space of DSKL in Figure \ref{fig: DRE Losses vs LSEL} Left.} The best hyperparameter combination is highlighted in cyan.
    } \label{tab: Hparams of DSKL in fig DREvsLSEL}
    \begin{center}
    \begin{small}
    \begin{sc}
    \begin{tabular}{ll}
        \toprule
        Learning rate & \{ \textcolor{cyan}{$10^{-3}$}, {$10^{-4}$}, $10^{-5}$, $10^{-6}$ \} \\
        Weight decay & \{ {$10^{-2}$}, {$10^{-3}$}, \textcolor{cyan}{$10^{-4}$}, $10^{-5}$ \} \\
        Optimizer & \{ {Adam}, \textcolor{cyan}{RMSprop}, Momentum \} \\
        \midrule
        \# tuning trials & 150 \\
        \bottomrule
    \end{tabular}
    \end{sc}
    \end{small}
    \end{center}
\end{table}

\begin{table}[htb]
    \caption{
        \textbf{Hyperparameter Search Space of BARR in Figure \ref{fig: DRE Losses vs LSEL} Left.} The best hyperparameter combination is highlighted in cyan.
    } \label{tab: Hparams of BARR in fig DREvsLSEL}
    \begin{center}
    \begin{small}
    \begin{sc}
    \begin{tabular}{ll}
        \toprule
        Learning rate & \{ $10^{-3}$, \textcolor{cyan}{$10^{-4}$}, $10^{-5}$, $10^{-6}$ \} \\
        Weight decay & \{ {$10^{-2}$}, {$10^{-3}$}, \textcolor{cyan}{$10^{-4}$}, $10^{-5}$ \} \\
        $\ga$ & \{ $10^{-4}$, \textcolor{cyan}{$10^{-3}$}, $10^{-2}$,  {$1$}, $10$ \} \\
        Optimizer & \{ \textcolor{cyan}{Adam}, RMSprop, Momentum \} \\
        \midrule
        \# tuning trials & 150 \\
        \bottomrule
    \end{tabular}
    \end{sc}
    \end{small}
    \end{center}
\end{table}

\begin{table}[htb]
    \caption{
        \textbf{Hyperparameter Search Space of LLLR in Figure \ref{fig: DRE Losses vs LSEL} Left.} The best hyperparameter combination is highlighted in cyan.
    } \label{tab: Hparams of LLLR in fig DREvsLSEL}
    \begin{center}
    \begin{small}
    \begin{sc}
    \begin{tabular}{ll}
        \toprule
        Learning rate & \{ \textcolor{cyan}{$10^{-3}$}, {$10^{-4}$}, $10^{-5}$, $10^{-6}$ \} \\
        Weight decay & \{ {$10^{-2}$}, {$10^{-3}$}, \textcolor{cyan}{$10^{-4}$}, $10^{-5}$ \} \\
        Optimizer & \{ \textcolor{cyan}{Adam}, RMSprop, Momentum \} \\
        \midrule
        \# tuning trials & 150 \\
        \bottomrule
    \end{tabular}
    \end{sc}
    \end{small}
    \end{center}
\end{table}

\begin{table}[htb]
    \caption{
        \textbf{Hyperparameter Search Space of Logistic Loss in Figure \ref{fig: DRE Losses vs LSEL} Left.} The best hyperparameter combination is highlighted in cyan.
    } \label{tab: Hparams of the logistic loss in fig DREvsLSEL}
    \begin{center}
    \begin{small}
    \begin{sc}
    \begin{tabular}{ll}
        \toprule
        Learning rate & \{ $10^{-3}$, \textcolor{cyan}{$10^{-4}$}, $10^{-5}$, $10^{-6}$ \} \\
        Weight decay & \{ {$10^{-2}$}, {$10^{-3}$}, {$10^{-4}$}, \textcolor{cyan}{$10^{-5}$} \} \\
        Optimizer & \{ {Adam}, \textcolor{cyan}{RMSprop}, Momentum \} \\
        \midrule
        \# tuning trials & 150 \\
        \bottomrule
    \end{tabular}
    \end{sc}
    \end{small}
    \end{center}
\end{table}

\begin{table}[htb]
    \caption{
        \textbf{Hyperparameter Search Space of LSEL in Figure \ref{fig: DRE Losses vs LSEL} Left.} The best hyperparameter combination is highlighted in cyan.
    } \label{tab: Hparams of the LSEL in fig DREvsLSEL}
    \begin{center}
    \begin{small}
    \begin{sc}
    \begin{tabular}{ll}
        \toprule
        Learning rate & \{ $10^{-3}$, \textcolor{cyan}{$10^{-4}$}, $10^{-5}$, $10^{-6}$ \} \\
        Weight decay & \{ {$10^{-2}$}, {$10^{-3}$}, \textcolor{cyan}{$10^{-4}$}, $10^{-5}$ \} \\
        Optimizer & \{ {Adam}, \textcolor{cyan}{RMSprop}, Momentum \} \\
        \midrule
        \# tuning trials & 150 \\
        \bottomrule
    \end{tabular}
    \end{sc}
    \end{small}
    \end{center}
\end{table}

\subsection{Figure \ref{fig: DRE Losses vs LSEL} Right: LSEL v.s. Logistic Loss}
The experimental condition follows that of Appendix \ref{app: Figure fig: SATCs NMNIST-100f}. The logistic loss is defined as (\ref{eq: logistic loss with empirical approximation}). The hyperparameters are given in Tables \ref{tab: Hparams of Logistic in fig DRE Losses vs LSEL} and \ref{tab: Hparams of LSEL in fig DRE Losses vs LSEL}.

\begin{table}[htb]
    \caption{
        \textbf{Hyperparameter Search Space of Logistic Loss in Figure \ref{fig: DRE Losses vs LSEL} Right.} The best hyperparameter combination is highlighted in cyan. 
    } \label{tab: Hparams of Logistic in fig DRE Losses vs LSEL}
    \begin{center}
    \begin{small}
    \begin{sc}
    \begin{tabular}{ll}
        \toprule
        Order  & \{ 0, {25}, 50, 75, \textcolor{cyan}{99} \} \\
        Learning rate & \{ \textcolor{cyan}{$10^{-2}$}, $10^{-3}$, $10^{-4}$ \} \\
        Weight decay & \{ $10^{-3}$, $10^{-4}$, \textcolor{cyan}{$10^{-5}$} \} \\
        $\ga$ & \{ {$10^{-1}$}, $1$, $10$, \textcolor{cyan}{$10^{2}$}, $10^{3}$\} \\
        Optimizer & \{ \textcolor{cyan}{Adam}, RMSprop \} \\
        \midrule
        \# tuning trials & 300 \\
        \bottomrule
    \end{tabular}
    \end{sc}
    \end{small}
    \end{center}
\end{table}

\begin{table}[htb]
    \caption{
        \textbf{Hyperparameter Search Space of LSEL in Figure \ref{fig: DRE Losses vs LSEL} Right.} The best hyperparameter combination is highlighted in cyan. 
    } \label{tab: Hparams of LSEL in fig DRE Losses vs LSEL}
    \begin{center}
    \begin{small}
    \begin{sc}
    \begin{tabular}{ll}
        \toprule
        Order  & \{ 0, {25}, 50, 75, \textcolor{cyan}{99} \} \\
        Learning rate & \{ \textcolor{cyan}{$10^{-2}$}, $10^{-3}$, $10^{-4}$ \} \\
        Weight decay & \{ $10^{-3}$, \textcolor{cyan}{$10^{-4}$}, {$10^{-5}$} \} \\
        $\ga$ & \{ \textcolor{cyan}{$10^{-1}$}, $1$, $10$, $10^{2}$, $10^{3}$\} \\
        Optimizer & \{ \textcolor{cyan}{Adam}, RMSprop \} \\
        \midrule
        \# tuning trials & 300 \\
        \bottomrule
    \end{tabular}
    \end{sc}
    \end{small}
    \end{center}
\end{table}

\clearpage
\subsection{Exact Error Rates in Figures \ref{fig: SATCs} and \ref{fig: Ablation study Mult vs LSEL vs Combo}}
Tables \ref{tab: DRE Losses vs LSEL (Left: NMNIST-H)}--\ref{tab: Values of ablation on NMNSIT-100f 51--100} show the averaged per-class error rates plotted in the figures in Section \ref{sec:experiment}.

\begin{table*}[h]
    \caption{
        \textbf{Averaged Per-Class Error Rates (\%) and SEM of Figure \ref{fig: DRE Losses vs LSEL} (Left: NMNIST-H).} Blanks mean N/A.
    } \label{tab: DRE Losses vs LSEL (Left: NMNIST-H)}
    \begin{minipage}[b]{1.0\linewidth}
        \begin{center}
        \begin{small}
        \begin{sc}
        \begin{tabular}{lllll}
            \\
            \toprule
            Time & LSIF & LSIFwC & DSKL & BARR \\
            \midrule
            1.00 & 83.822 $\pm$ 2.142 & 81.585 $\pm$ 2.908 &  & 66.579 $\pm$ 0.108  \\
            2.00 & 73.330 $\pm$ 2.142 & 67.749 $\pm$ 2.908 & 45.145 $\pm$ 0.170 & 45.902 $\pm$ 0.108  \\
            3.00 & 70.492 $\pm$ 2.142 & 63.237 $\pm$ 2.908 & 30.719 $\pm$ 0.170 & 30.698 $\pm$ 0.108  \\
            4.00 & 68.258 $\pm$ 2.142 & 60.061 $\pm$ 2.908 & 20.409 $\pm$ 0.170 & 20.172 $\pm$ 0.108  \\
            5.00 & 66.485 $\pm$ 2.142 & 57.860 $\pm$ 2.908 & 18.102 $\pm$ 0.170 & 12.990 $\pm$ 0.108  \\
            6.00 & 64.975 $\pm$ 2.142 & 55.939 $\pm$ 2.908 & 17.899 $\pm$ 0.170 & 8.374 $\pm$ 0.108  \\
            7.00 & 63.549 $\pm$ 2.142 & 54.317 $\pm$ 2.908 & 14.442 $\pm$ 0.170 & 5.529 $\pm$ 0.108  \\
            8.00 & 62.137 $\pm$ 2.142 & 52.635 $\pm$ 2.908 & 12.942 $\pm$ 0.170 & 3.993 $\pm$ 0.108  \\
            9.00 & 60.551 $\pm$ 2.142 & 50.886 $\pm$ 2.908 & 11.681 $\pm$ 0.170 & 3.202 $\pm$ 0.108  \\
            10.00 & 58.657 $\pm$ 2.142 & 48.947 $\pm$ 2.908 & 10.368 $\pm$ 0.170 & 2.787 $\pm$ 0.108  \\
            11.00 & 56.716 $\pm$ 2.142 & 46.501 $\pm$ 2.908 & 9.858 $\pm$ 0.170 & 2.563 $\pm$ 0.108  \\
            12.00 & 54.976 $\pm$ 2.142 & 44.114 $\pm$ 2.908 & 8.493 $\pm$ 0.170 & 2.457 $\pm$ 0.108  \\
            13.00 & 52.800 $\pm$ 2.142 & 41.901 $\pm$ 2.908 & 7.909 $\pm$ 0.170 & 2.404 $\pm$ 0.108  \\
            14.00 & 49.940 $\pm$ 2.142 & 40.073 $\pm$ 2.908 & 7.151 $\pm$ 0.170 & 2.376 $\pm$ 0.108  \\
            15.00 & 47.293 $\pm$ 2.142 & 38.990 $\pm$ 2.908 & 6.216 $\pm$ 0.170 & 2.359 $\pm$ 0.108  \\
            16.00 & 45.203 $\pm$ 2.142 & 38.013 $\pm$ 2.908 & 5.821 $\pm$ 0.170 & 2.352 $\pm$ 0.108  \\
            17.00 & 42.664 $\pm$ 2.142 & 36.799 $\pm$ 2.908 & 4.632 $\pm$ 0.170 & 2.348 $\pm$ 0.108  \\
            18.00 & 38.406 $\pm$ 2.142 & 34.540 $\pm$ 2.908 & 3.971 $\pm$ 0.170 & 2.345 $\pm$ 0.108  \\
            19.00 & 32.666 $\pm$ 2.142 & 30.737 $\pm$ 2.908 & 3.484 $\pm$ 0.170 & 2.343 $\pm$ 0.108  \\
            20.00 & 30.511 $\pm$ 2.142 & 29.880 $\pm$ 2.908 & 2.088 $\pm$ 0.170 & 2.343 $\pm$ 0.108  \\
            \midrule
            Trials & 40 & 40 & 60 & 40 \\
            \bottomrule
            \\
        \end{tabular}
        \end{sc}
        \end{small}
        \end{center}
    \end{minipage}
        \begin{minipage}[b]{1.0\linewidth}
        \begin{center}
        \begin{small}
        \begin{sc}
        \begin{tabular}{llll}
            \toprule
            Time & LLLR & Logistic & LSEL \\
            \midrule
            1.00 & 63.930 $\pm$ 0.026 &  & 63.956 $\pm$ 0.020 \\
            2.00 & 44.250 $\pm$ 0.026 & 43.873 $\pm$ 0.097 & 43.918 $\pm$ 0.020 \\
            3.00 & 30.693 $\pm$ 0.026 & 29.058 $\pm$ 0.097 & 28.957 $\pm$ 0.020 \\
            4.00 & 21.370 $\pm$ 0.026 & 18.973 $\pm$ 0.097 & 18.865 $\pm$ 0.020 \\
            5.00 & 14.834 $\pm$ 0.026 & 11.992 $\pm$ 0.097 & 11.910 $\pm$ 0.020 \\
            6.00 & 10.386 $\pm$ 0.026 & 7.522 $\pm$ 0.097 & 7.444 $\pm$ 0.020 \\
            7.00 & 7.516 $\pm$ 0.026 & 4.970 $\pm$ 0.097 & 4.877 $\pm$ 0.020 \\
            8.00 & 5.693 $\pm$ 0.026 & 3.485 $\pm$ 0.097 & 3.403 $\pm$ 0.020 \\
            9.00 & 4.563 $\pm$ 0.026 & 2.702 $\pm$ 0.097 & 2.647 $\pm$ 0.020 \\
            10.00 & 3.818 $\pm$ 0.026 & 2.326 $\pm$ 0.097 & 2.288 $\pm$ 0.020 \\
            11.00 & 3.316 $\pm$ 0.026 & 2.133 $\pm$ 0.097 & 2.103 $\pm$ 0.020 \\
            12.00 & 2.961 $\pm$ 0.026 & 2.030 $\pm$ 0.097 & 2.010 $\pm$ 0.020 \\
            13.00 & 2.705 $\pm$ 0.026 & 1.985 $\pm$ 0.097 & 1.977 $\pm$ 0.020 \\
            14.00 & 2.516 $\pm$ 0.026 & 1.964 $\pm$ 0.097 & 1.956 $\pm$ 0.020 \\
            15.00 & 2.365 $\pm$ 0.026 & 1.951 $\pm$ 0.097 & 1.941 $\pm$ 0.020 \\
            16.00 & 2.254 $\pm$ 0.026 & 1.943 $\pm$ 0.097 & 1.934 $\pm$ 0.020 \\
            17.00 & 2.163 $\pm$ 0.026 & 1.938 $\pm$ 0.097 & 1.932 $\pm$ 0.020 \\
            18.00 & 2.094 $\pm$ 0.026 & 1.937 $\pm$ 0.097 & 1.932 $\pm$ 0.020 \\
            19.00 & 2.036 $\pm$ 0.026 & 1.937 $\pm$ 0.097 & 1.932 $\pm$ 0.020 \\
            20.00 & 1.968 $\pm$ 0.026 & 1.937 $\pm$ 0.097 & 1.932 $\pm$ 0.020 \\
            \midrule
            Trials & 80 & 80 & 80 \\
            \bottomrule
        \end{tabular}
        \end{sc}
        \end{small}
        \end{center}
    \end{minipage}
\end{table*}

\begin{table*}[h]
    \caption{
        \textbf{Averaged Per-Class Error Rates (\%) and SEM of Figure \ref{fig: DRE Losses vs LSEL} (Right: NMNIST-100f). Frames 1--50.} Blanks mean N/A.
    } \label{tab: DRE Losses vs LSEL (Right: NMNIST-100f, 1--50)}
    \begin{center}
    \begin{small}
    \begin{sc}
    \begin{tabular}{lll}
        \toprule
        Time & Logistic (with M-TANDEM) & LSEL (with M-TANDEM) \\
        \midrule
        1.00 & 56.273 $\pm$ 0.045 & 56.196 $\pm$ 0.044 \\
        2.00 & 37.772 $\pm$ 0.045 & 37.492 $\pm$ 0.044 \\
        3.00 & 27.922 $\pm$ 0.045 & 27.520 $\pm$ 0.044 \\
        4.00 & 22.282 $\pm$ 0.045 & 21.889 $\pm$ 0.044 \\
        5.00 & 18.725 $\pm$ 0.045 & 18.369 $\pm$ 0.044 \\
        6.00 & 16.251 $\pm$ 0.045 & 15.901 $\pm$ 0.044 \\
        7.00 & 14.383 $\pm$ 0.045 & 14.035 $\pm$ 0.044 \\
        8.00 & 12.898 $\pm$ 0.045 & 12.590 $\pm$ 0.044 \\
        9.00 & 11.730 $\pm$ 0.045 & 11.421 $\pm$ 0.044 \\
        10.00 & 10.758 $\pm$ 0.045 & 10.468 $\pm$ 0.044 \\
        11.00 & 9.943 $\pm$ 0.045 & 9.659 $\pm$ 0.044 \\
        12.00 & 9.245 $\pm$ 0.045 & 8.992 $\pm$ 0.044 \\
        13.00 & 8.663 $\pm$ 0.045 & 8.410 $\pm$ 0.044 \\
        14.00 & 8.158 $\pm$ 0.045 & 7.907 $\pm$ 0.044 \\
        15.00 & 7.721 $\pm$ 0.045 & 7.472 $\pm$ 0.044 \\
        16.00 & 7.341 $\pm$ 0.045 & 7.096 $\pm$ 0.044 \\
        17.00 & 7.006 $\pm$ 0.045 & 6.776 $\pm$ 0.044 \\
        18.00 & 6.717 $\pm$ 0.045 & 6.498 $\pm$ 0.044 \\
        19.00 & 6.476 $\pm$ 0.045 & 6.257 $\pm$ 0.044 \\
        20.00 & 6.258 $\pm$ 0.045 & 6.041 $\pm$ 0.044 \\
        21.00 & 6.069 $\pm$ 0.045 & 5.851 $\pm$ 0.044 \\
        22.00 & 5.900 $\pm$ 0.045 & 5.686 $\pm$ 0.044 \\
        23.00 & 5.748 $\pm$ 0.045 & 5.545 $\pm$ 0.044 \\
        24.00 & 5.615 $\pm$ 0.045 & 5.423 $\pm$ 0.044 \\
        25.00 & 5.502 $\pm$ 0.045 & 5.317 $\pm$ 0.044 \\
        26.00 & 5.406 $\pm$ 0.045 & 5.221 $\pm$ 0.044 \\
        27.00 & 5.318 $\pm$ 0.045 & 5.136 $\pm$ 0.044 \\
        28.00 & 5.246 $\pm$ 0.045 & 5.064 $\pm$ 0.044 \\
        29.00 & 5.175 $\pm$ 0.045 & 5.001 $\pm$ 0.044 \\
        30.00 & 5.114 $\pm$ 0.045 & 4.947 $\pm$ 0.044 \\
        31.00 & 5.057 $\pm$ 0.045 & 4.898 $\pm$ 0.044 \\
        32.00 & 5.011 $\pm$ 0.045 & 4.854 $\pm$ 0.044 \\
        33.00 & 4.974 $\pm$ 0.045 & 4.816 $\pm$ 0.044 \\
        34.00 & 4.939 $\pm$ 0.045 & 4.784 $\pm$ 0.044 \\
        35.00 & 4.908 $\pm$ 0.045 & 4.757 $\pm$ 0.044 \\
        36.00 & 4.882 $\pm$ 0.045 & 4.733 $\pm$ 0.044 \\
        37.00 & 4.858 $\pm$ 0.045 & 4.713 $\pm$ 0.044 \\
        38.00 & 4.837 $\pm$ 0.045 & 4.694 $\pm$ 0.044 \\
        39.00 & 4.818 $\pm$ 0.045 & 4.676 $\pm$ 0.044 \\
        40.00 & 4.801 $\pm$ 0.045 & 4.661 $\pm$ 0.044 \\
        41.00 & 4.786 $\pm$ 0.045 & 4.649 $\pm$ 0.044 \\
        42.00 & 4.773 $\pm$ 0.045 & 4.639 $\pm$ 0.044 \\
        43.00 & 4.761 $\pm$ 0.045 & 4.630 $\pm$ 0.044 \\
        44.00 & 4.751 $\pm$ 0.045 & 4.623 $\pm$ 0.044 \\
        45.00 & 4.742 $\pm$ 0.045 & 4.615 $\pm$ 0.044 \\
        46.00 & 4.735 $\pm$ 0.045 & 4.608 $\pm$ 0.044 \\
        47.00 & 4.727 $\pm$ 0.045 & 4.604 $\pm$ 0.044 \\
        48.00 & 4.722 $\pm$ 0.045 & 4.598 $\pm$ 0.044 \\
        49.00 & 4.716 $\pm$ 0.045 & 4.594 $\pm$ 0.044 \\
        50.00 & 4.711 $\pm$ 0.045 & 4.590 $\pm$ 0.044 \\
        \midrule
        Trials & 150 & 150  \\
        \bottomrule
    \end{tabular}
    \end{sc}
    \end{small}
    \end{center}
\end{table*}

\begin{table*}[h]
    \caption{
        \textbf{Averaged Per-Class Error Rates (\%) and SEM of Figure \ref{fig: DRE Losses vs LSEL} (Right: NMNIST-100f). Frames 51--100.} Blanks mean N/A.
    } \label{tab: DRE Losses vs LSEL (Right: NMNIST-100f, 51--100)}
    \begin{center}
    \begin{small}
    \begin{sc}
    \begin{tabular}{lll}
        \toprule
        Time & Logistic (with M-TANDEM) & LSEL (with M-TANDEM) \\
        \midrule
        51.00 & 4.706 $\pm$ 0.045 & 4.587 $\pm$ 0.044 \\
        52.00 & 4.703 $\pm$ 0.045 & 4.586 $\pm$ 0.044 \\
        53.00 & 4.700 $\pm$ 0.045 & 4.584 $\pm$ 0.044 \\
        54.00 & 4.697 $\pm$ 0.045 & 4.580 $\pm$ 0.044 \\
        55.00 & 4.694 $\pm$ 0.045 & 4.577 $\pm$ 0.044 \\
        56.00 & 4.692 $\pm$ 0.045 & 4.576 $\pm$ 0.044 \\
        57.00 & 4.690 $\pm$ 0.045 & 4.575 $\pm$ 0.044 \\
        58.00 & 4.688 $\pm$ 0.045 & 4.573 $\pm$ 0.044 \\
        59.00 & 4.687 $\pm$ 0.045 & 4.572 $\pm$ 0.044 \\
        60.00 & 4.686 $\pm$ 0.045 & 4.572 $\pm$ 0.044 \\
        61.00 & 4.685 $\pm$ 0.045 & 4.571 $\pm$ 0.044 \\
        62.00 & 4.684 $\pm$ 0.045 & 4.570 $\pm$ 0.044 \\
        63.00 & 4.683 $\pm$ 0.045 & 4.569 $\pm$ 0.044 \\
        64.00 & 4.683 $\pm$ 0.045 & 4.569 $\pm$ 0.044 \\
        65.00 & 4.682 $\pm$ 0.045 & 4.569 $\pm$ 0.044 \\
        66.00 & 4.681 $\pm$ 0.045 & 4.568 $\pm$ 0.044 \\
        67.00 & 4.681 $\pm$ 0.045 & 4.568 $\pm$ 0.044 \\
        68.00 & 4.680 $\pm$ 0.045 & 4.567 $\pm$ 0.044 \\
        69.00 & 4.679 $\pm$ 0.045 & 4.567 $\pm$ 0.044 \\
        70.00 & 4.678 $\pm$ 0.045 & 4.567 $\pm$ 0.044 \\
        71.00 & 4.678 $\pm$ 0.045 & 4.566 $\pm$ 0.044 \\
        72.00 & 4.678 $\pm$ 0.045 & 4.566 $\pm$ 0.044 \\
        73.00 & 4.678 $\pm$ 0.045 & 4.565 $\pm$ 0.044 \\
        74.00 & 4.678 $\pm$ 0.045 & 4.565 $\pm$ 0.044 \\
        75.00 & 4.678 $\pm$ 0.045 & 4.565 $\pm$ 0.044 \\
        76.00 & 4.678 $\pm$ 0.045 & 4.565 $\pm$ 0.044 \\
        77.00 & 4.678 $\pm$ 0.045 & 4.565 $\pm$ 0.044 \\
        78.00 & 4.677 $\pm$ 0.045 & 4.565 $\pm$ 0.044 \\
        79.00 & 4.677 $\pm$ 0.045 & 4.564 $\pm$ 0.044 \\
        80.00 & 4.676 $\pm$ 0.045 & 4.565 $\pm$ 0.044 \\
        81.00 & 4.676 $\pm$ 0.045 & 4.564 $\pm$ 0.044 \\
        82.00 & 4.676 $\pm$ 0.045 & 4.564 $\pm$ 0.044 \\
        83.00 & 4.676 $\pm$ 0.045 & 4.564 $\pm$ 0.044 \\
        84.00 & 4.676 $\pm$ 0.045 & 4.564 $\pm$ 0.044 \\
        85.00 & 4.676 $\pm$ 0.045 & 4.564 $\pm$ 0.044 \\
        86.00 & 4.676 $\pm$ 0.045 & 4.564 $\pm$ 0.044 \\
        87.00 & 4.676 $\pm$ 0.045 & 4.564 $\pm$ 0.044 \\
        88.00 & 4.676 $\pm$ 0.045 & 4.564 $\pm$ 0.044 \\
        89.00 & 4.676 $\pm$ 0.045 & 4.564 $\pm$ 0.044 \\
        90.00 & 4.676 $\pm$ 0.045 & 4.564 $\pm$ 0.044 \\
        91.00 & 4.676 $\pm$ 0.045 & 4.564 $\pm$ 0.044 \\
        92.00 & 4.676 $\pm$ 0.045 & 4.564 $\pm$ 0.044 \\
        93.00 & 4.676 $\pm$ 0.045 & 4.564 $\pm$ 0.044 \\
        94.00 & 4.676 $\pm$ 0.045 & 4.564 $\pm$ 0.044 \\
        95.00 & 4.676 $\pm$ 0.045 & 4.564 $\pm$ 0.044 \\
        96.00 & 4.676 $\pm$ 0.045 & 4.564 $\pm$ 0.044 \\
        97.00 & 4.676 $\pm$ 0.045 & 4.564 $\pm$ 0.044 \\
        98.00 & 4.676 $\pm$ 0.045 & 4.564 $\pm$ 0.044 \\
        99.00 & 4.676 $\pm$ 0.045 & 4.564 $\pm$ 0.044 \\
        100.00 & 4.676 $\pm$ 0.045 & 4.564 $\pm$ 0.044 \\
        \midrule
        Trials & 150 & 150 \\
        \bottomrule
    \end{tabular}
    \end{sc}
    \end{small}
    \end{center}
\end{table*}

\begin{table*}[h]
    \caption{
        \textbf{Averaged Per-Class Error Rates (\%) and SEM of Figure \ref{fig: SATCs} (NMNIST-H).} Blanks mean N/A.
    } \label{tab: Values of fig SATCs: NMNIST-H}
    \begin{center}
    \begin{small}
    \begin{sc}
    \begin{tabular}{llllll}
        \toprule
        Time & MSPRT-TANDEM & LSTM-s & LSTM-m & EARLIEST $10^{-2}$ & EARLIEST $10^{2}$ \\
        \midrule
        1.00 & & 64.003 $\pm$ 0.010 & 64.009 $\pm$ 0.011 &  &  \\
        1.34 & & & & & 57.973 $\pm$ 0.066 \\
        2.00 & 44.172 $\pm$ 0.046 & 46.676 $\pm$ 0.008 & 46.697 $\pm$ 0.008 & & \\
        3.00 & 29.190 $\pm$ 0.046 & 35.517 $\pm$ 0.007 & 35.536 $\pm$ 0.008 & & \\
        4.00 & 18.964 $\pm$ 0.046 & 25.936 $\pm$ 0.007 & 25.940 $\pm$ 0.007 & & \\
        5.00 & 11.994 $\pm$ 0.046 & 20.510 $\pm$ 0.007 & 20.512 $\pm$ 0.006 & & \\
        6.00 & 7.449 $\pm$ 0.046 & 15.393 $\pm$ 0.007 & 15.404 $\pm$ 0.006 & & \\
        7.00 & 4.870 $\pm$ 0.046 & 12.603 $\pm$ 0.005 & 12.614 $\pm$ 0.005 & & \\
        8.00 & 3.383 $\pm$ 0.046 & 10.025 $\pm$ 0.005 & 10.018 $\pm$ 0.005 & & \\
        9.00 & 2.639 $\pm$ 0.046 & 8.036 $\pm$ 0.005 & 8.033 $\pm$ 0.005 & & \\
        10.00 & 2.281 $\pm$ 0.046 & 6.963 $\pm$ 0.005 & 6.958 $\pm$ 0.005 & & \\
        11.00 & 2.087 $\pm$ 0.046 & 5.788 $\pm$ 0.005 & 5.801 $\pm$ 0.005 & & \\
        12.00 & 1.996 $\pm$ 0.046 & 4.886 $\pm$ 0.003 & 4.892 $\pm$ 0.004 & & \\
        13.00 & 1.963 $\pm$ 0.046 & 4.398 $\pm$ 0.003 & 4.400 $\pm$ 0.003 & & \\
        13.41 & & & & 3.162 $\pm$ 0.027 & \\
        14.00 & 1.942 $\pm$ 0.046 & 3.735 $\pm$ 0.003 & 3.737 $\pm$ 0.003 & & \\
        15.00 & 1.926 $\pm$ 0.046 & 3.190 $\pm$ 0.003 & 3.198 $\pm$ 0.003 & & \\
        16.00 & 1.919 $\pm$ 0.046 & 2.841 $\pm$ 0.003 & 2.850 $\pm$ 0.003 & & \\
        17.00 & 1.917 $\pm$ 0.046 & 2.576 $\pm$ 0.003 & 2.576 $\pm$ 0.003 & & \\
        18.00 & 1.916 $\pm$ 0.046 & 2.376 $\pm$ 0.003 & 2.374 $\pm$ 0.003 & & \\
        19.00 & 1.916 $\pm$ 0.046 & 2.118 $\pm$ 0.003 & 2.119 $\pm$ 0.003 & & \\
        20.00 & 1.916 $\pm$ 0.046 & 1.923 $\pm$ 0.003 & 1.921 $\pm$ 0.003 & & \\
        \midrule
        Trials & 300 & 300 & 300 & 50 & 45 \\
        \bottomrule
    \end{tabular}
    \end{sc}
    \end{small}
    \end{center}
\end{table*}

\begin{table*}[h]
    \caption{
        \textbf{Averaged Per-Class Error Rates (\%) and SEM of Figure \ref{fig: SATCs} (NMNIST-100f). Frames 1--50.} Blanks mean N/A.
    } \label{tab: Values of fig SATCs: NMNIST-100f, 1--50}
    \begin{center}
    \begin{small}
    \begin{sc}
    \begin{tabular}{llllll}
        \toprule
        Time & MSPRT-TANDEM & LSTM-s & LSTM-m & EARLIEST Lam1e-2 & EARLIEST Lam1e-4 \\
        \midrule
        1.00 & & 54.677 $\pm$ 0.016 & 54.696 $\pm$ 0.016 &  &  \\
        2.00 & 35.942 $\pm$ 0.148 & 38.748 $\pm$ 0.014 & 38.677 $\pm$ 0.015 & & \\
        3.00 & 25.203 $\pm$ 0.148 & 29.348 $\pm$ 0.014 & 29.350 $\pm$ 0.014 & & \\
        4.00 & 19.167 $\pm$ 0.148 & 24.048 $\pm$ 0.014 & 23.981 $\pm$ 0.015 & & \\
        5.00 & 15.450 $\pm$ 0.148 & 20.490 $\pm$ 0.014 & 20.422 $\pm$ 0.014 & & \\
        6.00 & 12.997 $\pm$ 0.148 & 18.065 $\pm$ 0.014 & 17.959 $\pm$ 0.015 & & \\
        7.00 & 11.224 $\pm$ 0.148 & 15.715 $\pm$ 0.014 & 15.606 $\pm$ 0.015 & & \\
        8.00 & 9.912 $\pm$ 0.148 & 14.306 $\pm$ 0.014 & 14.215 $\pm$ 0.014 & & \\
        9.00 & 8.880 $\pm$ 0.148 & 13.124 $\pm$ 0.014 & 13.046 $\pm$ 0.013 & & \\
        10.00 & 8.068 $\pm$ 0.148 & 12.129 $\pm$ 0.012 & 12.038 $\pm$ 0.012 & & \\
        11.00 & 7.405 $\pm$ 0.148 & 11.386 $\pm$ 0.013 & 11.302 $\pm$ 0.013 & & \\
        12.00 & 6.862 $\pm$ 0.148 & 10.838 $\pm$ 0.012 & 10.767 $\pm$ 0.013 & & \\
        13.00 & 6.409 $\pm$ 0.148 & 10.121 $\pm$ 0.012 & 10.034 $\pm$ 0.012 & & \\
        14.00 & 6.032 $\pm$ 0.148 & 9.626 $\pm$ 0.013 & 9.555 $\pm$ 0.013 & & \\
        15.00 & 5.715 $\pm$ 0.148 & 9.102 $\pm$ 0.012 & 9.007 $\pm$ 0.012 & & \\
        16.00 & 5.444 $\pm$ 0.148 & 8.710 $\pm$ 0.013 & 8.625 $\pm$ 0.013 & & \\
        17.00 & 5.212 $\pm$ 0.148 & 8.268 $\pm$ 0.011 & 8.235 $\pm$ 0.012 & & \\
        18.00 & 5.025 $\pm$ 0.148 & 7.903 $\pm$ 0.012 & 7.852 $\pm$ 0.012 & & \\
        19.00 & 4.865 $\pm$ 0.148 & 7.613 $\pm$ 0.012 & 7.586 $\pm$ 0.011 & & \\
        19.46 & & & & 20.044 $\pm$ 0.229 & \\
        20.00 & 4.730 $\pm$ 0.148 & 7.366 $\pm$ 0.012 & 7.354 $\pm$ 0.011 & & \\
        21.00 & 4.623 $\pm$ 0.148 & 7.150 $\pm$ 0.012 & 7.150 $\pm$ 0.011 & & \\
        22.00 & 4.536 $\pm$ 0.148 & 6.998 $\pm$ 0.011 & 7.006 $\pm$ 0.011 & & \\
        23.00 & 4.464 $\pm$ 0.148 & 6.816 $\pm$ 0.011 & 6.821 $\pm$ 0.011 & & \\
        24.00 & 4.401 $\pm$ 0.148 & 6.674 $\pm$ 0.011 & 6.655 $\pm$ 0.011 & & \\
        25.00 & 4.353 $\pm$ 0.148 & 6.491 $\pm$ 0.012 & 6.502 $\pm$ 0.013 & & \\
        26.00 & 4.310 $\pm$ 0.148 & 6.351 $\pm$ 0.011 & 6.373 $\pm$ 0.012 & & \\
        27.00 & 4.275 $\pm$ 0.148 & 6.200 $\pm$ 0.011 & 6.200 $\pm$ 0.013 & & \\
        28.00 & 4.246 $\pm$ 0.148 & 6.061 $\pm$ 0.011 & 6.077 $\pm$ 0.013 & & \\
        29.00 & 4.222 $\pm$ 0.148 & 5.935 $\pm$ 0.011 & 5.948 $\pm$ 0.012 & & \\
        30.00 & 4.203 $\pm$ 0.148 & 5.876 $\pm$ 0.012 & 5.870 $\pm$ 0.011 & & \\
        31.00 & 4.186 $\pm$ 0.148 & 5.811 $\pm$ 0.011 & 5.823 $\pm$ 0.012 & & \\
        32.00 & 4.171 $\pm$ 0.148 & 5.658 $\pm$ 0.011 & 5.676 $\pm$ 0.012 & & \\
        33.00 & 4.160 $\pm$ 0.148 & 5.552 $\pm$ 0.011 & 5.566 $\pm$ 0.012 & & \\
        34.00 & 4.147 $\pm$ 0.148 & 5.456 $\pm$ 0.012 & 5.474 $\pm$ 0.012 & & \\
        35.00 & 4.139 $\pm$ 0.148 & 5.395 $\pm$ 0.011 & 5.412 $\pm$ 0.012 & & \\
        36.00 & 4.133 $\pm$ 0.148 & 5.307 $\pm$ 0.011 & 5.319 $\pm$ 0.012 & & \\
        37.00 & 4.126 $\pm$ 0.148 & 5.231 $\pm$ 0.011 & 5.249 $\pm$ 0.012 & & \\
        38.00 & 4.121 $\pm$ 0.148 & 5.183 $\pm$ 0.011 & 5.211 $\pm$ 0.012 & & \\
        39.00 & 4.118 $\pm$ 0.148 & 5.148 $\pm$ 0.011 & 5.182 $\pm$ 0.011 & & \\
        40.00 & 4.116 $\pm$ 0.148 & 5.121 $\pm$ 0.011 & 5.168 $\pm$ 0.012 & & \\
        41.00 & 4.113 $\pm$ 0.148 & 5.055 $\pm$ 0.011 & 5.098 $\pm$ 0.012 & & \\
        42.00 & 4.112 $\pm$ 0.148 & 5.057 $\pm$ 0.012 & 5.091 $\pm$ 0.013 & & \\
        43.00 & 4.110 $\pm$ 0.148 & 4.980 $\pm$ 0.012 & 5.017 $\pm$ 0.012 & & \\
        44.00 & 4.109 $\pm$ 0.148 & 4.940 $\pm$ 0.011 & 4.976 $\pm$ 0.013 & & \\
        45.00 & 4.109 $\pm$ 0.148 & 4.892 $\pm$ 0.012 & 4.921 $\pm$ 0.013 & & \\
        46.00 & 4.108 $\pm$ 0.148 & 4.825 $\pm$ 0.012 & 4.865 $\pm$ 0.013 & & \\
        47.00 & 4.108 $\pm$ 0.148 & 4.780 $\pm$ 0.012 & 4.795 $\pm$ 0.013 & & \\
        48.00 & 4.108 $\pm$ 0.148 & 4.730 $\pm$ 0.012 & 4.767 $\pm$ 0.013 & & \\
        49.00 & 4.107 $\pm$ 0.148 & 4.709 $\pm$ 0.012 & 4.742 $\pm$ 0.013 & & \\
        50.00 & 4.107 $\pm$ 0.148 & 4.690 $\pm$ 0.012 & 4.731 $\pm$ 0.013 & & \\
        \midrule
        Trials & 150 & 150  & 150 & 50 & 60  \\
        \bottomrule
    \end{tabular}
    \end{sc}
    \end{small}
    \end{center}
\end{table*}

\begin{table*}[h]
    \caption{
        \textbf{Averaged Per-Class Error Rates (\%) and SEM of Figure \ref{fig: SATCs} (NMNIST-100f). Frames 51--100.} Blanks mean N/A.
    } \label{tab: Values of fig SATCs:NMNIST-100f, 51--100}
    \begin{center}
    \begin{small}
    \begin{sc}
    \begin{tabular}{llllll}
        \toprule
        Time & MSPRT-TANDEM & LSTM-s & LSTM-m & EARLIEST $10^{-2}$ & EARLIEST $10^{-4}$ \\
        \midrule
        51.00 & 4.107 $\pm$ 0.148 & 4.620 $\pm$ 0.012 & 4.661 $\pm$ 0.014 & & \\
        52.00 & 4.107 $\pm$ 0.148 & 4.588 $\pm$ 0.012 & 4.637 $\pm$ 0.014 & & \\
        53.00 & 4.107 $\pm$ 0.148 & 4.595 $\pm$ 0.011 & 4.634 $\pm$ 0.013 & & \\
        54.00 & 4.106 $\pm$ 0.148 & 4.538 $\pm$ 0.011 & 4.580 $\pm$ 0.013 & & \\
        55.00 & 4.107 $\pm$ 0.148 & 4.526 $\pm$ 0.010 & 4.566 $\pm$ 0.013 & & \\
        56.00 & 4.107 $\pm$ 0.148 & 4.531 $\pm$ 0.012 & 4.555 $\pm$ 0.013 & & \\
        57.00 & 4.107 $\pm$ 0.148 & 4.470 $\pm$ 0.011 & 4.509 $\pm$ 0.012 & & \\
        58.00 & 4.107 $\pm$ 0.148 & 4.444 $\pm$ 0.011 & 4.489 $\pm$ 0.012 & & \\
        59.00 & 4.108 $\pm$ 0.148 & 4.449 $\pm$ 0.011 & 4.496 $\pm$ 0.012 & & \\
        60.00 & 4.108 $\pm$ 0.148 & 4.447 $\pm$ 0.011 & 4.494 $\pm$ 0.013 & & \\
        61.00 & 4.108 $\pm$ 0.148 & 4.459 $\pm$ 0.011 & 4.484 $\pm$ 0.012 & & \\
        62.00 & 4.108 $\pm$ 0.148 & 4.408 $\pm$ 0.011 & 4.447 $\pm$ 0.012 & & \\
        63.00 & 4.108 $\pm$ 0.148 & 4.391 $\pm$ 0.011 & 4.429 $\pm$ 0.013 & & \\
        64.00 & 4.108 $\pm$ 0.148 & 4.377 $\pm$ 0.011 & 4.412 $\pm$ 0.013 & & \\
        65.00 & 4.108 $\pm$ 0.148 & 4.370 $\pm$ 0.011 & 4.397 $\pm$ 0.013 & & \\
        66.00 & 4.109 $\pm$ 0.148 & 4.378 $\pm$ 0.011 & 4.405 $\pm$ 0.013 & & \\
        67.00 & 4.109 $\pm$ 0.148 & 4.394 $\pm$ 0.011 & 4.422 $\pm$ 0.012 & & \\
        68.00 & 4.109 $\pm$ 0.148 & 4.360 $\pm$ 0.011 & 4.386 $\pm$ 0.012 & & \\
        69.00 & 4.109 $\pm$ 0.148 & 4.307 $\pm$ 0.012 & 4.324 $\pm$ 0.012 & & \\
        70.00 & 4.109 $\pm$ 0.148 & 4.277 $\pm$ 0.011 & 4.315 $\pm$ 0.012 & & \\
        71.00 & 4.109 $\pm$ 0.148 & 4.268 $\pm$ 0.012 & 4.312 $\pm$ 0.012 & & \\
        72.00 & 4.109 $\pm$ 0.148 & 4.252 $\pm$ 0.012 & 4.286 $\pm$ 0.011 & & \\
        73.00 & 4.109 $\pm$ 0.148 & 4.247 $\pm$ 0.011 & 4.278 $\pm$ 0.012 & & \\
        74.00 & 4.109 $\pm$ 0.148 & 4.203 $\pm$ 0.011 & 4.221 $\pm$ 0.012 & & \\
        75.00 & 4.109 $\pm$ 0.148 & 4.207 $\pm$ 0.011 & 4.234 $\pm$ 0.012 & & \\
        76.00 & 4.109 $\pm$ 0.148 & 4.203 $\pm$ 0.011 & 4.217 $\pm$ 0.011 & & \\
        77.00 & 4.109 $\pm$ 0.148 & 4.171 $\pm$ 0.011 & 4.191 $\pm$ 0.012 & & \\
        78.00 & 4.109 $\pm$ 0.148 & 4.154 $\pm$ 0.011 & 4.182 $\pm$ 0.012 & & \\
        79.00 & 4.109 $\pm$ 0.148 & 4.155 $\pm$ 0.010 & 4.179 $\pm$ 0.012 & & \\
        80.00 & 4.109 $\pm$ 0.148 & 4.129 $\pm$ 0.010 & 4.144 $\pm$ 0.012 & & \\
        81.00 & 4.109 $\pm$ 0.148 & 4.134 $\pm$ 0.011 & 4.138 $\pm$ 0.012 & & \\
        82.00 & 4.109 $\pm$ 0.148 & 4.113 $\pm$ 0.011 & 4.126 $\pm$ 0.013 & & \\
        83.00 & 4.109 $\pm$ 0.148 & 4.117 $\pm$ 0.011 & 4.135 $\pm$ 0.013 & & \\
        84.00 & 4.109 $\pm$ 0.148 & 4.093 $\pm$ 0.011 & 4.105 $\pm$ 0.012 & & \\
        85.00 & 4.109 $\pm$ 0.148 & 4.096 $\pm$ 0.012 & 4.113 $\pm$ 0.012 & & \\
        86.00 & 4.109 $\pm$ 0.148 & 4.106 $\pm$ 0.012 & 4.112 $\pm$ 0.012 & & \\
        87.00 & 4.109 $\pm$ 0.148 & 4.099 $\pm$ 0.011 & 4.110 $\pm$ 0.013 & & \\
        88.00 & 4.109 $\pm$ 0.148 & 4.093 $\pm$ 0.011 & 4.101 $\pm$ 0.012 & & \\
        89.00 & 4.109 $\pm$ 0.148 & 4.072 $\pm$ 0.011 & 4.094 $\pm$ 0.013 & & \\
        90.00 & 4.109 $\pm$ 0.148 & 4.077 $\pm$ 0.011 & 4.095 $\pm$ 0.013 & & \\
        91.00 & 4.109 $\pm$ 0.148 & 4.063 $\pm$ 0.011 & 4.086 $\pm$ 0.013 & & \\
        92.00 & 4.109 $\pm$ 0.148 & 4.069 $\pm$ 0.011 & 4.083 $\pm$ 0.013 & & \\
        93.00 & 4.109 $\pm$ 0.148 & 4.073 $\pm$ 0.011 & 4.096 $\pm$ 0.013 & & \\
        94.00 & 4.109 $\pm$ 0.148 & 4.069 $\pm$ 0.011 & 4.105 $\pm$ 0.013 & & \\
        95.00 & 4.109 $\pm$ 0.148 & 4.062 $\pm$ 0.011 & 4.101 $\pm$ 0.013 & & \\
        96.00 & 4.109 $\pm$ 0.148 & 4.076 $\pm$ 0.012 & 4.109 $\pm$ 0.013 & & \\
        97.00 & 4.109 $\pm$ 0.148 & 4.056 $\pm$ 0.011 & 4.098 $\pm$ 0.012 & & \\
        98.00 & 4.109 $\pm$ 0.148 & 4.063 $\pm$ 0.011 & 4.104 $\pm$ 0.012 & & \\
        99.00 & 4.109 $\pm$ 0.148 & 4.066 $\pm$ 0.011 & 4.100 $\pm$ 0.012 & & \\
        99.99 & & & & & 4.586 $\pm$ 0.020 \\
        100.00 & 4.109 $\pm$ 0.148 & 4.057 $\pm$ 0.011 & 4.097 $\pm$ 0.012 & & \\
        \midrule
        Trials & 150 & 150 & 150 & 50 & 60 \\
        \bottomrule
    \end{tabular}
    \end{sc}
    \end{small}
    \end{center}
\end{table*}

\begin{table*}[h]
    \caption{
        \textbf{Averaged Per-Class Error Rates (\%) and SEM of Figure \ref{fig: Ablation study Mult vs LSEL vs Combo}. Frames 1--50.} Blanks mean N/A.
    } \label{tab: Values of ablation on NMNSIT-100f 1--50}
    \begin{center}
    \begin{small}
    \begin{sc}
    \begin{tabular}{llll}
        \toprule
        Time & TANDEM LSEL & TANDEM Mult & TANDEM LSEL+Mult \\
        \midrule
        1.00 & 56.196 $\pm$ 0.044 & 55.705 $\pm$ 0.036 &  \\
        2.00 & 37.742 $\pm$ 0.044 & 36.650 $\pm$ 0.036 & 35.942 $\pm$ 0.148 \\
        3.00 & 27.721 $\pm$ 0.044 & 26.072 $\pm$ 0.036 & 25.203 $\pm$ 0.148 \\
        4.00 & 22.037 $\pm$ 0.044 & 20.019 $\pm$ 0.036 & 19.167 $\pm$ 0.148 \\
        5.00 & 18.486 $\pm$ 0.044 & 16.320 $\pm$ 0.036 & 15.450 $\pm$ 0.148 \\
        6.00 & 16.015 $\pm$ 0.044 & 13.842 $\pm$ 0.036 & 12.997 $\pm$ 0.148 \\
        7.00 & 14.127 $\pm$ 0.044 & 12.028 $\pm$ 0.036 & 11.224 $\pm$ 0.148 \\
        8.00 & 12.664 $\pm$ 0.044 & 10.636 $\pm$ 0.036 & 9.912 $\pm$ 0.148 \\
        9.00 & 11.487 $\pm$ 0.044 & 9.564 $\pm$ 0.036 & 8.880 $\pm$ 0.148 \\
        10.00 & 10.526 $\pm$ 0.044 & 8.692 $\pm$ 0.036 & 8.068 $\pm$ 0.148 \\
        11.00 & 9.711 $\pm$ 0.044 & 7.989 $\pm$ 0.036 & 7.405 $\pm$ 0.148 \\
        12.00 & 9.042 $\pm$ 0.044 & 7.393 $\pm$ 0.036 & 6.862 $\pm$ 0.148 \\
        13.00 & 8.454 $\pm$ 0.044 & 6.897 $\pm$ 0.036 & 6.409 $\pm$ 0.148 \\
        14.00 & 7.945 $\pm$ 0.044 & 6.477 $\pm$ 0.036 & 6.032 $\pm$ 0.148 \\
        15.00 & 7.510 $\pm$ 0.044 & 6.119 $\pm$ 0.036 & 5.715 $\pm$ 0.148 \\
        16.00 & 7.124 $\pm$ 0.044 & 5.830 $\pm$ 0.036 & 5.444 $\pm$ 0.148 \\
        17.00 & 6.803 $\pm$ 0.044 & 5.579 $\pm$ 0.036 & 5.212 $\pm$ 0.148 \\
        18.00 & 6.519 $\pm$ 0.044 & 5.370 $\pm$ 0.036 & 5.025 $\pm$ 0.148 \\
        19.00 & 6.276 $\pm$ 0.044 & 5.188 $\pm$ 0.036 & 4.865 $\pm$ 0.148 \\
        20.00 & 6.060 $\pm$ 0.044 & 5.032 $\pm$ 0.036 & 4.730 $\pm$ 0.148 \\
        21.00 & 5.869 $\pm$ 0.044 & 4.899 $\pm$ 0.036 & 4.623 $\pm$ 0.148 \\
        22.00 & 5.700 $\pm$ 0.044 & 4.785 $\pm$ 0.036 & 4.536 $\pm$ 0.148 \\
        23.00 & 5.556 $\pm$ 0.044 & 4.687 $\pm$ 0.036 & 4.464 $\pm$ 0.148 \\
        24.00 & 5.434 $\pm$ 0.044 & 4.604 $\pm$ 0.036 & 4.401 $\pm$ 0.148 \\
        25.00 & 5.327 $\pm$ 0.044 & 4.537 $\pm$ 0.036 & 4.353 $\pm$ 0.148 \\
        26.00 & 5.229 $\pm$ 0.044 & 4.477 $\pm$ 0.036 & 4.310 $\pm$ 0.148 \\
        27.00 & 5.143 $\pm$ 0.044 & 4.426 $\pm$ 0.036 & 4.275 $\pm$ 0.148 \\
        28.00 & 5.070 $\pm$ 0.044 & 4.383 $\pm$ 0.036 & 4.246 $\pm$ 0.148 \\
        29.00 & 5.006 $\pm$ 0.044 & 4.353 $\pm$ 0.036 & 4.222 $\pm$ 0.148 \\
        30.00 & 4.952 $\pm$ 0.044 & 4.325 $\pm$ 0.036 & 4.203 $\pm$ 0.148 \\
        31.00 & 4.903 $\pm$ 0.044 & 4.299 $\pm$ 0.036 & 4.186 $\pm$ 0.148 \\
        32.00 & 4.858 $\pm$ 0.044 & 4.279 $\pm$ 0.036 & 4.171 $\pm$ 0.148 \\
        33.00 & 4.820 $\pm$ 0.044 & 4.261 $\pm$ 0.036 & 4.160 $\pm$ 0.148 \\
        34.00 & 4.787 $\pm$ 0.044 & 4.246 $\pm$ 0.036 & 4.147 $\pm$ 0.148 \\
        35.00 & 4.760 $\pm$ 0.044 & 4.234 $\pm$ 0.036 & 4.139 $\pm$ 0.148 \\
        36.00 & 4.736 $\pm$ 0.044 & 4.222 $\pm$ 0.036 & 4.133 $\pm$ 0.148 \\
        37.00 & 4.715 $\pm$ 0.044 & 4.211 $\pm$ 0.036 & 4.126 $\pm$ 0.148 \\
        38.00 & 4.697 $\pm$ 0.044 & 4.204 $\pm$ 0.036 & 4.121 $\pm$ 0.148 \\
        39.00 & 4.679 $\pm$ 0.044 & 4.197 $\pm$ 0.036 & 4.118 $\pm$ 0.148 \\
        40.00 & 4.662 $\pm$ 0.044 & 4.193 $\pm$ 0.036 & 4.116 $\pm$ 0.148 \\
        41.00 & 4.651 $\pm$ 0.044 & 4.188 $\pm$ 0.036 & 4.113 $\pm$ 0.148 \\
        42.00 & 4.640 $\pm$ 0.044 & 4.183 $\pm$ 0.036 & 4.112 $\pm$ 0.148 \\
        43.00 & 4.631 $\pm$ 0.044 & 4.180 $\pm$ 0.036 & 4.110 $\pm$ 0.148 \\
        44.00 & 4.624 $\pm$ 0.044 & 4.176 $\pm$ 0.036 & 4.109 $\pm$ 0.148 \\
        45.00 & 4.615 $\pm$ 0.044 & 4.174 $\pm$ 0.036 & 4.109 $\pm$ 0.148 \\
        46.00 & 4.609 $\pm$ 0.044 & 4.172 $\pm$ 0.036 & 4.108 $\pm$ 0.148 \\
        47.00 & 4.604 $\pm$ 0.044 & 4.170 $\pm$ 0.036 & 4.108 $\pm$ 0.148 \\
        48.00 & 4.599 $\pm$ 0.044 & 4.168 $\pm$ 0.036 & 4.108 $\pm$ 0.148 \\
        49.00 & 4.594 $\pm$ 0.044 & 4.167 $\pm$ 0.036 & 4.107 $\pm$ 0.148 \\
        50.00 & 4.591 $\pm$ 0.044 & 4.166 $\pm$ 0.036 & 4.107 $\pm$ 0.148 \\
        \midrule
        Trials & 150 & 150 & 150 \\
        \bottomrule
    \end{tabular}
    \end{sc}
    \end{small}
    \end{center}
\end{table*}

\begin{table*}[h]
    \caption{
        \textbf{Averaged Per-Class Error Rates (\%) and SEM of Figure \ref{fig: Ablation study Mult vs LSEL vs Combo}. Frames 51--100.} Blanks mean N/A.
    } \label{tab: Values of ablation on NMNSIT-100f 51--100}
    \begin{center}
    \begin{small}
    \begin{sc}
    \begin{tabular}{llll}
        \toprule
        Time & TANDEM LSEL & TANDEM Mult & TANDEM LSEL+Mult \\
        \midrule
        51.00 & 4.588 $\pm$ 0.044 & 4.165 $\pm$ 0.036 & 4.107 $\pm$ 0.148 \\
        52.00 & 4.586 $\pm$ 0.044 & 4.164 $\pm$ 0.036 & 4.107 $\pm$ 0.148 \\
        53.00 & 4.584 $\pm$ 0.044 & 4.164 $\pm$ 0.036 & 4.107 $\pm$ 0.148 \\
        54.00 & 4.580 $\pm$ 0.044 & 4.163 $\pm$ 0.036 & 4.106 $\pm$ 0.148 \\
        55.00 & 4.577 $\pm$ 0.044 & 4.163 $\pm$ 0.036 & 4.107 $\pm$ 0.148 \\
        56.00 & 4.576 $\pm$ 0.044 & 4.163 $\pm$ 0.036 & 4.107 $\pm$ 0.148 \\
        57.00 & 4.575 $\pm$ 0.044 & 4.163 $\pm$ 0.036 & 4.107 $\pm$ 0.148 \\
        58.00 & 4.573 $\pm$ 0.044 & 4.162 $\pm$ 0.036 & 4.107 $\pm$ 0.148 \\
        59.00 & 4.572 $\pm$ 0.044 & 4.163 $\pm$ 0.036 & 4.108 $\pm$ 0.148 \\
        60.00 & 4.572 $\pm$ 0.044 & 4.163 $\pm$ 0.036 & 4.108 $\pm$ 0.148 \\
        61.00 & 4.571 $\pm$ 0.044 & 4.163 $\pm$ 0.036 & 4.108 $\pm$ 0.148 \\
        62.00 & 4.570 $\pm$ 0.044 & 4.164 $\pm$ 0.036 & 4.108 $\pm$ 0.148 \\
        63.00 & 4.570 $\pm$ 0.044 & 4.163 $\pm$ 0.036 & 4.108 $\pm$ 0.148 \\
        64.00 & 4.569 $\pm$ 0.044 & 4.163 $\pm$ 0.036 & 4.108 $\pm$ 0.148 \\
        65.00 & 4.569 $\pm$ 0.044 & 4.163 $\pm$ 0.036 & 4.108 $\pm$ 0.148 \\
        66.00 & 4.568 $\pm$ 0.044 & 4.164 $\pm$ 0.036 & 4.109 $\pm$ 0.148 \\
        67.00 & 4.568 $\pm$ 0.044 & 4.164 $\pm$ 0.036 & 4.109 $\pm$ 0.148 \\
        68.00 & 4.567 $\pm$ 0.044 & 4.164 $\pm$ 0.036 & 4.109 $\pm$ 0.148 \\
        69.00 & 4.567 $\pm$ 0.044 & 4.163 $\pm$ 0.036 & 4.109 $\pm$ 0.148 \\
        70.00 & 4.567 $\pm$ 0.044 & 4.164 $\pm$ 0.036 & 4.109 $\pm$ 0.148 \\
        71.00 & 4.566 $\pm$ 0.044 & 4.164 $\pm$ 0.036 & 4.109 $\pm$ 0.148 \\
        72.00 & 4.566 $\pm$ 0.044 & 4.164 $\pm$ 0.036 & 4.109 $\pm$ 0.148 \\
        73.00 & 4.565 $\pm$ 0.044 & 4.164 $\pm$ 0.036 & 4.109 $\pm$ 0.148 \\
        74.00 & 4.565 $\pm$ 0.044 & 4.164 $\pm$ 0.036 & 4.109 $\pm$ 0.148 \\
        75.00 & 4.565 $\pm$ 0.044 & 4.164 $\pm$ 0.036 & 4.109 $\pm$ 0.148 \\
        76.00 & 4.565 $\pm$ 0.044 & 4.165 $\pm$ 0.036 & 4.109 $\pm$ 0.148 \\
        77.00 & 4.565 $\pm$ 0.044 & 4.165 $\pm$ 0.036 & 4.109 $\pm$ 0.148 \\
        78.00 & 4.565 $\pm$ 0.044 & 4.165 $\pm$ 0.036 & 4.109 $\pm$ 0.148 \\
        79.00 & 4.564 $\pm$ 0.044 & 4.165 $\pm$ 0.036 & 4.109 $\pm$ 0.148 \\
        80.00 & 4.565 $\pm$ 0.044 & 4.165 $\pm$ 0.036 & 4.109 $\pm$ 0.148 \\
        81.00 & 4.564 $\pm$ 0.044 & 4.165 $\pm$ 0.036 & 4.109 $\pm$ 0.148 \\
        82.00 & 4.564 $\pm$ 0.044 & 4.165 $\pm$ 0.036 & 4.109 $\pm$ 0.148 \\
        83.00 & 4.564 $\pm$ 0.044 & 4.165 $\pm$ 0.036 & 4.109 $\pm$ 0.148 \\
        84.00 & 4.564 $\pm$ 0.044 & 4.165 $\pm$ 0.036 & 4.109 $\pm$ 0.148 \\
        85.00 & 4.564 $\pm$ 0.044 & 4.165 $\pm$ 0.036 & 4.109 $\pm$ 0.148 \\
        86.00 & 4.564 $\pm$ 0.044 & 4.165 $\pm$ 0.036 & 4.109 $\pm$ 0.148 \\
        87.00 & 4.564 $\pm$ 0.044 & 4.165 $\pm$ 0.036 & 4.109 $\pm$ 0.148 \\
        88.00 & 4.564 $\pm$ 0.044 & 4.165 $\pm$ 0.036 & 4.109 $\pm$ 0.148 \\
        89.00 & 4.564 $\pm$ 0.044 & 4.165 $\pm$ 0.036 & 4.109 $\pm$ 0.148 \\
        90.00 & 4.564 $\pm$ 0.044 & 4.165 $\pm$ 0.036 & 4.109 $\pm$ 0.148 \\
        91.00 & 4.564 $\pm$ 0.044 & 4.165 $\pm$ 0.036 & 4.109 $\pm$ 0.148 \\
        92.00 & 4.564 $\pm$ 0.044 & 4.166 $\pm$ 0.036 & 4.109 $\pm$ 0.148 \\
        93.00 & 4.564 $\pm$ 0.044 & 4.166 $\pm$ 0.036 & 4.109 $\pm$ 0.148 \\
        94.00 & 4.564 $\pm$ 0.044 & 4.166 $\pm$ 0.036 & 4.109 $\pm$ 0.148 \\
        95.00 & 4.564 $\pm$ 0.044 & 4.166 $\pm$ 0.036 & 4.109 $\pm$ 0.148 \\
        96.00 & 4.564 $\pm$ 0.044 & 4.166 $\pm$ 0.036 & 4.109 $\pm$ 0.148 \\
        97.00 & 4.564 $\pm$ 0.044 & 4.166 $\pm$ 0.036 & 4.109 $\pm$ 0.148 \\
        98.00 & 4.564 $\pm$ 0.044 & 4.166 $\pm$ 0.036 & 4.109 $\pm$ 0.148 \\
        99.00 & 4.564 $\pm$ 0.044 & 4.166 $\pm$ 0.036 & 4.109 $\pm$ 0.148 \\
        100.00 & 4.564 $\pm$ 0.044 & 4.166 $\pm$ 0.036 & 4.109 $\pm$ 0.148 \\
        \midrule
        Trials & 150 & 150 & 150 \\
        \bottomrule
    \end{tabular}
    \end{sc}
    \end{small}
    \end{center}
\end{table*}

\clearpage
\section{LLR Trajectories} \label{app: LLR trajectories}
Figures \ref{fig: LLR trajectories of all} and \ref{fig: LLR trajectories of all2} show the LLR trajectories of  NMNIST-H, NMNIST-100f, UCF101, and HMDB51. The base models are the same as those in Figure \ref{fig: SATCs}.

\begin{figure}[htb]
    \begin{minipage}[b]{1.0\linewidth}
        \centering
        \includegraphics[width=\columnwidth, keepaspectratio]
        {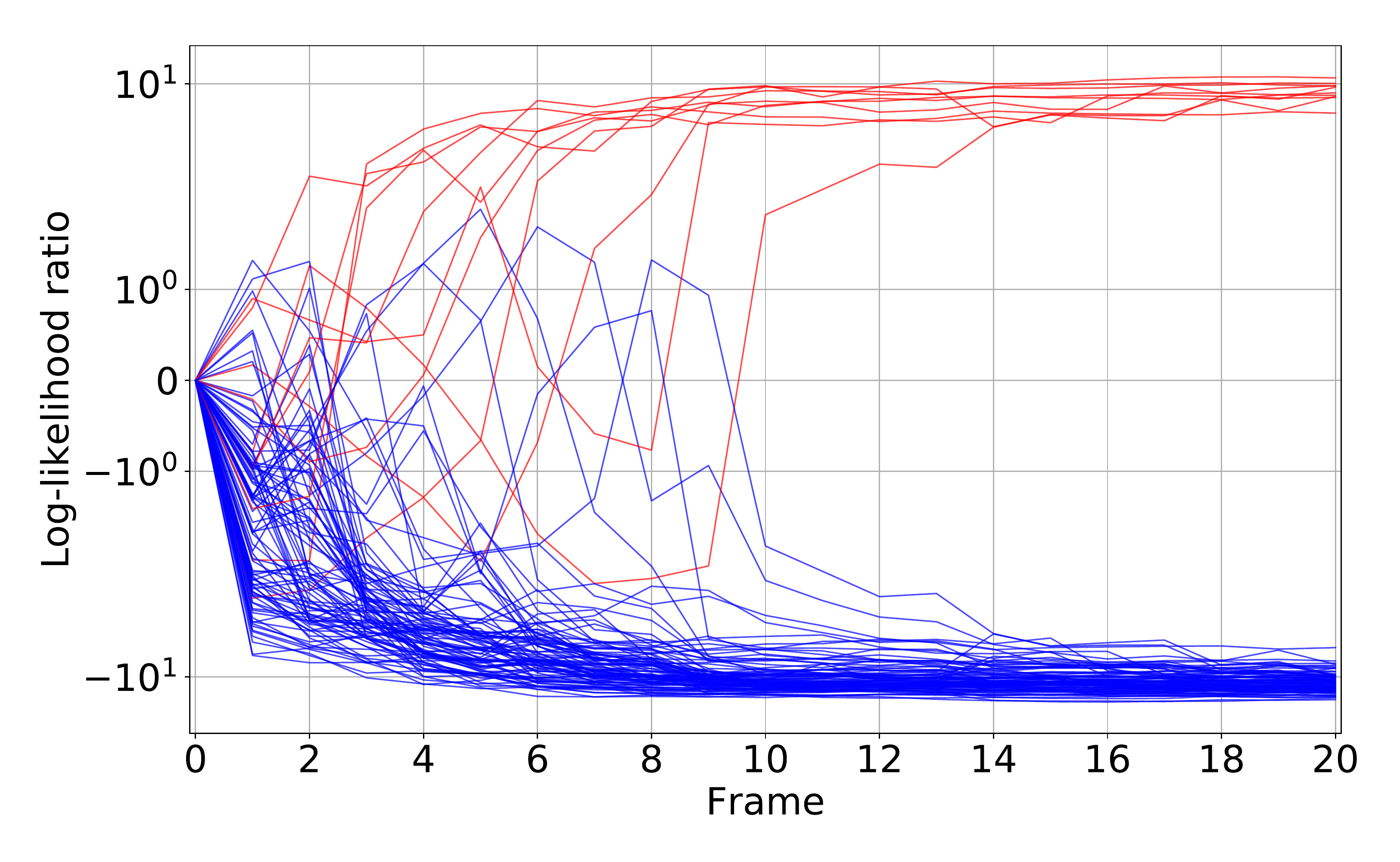} 
    \end{minipage}
    \begin{minipage}[b]{1.0\linewidth}
        \centering
        \includegraphics[width=\columnwidth, keepaspectratio]
        {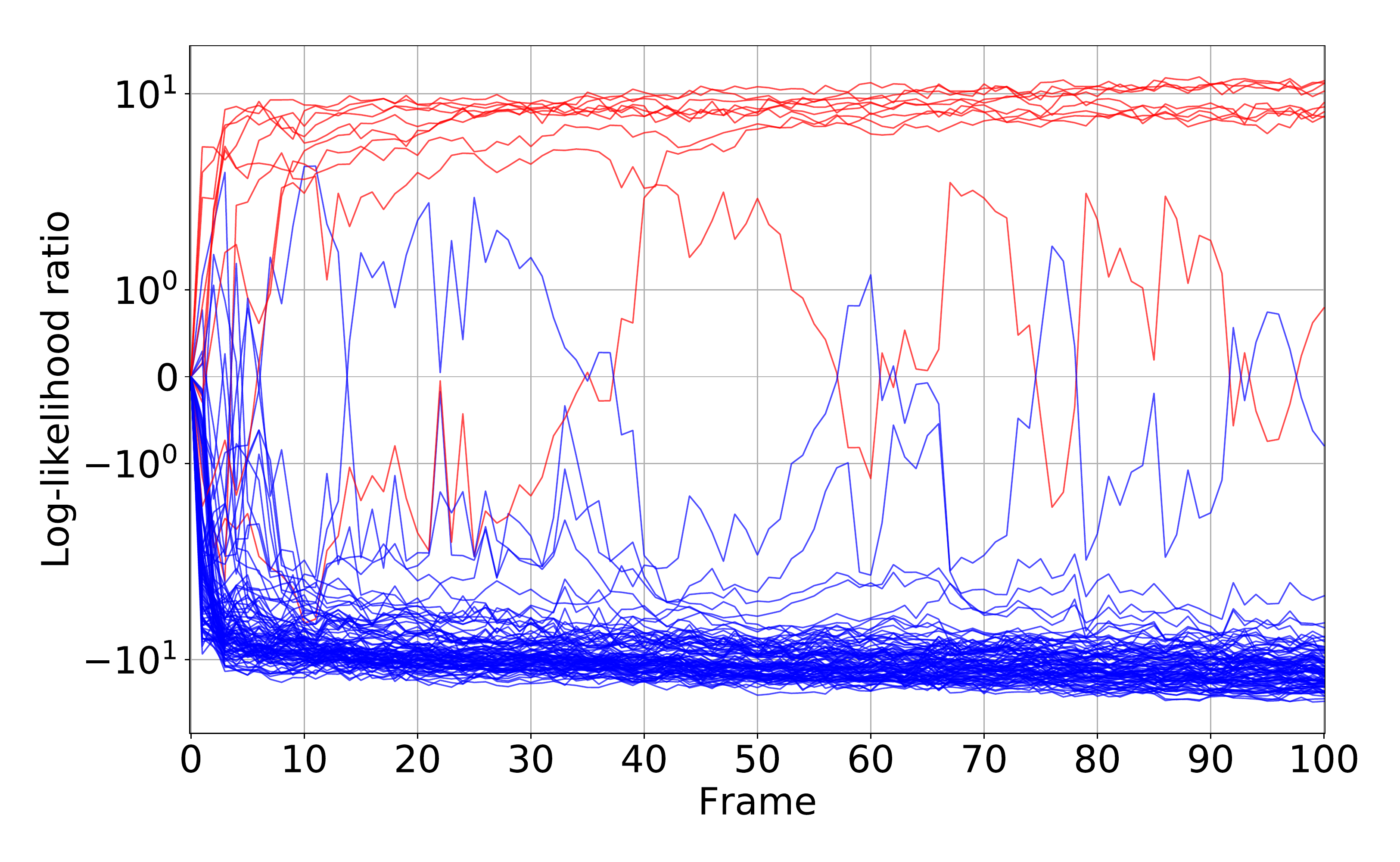} 
    \end{minipage}
    \caption{\textbf{Examples of LLR trajectories. Top: NMNIST-H. Bottom: NMNIST-100f.} The red curves represent $\min_{l} \{ \hat{\la}_{y_i l} \} $, while the blue curves represent $\min_{l} \{ \hat{\la}_{k l} \} $ ($k \neq y_i$). If the red curve reaches the threshold (necessarily positive), then the prediction is correct, while if the blue curve reaches the threshold, then the prediction is wrong. We plot ten different $i$'s randomly selected from the validation set. The red and blue curves are gradually separated as more frames are observed: evidence accumulation.}
    \label{fig: LLR trajectories of all}
\end{figure}

\begin{figure}[htb]
    \begin{minipage}[b]{1.0\linewidth}
        \centering
        \includegraphics[width=\columnwidth, keepaspectratio]
        {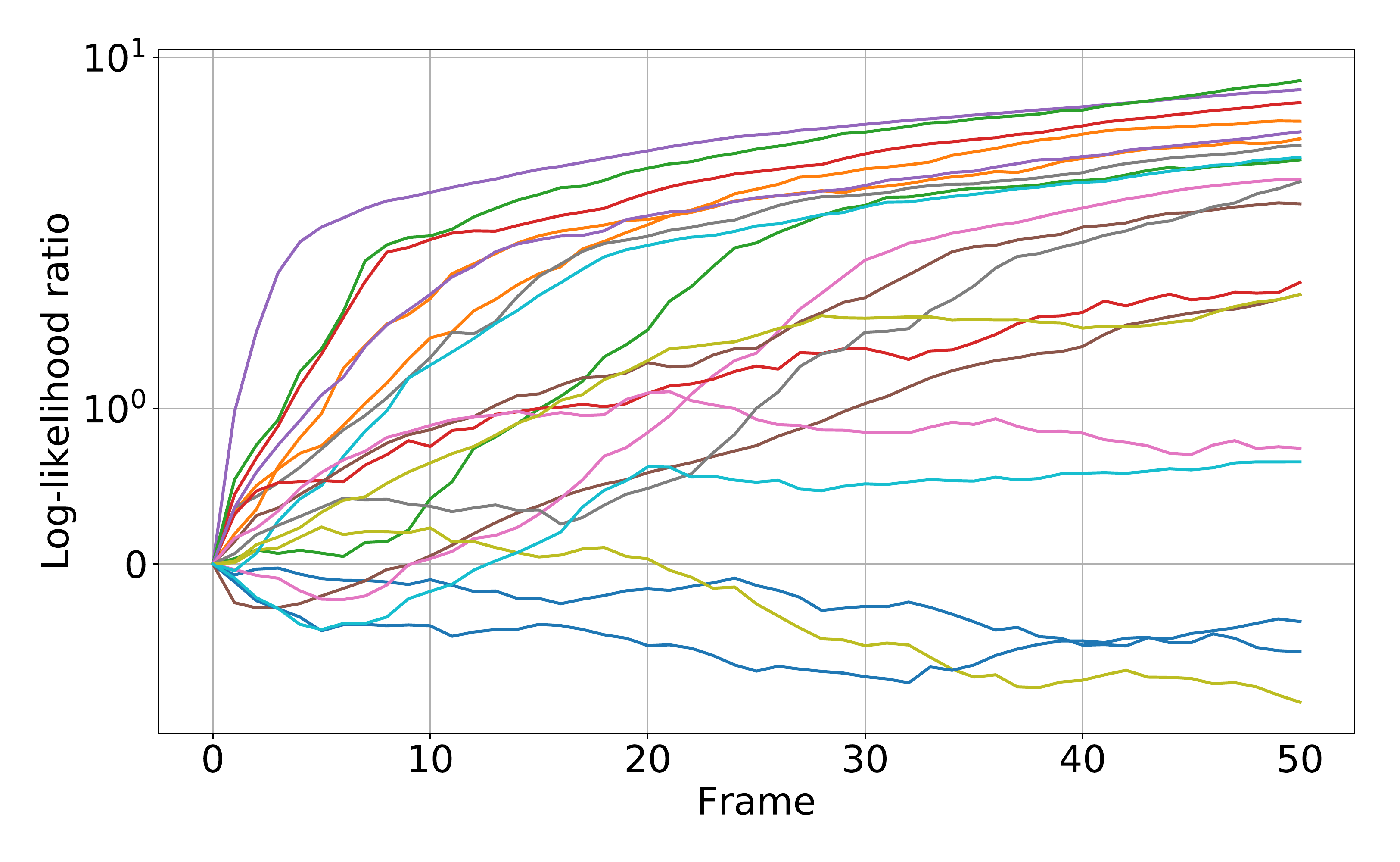} 
    \end{minipage}
    \begin{minipage}[b]{1.0\linewidth}
        \centering
        \includegraphics[width=\columnwidth, keepaspectratio]
        {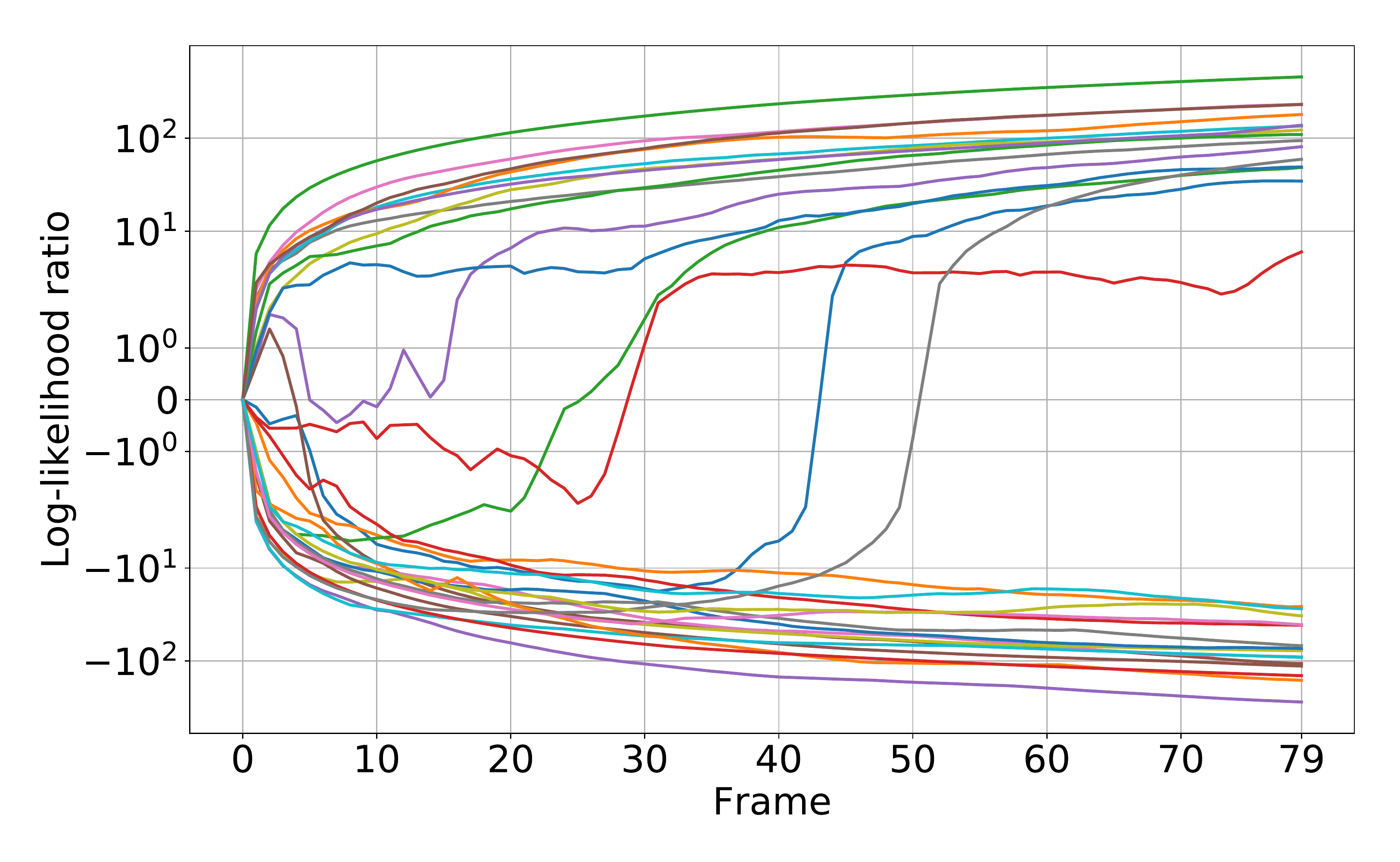} 
    \end{minipage}
    \caption{\textbf{Examples of LLR trajectories. Top: UCF101. Bottom: HMDB51.} All curves are $\min_{l} \{ \hat{\la}_{y_i l} \}$, and the negative rows $\min_{l} \{ \hat{\la}_{k l} \} $ ($k \neq y_i$) are omitted for clarity. Therefore, all the curves should be in the upper half-plane (positive LLRs). We plot 20 and 30 different $i$'s randomly selected from the validation sets of UCF101 and HMDB51, respectively. Several curves gradually go upwards as more frames are observed: evidence accumulation.}
    \label{fig: LLR trajectories of all2}
\end{figure}

\clearpage
\section{NMNIST-H and NMNIST-100f} \label{app: NMNIST-H and NMNIST-100f}
\citeApp{SPRT-TANDEM} propose an MNIST-based sequential dataset for early classification of time series:  \textit{NMNIST}; however, NMNIST is so simple that accuracy tends to saturate immediately and the performance comparison of models is difficult, especially in the early stage of sequential predictions.
We thus create a more complex dataset, NMNIST-H, with higher noise density than NMNIST. Figure \ref{fig: NMNIST-H} is an example video of NMNIST-H. It is hard, if not impossible, for humans to classify the video within 10 frames.

NMNIST-100f is a more challenging dataset than NMNIST-H. Each video consists of 100 frames, which is 5 times longer than in NMNIST and NMNSIT-H.
Figure \ref{fig: NMNIST-100f} is an example video of NMNIST-100f. Because of the dense noise, classification is unrealistic for humans, while MSPRT-TANDEM attains approximately 90 \% accuracy with only 8 frames (see Figure \ref{fig: SATCs}).

\begin{figure*}[htb]
    \begin{center}
    \centerline{\includegraphics[width=450pt]{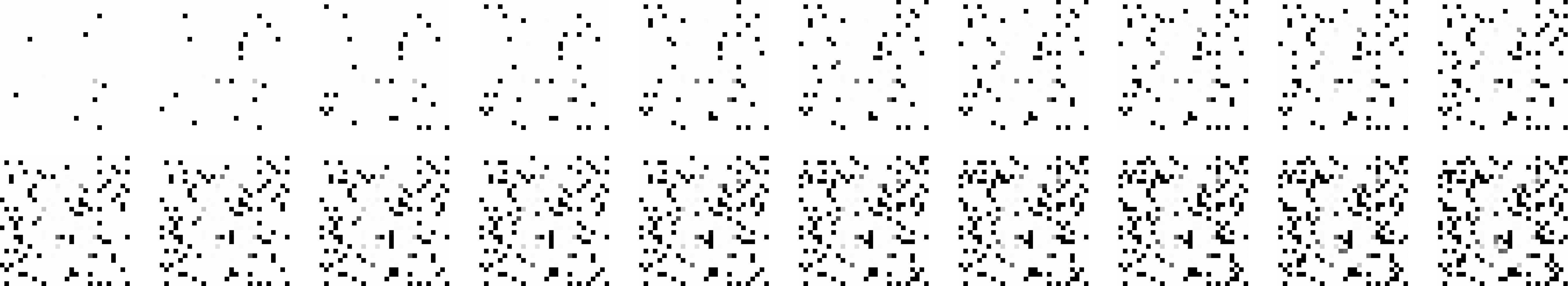}}
    \caption{
        \textbf{NMNIST-H.} The first frame is at the top left, and the last frame is at the bottom right. The original MNIST image (28$\times$28 pixels) is gradually revealed (10 pixels per frame). The label of this example is 6. The mean image is given in Figure \ref{fig: Mean images} (left).
        } \label{fig: NMNIST-H}
    \end{center}
\end{figure*}

\begin{figure*}[htb]
    \begin{center}
    \centerline{\includegraphics[width=450pt]{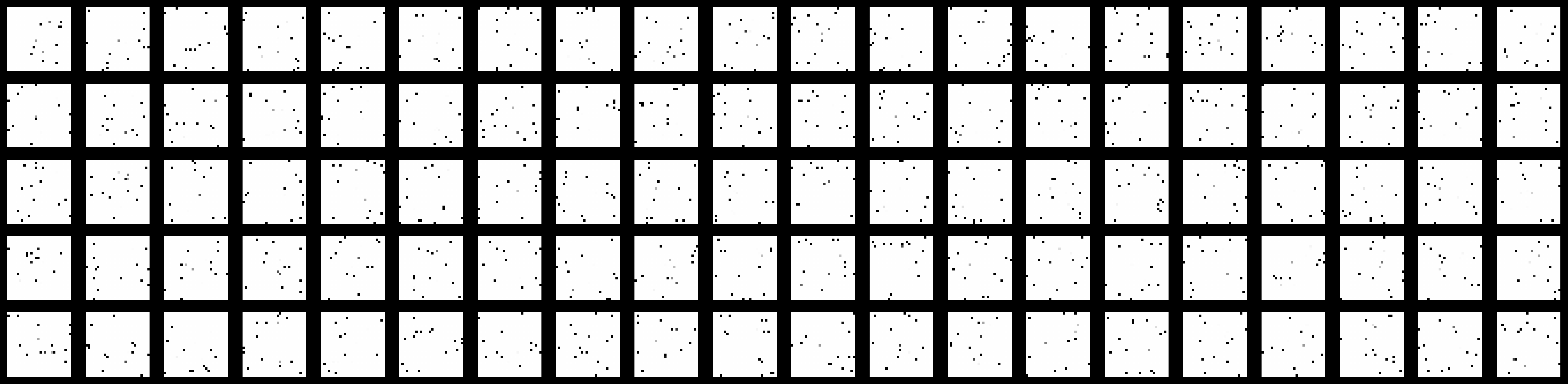}}
    \caption{
        \textbf{NMNIST-100f}. The first frame is at the top left, and the last frame is at the bottom right. An MNIST image (28$\times$28 pixels) is filled with white pixels except for 15 randomly selected pixels. Unlike NMNIST-H, the number of original pixels (15) is fixed throughout all frames. The label of this example is 3. The mean image is given in Figure \ref{fig: Mean images} (right). 
        } \label{fig: NMNIST-100f}
    \end{center}
\end{figure*}

\begin{figure*}[htb]
    \begin{minipage}[b]{0.5\linewidth}
        \centering
        \includegraphics[width=\columnwidth, keepaspectratio]
        {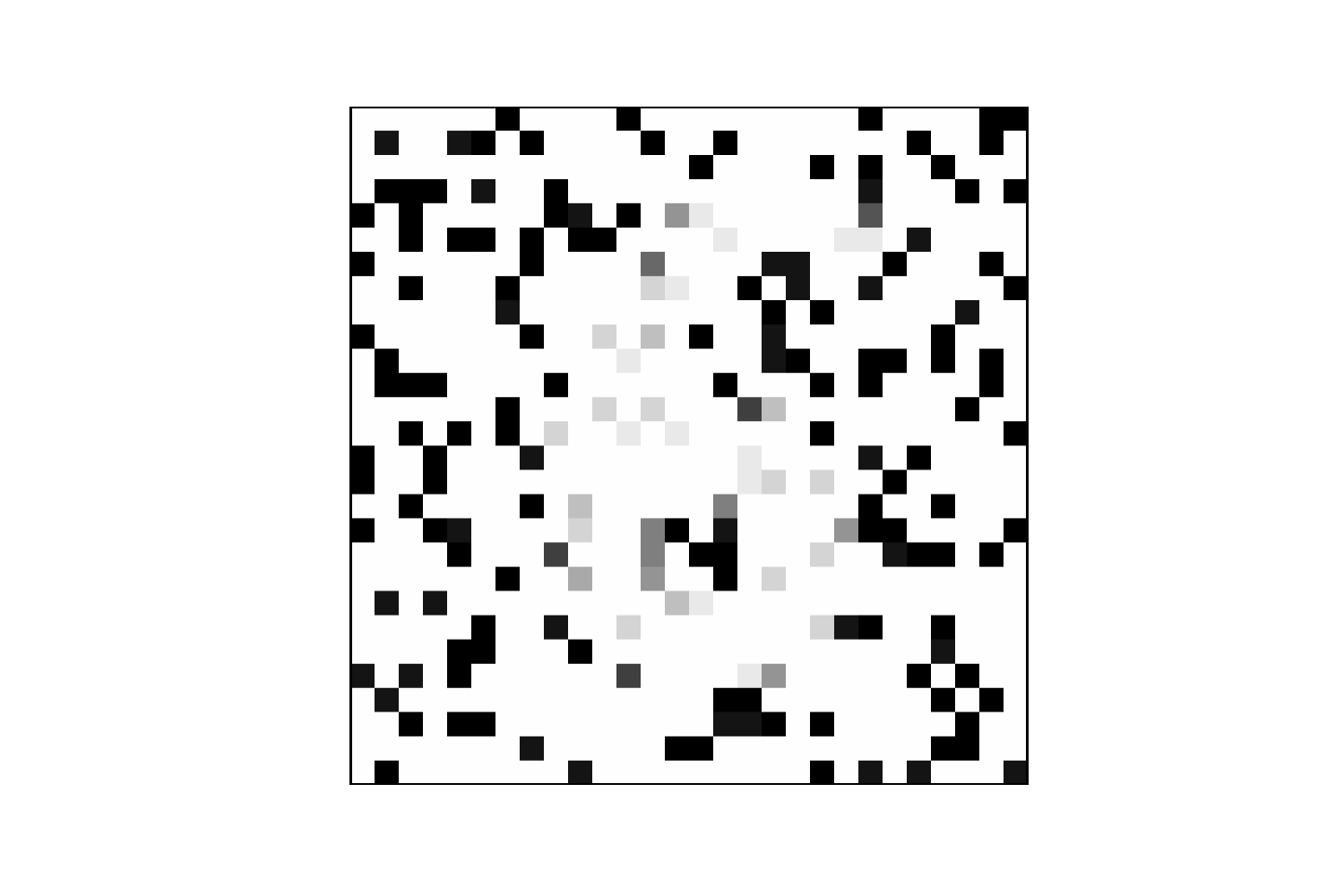} 
    \end{minipage}
    \begin{minipage}[b]{0.5\linewidth}
        \centering
        \includegraphics[width=\columnwidth, keepaspectratio]
        {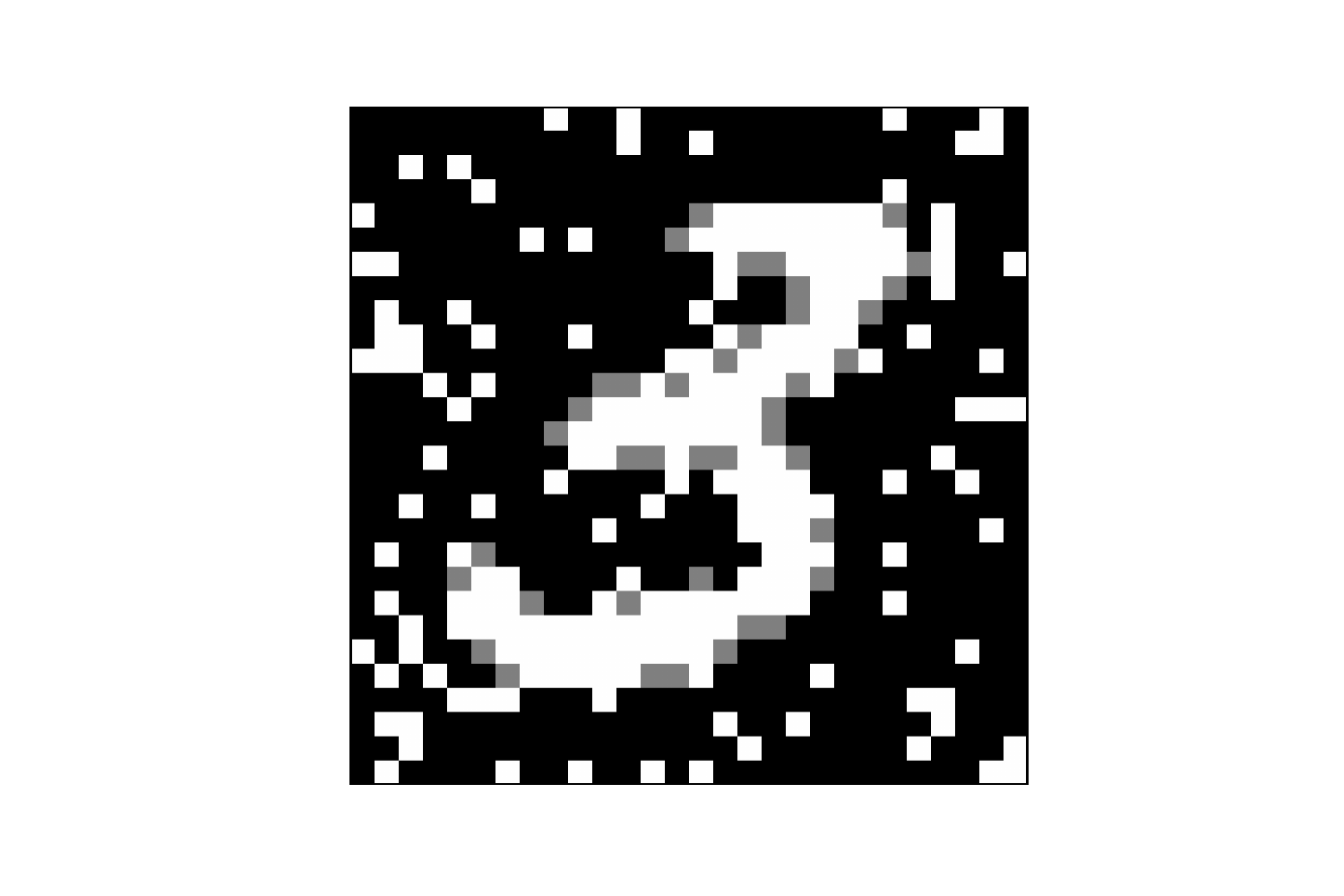} 
    \end{minipage}
    \caption{\textbf{Mean images of Figures \ref{fig: NMNIST-H} (left) and \ref{fig: NMNIST-100f} (right).}}
    \label{fig: Mean images}
\end{figure*}


%% file: supp_StatTest.tex
\section{Details of Statistical Tests} \label{app: Statistical Tests}

\subsection{Model Comparison: Figure \ref{fig: SATCs}}

For an objective comparison, we conduct statistical tests: two-way ANOVA \citeApp{FisherBook_ANOVA} followed by Tukey-Kramer multi-comparison test \citeApp{Tukey1949, Kramer1956}. In the tests, a small number of trials 
reduces test statistics, making it difficult to claim significance because the test statistic of the Tukey-Kramer test is proportional to $1/\sqrt{(1/n + 1/m)}$, where $n$ and $m$ are trial numbers of two models to be compared. These statistical tests are standard, e.g., in biological science, in which variable trial numbers are inevitable in experiments. All the statistical tests are executed with a customized \citeApp{MATLAB} script. 

In the two-way ANOVA, the two factors are defined as the phase (early and late stages of the SAT curve) and model. The actual numbers of frames shown in Table \ref{tab:phase_def} are chosen so that the compared models can use as similar frames as possible and thus depend only on the dataset. Note that EARLIEST cannot flexibly change the mean hitting time; thus, we include the results of EARLIEST to the groups with as close to the number of frames as possible. 
 
The $p$-values are summarized in Tables \ref{tab:NMNIST_100f_LSELvsLogistic} to \ref{tab:HMDB51_stat}. The $p$-values with asterisks are statistically significant: one, two and three asterisks show $p < 0.05$, $p < 0.01$, and $p < 0.001$, respectively.
Our results, especially in the late phase, are statistically significantly better than those in the early phase, confirming that accumulating evidence leads to better performance. 

\begin{table*}[htb]
  \centering
  \tiny
  \caption{Definition of the phases (in the number of frames).}
    \begin{tabular}{ccccc}
          & \multicolumn{2}{c}{Early phase} & \multicolumn{2}{c}{Late phase}\\
          & EARLIEST & All but EARLIEST  & EARLIEST & All but EARLIEST \\
    \midrule
    NMNIST-H & 1.34 & 2 & 13.41 & 13 \\
    NMNIST-100f & 13.41 & 19 & 99.99 & 100 \\
    UCF101 & 1.36 & 1 & 49.93 & 50 \\
    HMDB51 & 1.43 & 1 & 36.20 & 36 \\
    \bottomrule
    \end{tabular}%
  \label{tab:phase_def}%
  \vskip -0.1in
\end{table*}%

\begin{table*}[htp]
  \centering
  \tiny
  \caption{$p$-values from the Tukey-Kramer multi-comparison test conducted on NMNIST-100f.(Figure \ref{fig: DRE Losses vs LSEL})}
    \begin{tabular}{cc|cccc}
          &       & \multicolumn{2}{c|}{Logistic} & \multicolumn{1}{c|}{LSEL}  \\
          &       & early & \multicolumn{1}{c|}{late} & \multicolumn{1}{c|}{early}  \\
    \midrule
    Logistic & late  & ***1E-07 &              \\
\cmidrule{1-2}    \multirow{1}[2]{*}{LSEL} & early & ***1E-09 & ***1E-09 &   \\
          & late  & ***1E-09  & *2E-3 & ***1E-09          \\
    \bottomrule
    \end{tabular}%
  \label{tab:NMNIST_100f_LSELvsLogistic}%
  \vskip -0.1in
\end{table*}%

\begin{table*}[htb]
  \centering
  \tiny
  \caption{$p$-values from the Tukey-Kramer multi-comparison test conducted on NMNIST-H (Figure \ref{fig: SATCs}).}
    \begin{tabular}{cc|ccccccccc}
          &       & \multicolumn{2}{c|}{MSPRT-TANDEM} & \multicolumn{2}{c|}{NP test} & \multicolumn{2}{c|}{LSTM-s} & \multicolumn{2}{c|}{LSTM-m} & \multicolumn{1}{c}{EARLIEST} \\
          &       & early & \multicolumn{1}{c|}{late} & early & \multicolumn{1}{c|}{late} & early & \multicolumn{1}{c|}{late} & early & \multicolumn{1}{c|}{late} & early \\
    \midrule
    MSPRT-TANDEM & late  & ***1E-07 &       &       &       &       &       &         \\
\cmidrule{1-2}    \multirow{1}[2]{*}{NP test} & early & ***1E-07 & ***1E-07 &       &       &       &       &         \\
          & late  & ***1E-07 & ***1E-07 & ***1E-07 &       &       &       &         \\
\cmidrule{1-2}    \multirow{1}[2]{*}{LSTM-s} & early & ***1E-07 & ***1E-07 &  1.0     &   ***1E-07    &       &       &         \\
          & late  & ***1E-07 & ***1E-07 & ***1E-07 &   1.0    &  ***1E-07     &       &         \\
\cmidrule{1-2}    \multirow{1}[2]{*}{LSTM-m} & early & ***1E-07 & ***1E-07 &  1.0     &   ***1E-07    &  1.0     &  ***1E-07     &         \\
          & late  & ***1E-07 & ***1E-07 & ***1E-07 &   1.0    &   ***1E-07    &   1.0    &  ***1E-07       \\
\cmidrule{1-2}    \multirow{1}[2]{*}{EARLIEST} & early & ***1E-07 & ***1E-07 & ***1E-07 &      ***1E-07 & ***1E-07 &   ***1E-07    &   ***1E-07   & ***1E-07   \\
          & late  & ***1E-07 & ***1E-07 & ***1E-07 & ***1E-07 & ***1E-07 &   ***1E-07    &  ***1E-07     & ***1E-07& ***1E-07  \\
    \bottomrule
    \end{tabular}%
  \label{tab:NMNIST_H_stat}%
  \vskip -0.1in
\end{table*}%

\begin{table*}[htb]
  \centering
  \tiny
  \caption{$p$-values from the Tukey-Kramer multi-comparison test conducted on NMNIST-100f (Figure \ref{fig: SATCs}).}
    \begin{tabular}{cc|ccccccccc}
          &       & \multicolumn{2}{c|}{MSPRT-TANDEM} & \multicolumn{2}{c|}{NP test} & \multicolumn{2}{c|}{LSTM-s} & \multicolumn{2}{c|}{LSTM-m} & \multicolumn{1}{c}{EARLIEST} \\
          &       & early & \multicolumn{1}{c|}{late} & early & \multicolumn{1}{c|}{late} & early & \multicolumn{1}{c|}{late} & early & \multicolumn{1}{c|}{late} & early \\
    \midrule
    MSPRT-TANDEM & late  & ***1E-07 &       &       &       &       &       &         \\
\cmidrule{1-2}    \multirow{1}[2]{*}{NP test} & early & ***1E-07 & ***1E-07 &       &       &       &       &         \\
          & late  & ***1E-07 & 1.0 & ***1E-07 &       &       &       &         \\
\cmidrule{1-2}    \multirow{1}[2]{*}{LSTM-s} & early & ***1E-07 & ***1E-07 &  ***1E-07   &   ***1E-07    &       &       &         \\
          & late  & ***1E-07 & 1.0 & ***1E-07 &  1.0  &  ***1E-07     &       &         \\
\cmidrule{1-2}    \multirow{1}[2]{*}{LSTM-m} & early & ***1E-07 & ***1E-07 &  ***1E-07   &   ***1E-07    &  1.0    &  ***1E-07     &         \\
          & late  & ***1E-07 & 1.0 & ***1E-07 &  1.0 &   ***1E-07    &  1.0  &  ***1E-07       \\
\cmidrule{1-2}    \multirow{1}[2]{*}{EARLIEST} & early & ***1E-07 & ***1E-07 & ***1E-07 &      ***1E-07 & ***1E-07 &   ***1E-07    &   ***1E-07   & ***1E-07   \\
          & late  & ***7E-05 & ***1E-07 & ***1E-07 & ***1E-07 & ***1E-07 &   ***1E-07    &  ***1E-07     & ***1E-07& ***1E-07  \\
    \bottomrule
    \end{tabular}%
  \label{tab:NMNIST_100f_stat}%
  \vskip -0.1in
\end{table*}%

\begin{table*}[htb]
  \centering
  \tiny
  \caption{$p$-values from the Tukey-Kramer multi-comparison test conducted on UCF101 (Figure \ref{fig: SATCs}).}
    \begin{tabular}{cc|ccccccccc}
          &       & \multicolumn{2}{c|}{MSPRT-TANDEM} & \multicolumn{2}{c|}{NP test} & \multicolumn{2}{c|}{LSTM-s} & \multicolumn{2}{c|}{LSTM-m} & \multicolumn{1}{c}{EARLIEST} \\
          &       & early & \multicolumn{1}{c|}{late} & early & \multicolumn{1}{c|}{late} & early & \multicolumn{1}{c|}{late} & early & \multicolumn{1}{c|}{late} & early \\
    \midrule
    MSPRT-TANDEM & late  & ***1E-07 &       &       &       &       &       &         \\
\cmidrule{1-2}    \multirow{1}[2]{*}{NP test} & early & 1.0 & ***1E-07 &       &       &       &       &         \\
          & late  & ***1E-07 & 1.0 & ***1E-07 &       &       &       &         \\
\cmidrule{1-2}    \multirow{1}[2]{*}{LSTM-s} & early & ***1E-07 & ***1E-07 &  ***1E-07   &   ***1E-07    &       &       &         \\
          & late  & ***1E-07 & ***1E-07 & ***1E-07 &  ***1E-07  &  ***1E-07     &       &         \\
\cmidrule{1-2}    \multirow{1}[2]{*}{LSTM-m} & early & ***1E-07 & ***1E-07 &  ***1E-07   &   ***1E-07    & 1.0 & ***1E-07    &         \\
          & late  & ***1E-07 & ***1E-07 & ***1E-07 &  ***1E-07 &   ***1E-07    &  *3E-03  &  ***1E-07       \\
\cmidrule{1-2}    \multirow{1}[2]{*}{EARLIEST} & early & ***1E-07 & ***1E-07 & ***1E-07 &      ***1E-07 & 1E-01 &   ***1E-07    &   *6E-03   & ***1E-07   \\
          & late  & ***1E-07 & ***1E-07 & ***1E-07 & ***1E-07 & ***1E-07 &   ***1E-07    &  ***1E-07     & ***1E-07& ***1E-07  \\
    \bottomrule
    \end{tabular}%
  \label{tab:UCF101_stat}%
  \vskip -0.1in
\end{table*}%

\begin{table*}[htb]
  \centering
  \tiny
  \caption{$p$-values from the Tukey-Kramer multi-comparison test conducted on HMDB51 (Figure \ref{fig: SATCs}).}
    \begin{tabular}{cc|ccccccccc}
          &       & \multicolumn{2}{c|}{MSPRT-TANDEM} & \multicolumn{2}{c|}{NP test} & \multicolumn{2}{c|}{LSTM-s} & \multicolumn{2}{c|}{LSTM-m} & \multicolumn{1}{c}{EARLIEST} \\
          &       & early & \multicolumn{1}{c|}{late} & early & \multicolumn{1}{c|}{late} & early & \multicolumn{1}{c|}{late} & early & \multicolumn{1}{c|}{late} & early \\
    \midrule
    MSPRT-TANDEM & late  & ***1E-07 &       &       &       &       &       &         \\
\cmidrule{1-2}    \multirow{1}[2]{*}{NP test} & early & 1.0 &  ***1E-07  &       &       &       &       &         \\
          & late  & ***1E-07 & 0.2 & ***1E-07 &       &       &       &         \\
\cmidrule{1-2}    \multirow{1}[2]{*}{LSTM-s} & early & 1.0 & ***1E-07 &  1.0     &   ***1E-07    &       &       &         \\
          & late  & ***1E-07 & ***2E-07 & ***1E-07 &   ***2E-02    &  ***1E-07     &       &         \\
\cmidrule{1-2}    \multirow{1}[2]{*}{LSTM-m} & early & 1.0 & ***1E-07 &  1.0     &   ***1E-07    &  1.0     &  ***1E-07     &         \\
          & late  & ***1E-07 & ***1E-07 & ***1E-07 &   ***1E-04   &   ***1E-07    &   1.0    &  ***1E-07       \\
\cmidrule{1-2}    \multirow{1}[2]{*}{EARLIEST} & early & 1.0 & ***1E-07 & 1.0 &      ***1E-07 & 1.0 &   ***1E-07    &  1.0  & ***1E-07   \\
          & late  & ***1E-07 & ***1E-07 & ***1E-07 & ***1E-07 & ***1E-07 &   ***1E-07    &  ***1E-07     & ***7E-07& ***1E-07  \\
    \bottomrule
    \end{tabular}%
  \label{tab:HMDB51_stat}%
  \vskip -0.1in
\end{table*}%

\subsection{Ablation Study: Figure \ref{fig: Ablation study Mult vs LSEL vs Combo}}
We also test the three conditions in the ablation study. We select one phase at the 20th frame to conduct the one-way ANOVA with one model factor: LSEL + Multiplet loss, LSEL only, and Multiplet only. The $p$-values are summarized in Table \ref{tab:ablation_stat}. The result shows that using both the LSEL and the multiplet loss is statistically significantly better than using either of the two losses.

\begin{table*}[htb]
  \centering
  \tiny
  \caption{Figure \ref{fig: Ablation study Mult vs LSEL vs Combo}: $p$-values from the Tukey-Kramer multi-comparison test conducted on the ablation test.}
    \begin{tabular}{c|cc}
          &        \multicolumn{2}{c}{MSPRT-TANDEM} \\
          &        LSEL+Multiplet & LSEL only \\
    \midrule
    LSEL only  & ***1E-09 &  \\
    Multiplet only  & ***2E-07 & ***1E-09  \\
    \bottomrule
    \end{tabular}%
  \label{tab:ablation_stat}%
  \vskip -0.1in
\end{table*}%

%% file: supp_Discussion.tex
\section{Supplementary Discussion} \label{app: Supplementary Discussion}

\paragraph{Interpretability.}
Interpretability of classification results is one of the important interests in early classification of time series \citeApp{xing2011extracting_localized_class_signature, ghalwash2012early_localized_class_signature, ghalwash2013extraction_localized_class_signature, ghalwash2014utilizing_localized_class_signature, karim2019framework_localized_class_signature}. MSPRT-TANDEM can use the LLR trajectory to visualize the prediction process (see Figures \ref{fig: DRE Losses vs LSEL} (Right) and \ref{fig: LLR trajectories of all}); a large gap of LLRs between two timestamps means that these timestamps are decisive.

\paragraph{Threshold matrix.}
In our experiment, we use single-valued threshold matrices for simplicity. General threshold matrices may enhance performance, especially when the dataset is class-imbalanced \citeApp{longadge2013class, ali2015classification, hong2016dealing}. Tuning the threshold after training is referred to as thresholding, or threshold-moving \citeApp{richard1991neural, buda2018systematic_class_imbalance}. MSPRT-TANDEM has multiple thresholds in the matrix form, and thus it is an interesting future work to exploit such a matrix structure to attain higher accuracy.


\paragraph{How to determine threshold.}
A user can choose a threshold by evaluating the mean hitting time and accuracy on a dataset at hand (possibly the validation dataset, training dataset, or both). As mentioned in Section \ref{sec: Overall Architecture: MSPRT-TANDEM}, we do not have to retrain the model (mentioned in Section \ref{sec: Overall Architecture: MSPRT-TANDEM}). We can modify the threshold even after deployment, if necessary, to address domain shift. This flexibility is a huge advantage compared with other models that require additional training every time the user wants to control the speed-accuracy tradeoff \citeApp{Cai2020ICLROnce-for-all_efficient_deploy}.